\documentclass{article}

\usepackage{microtype}
\usepackage{graphicx}
\usepackage{booktabs} 

\usepackage{hyperref}

\usepackage[accepted]{icml2025}

\usepackage{amsmath}
\usepackage{amssymb}
\usepackage{mathtools}
\usepackage{amsthm}

\usepackage[capitalize,noabbrev]{cleveref}

\usepackage{amsfonts}       
\usepackage{nicefrac}       
\usepackage{xcolor}         
\usepackage{mathrsfs,yfonts}
\usepackage{subcaption}
\theoremstyle{plain}
\newtheorem{theorem}{Theorem}[section]

\newtheorem{lemma}[theorem]{Lemma}

\theoremstyle{definition}
\newtheorem{definition}[theorem]{Definition}
\newtheorem{assumption}[theorem]{Assumption}
\theoremstyle{remark}
\newtheorem{remark}[theorem]{Remark}

\newcommand{\R}{\mathbb{R}}
\DeclareMathOperator{\E}{\mathbb{E}}

\newcommand{\p}{\mathbb{P}}
\DeclareMathOperator{\argmin}{\operatorname{argmin}}
\DeclareMathOperator{\argmax}{\operatorname{argmax}}
\allowdisplaybreaks

\usepackage[textsize=tiny]{todonotes}

\icmltitlerunning{Self-Interested Agents in Collaborative Machine Learning: An Incentivized Adaptive Data-Centric Framework}

\begin{document}

\twocolumn[
\icmltitle{
Self-Interested Agents in Collaborative Machine Learning:\\An Incentivized Adaptive Data-Centric Framework}


\icmlsetsymbol{equal}{*}

\begin{icmlauthorlist}
    \icmlauthor{Nithia Vijayan}{yyy}
    \icmlauthor{Bryan Kian Hsiang Low}{yyy}
\end{icmlauthorlist}

\icmlaffiliation{yyy}{Department of Computer Science, National University of  Singapore, Republic of Singapore}

\icmlcorrespondingauthor{Nithia Vijayan}{nithia@comp.nus.edu.sg}

\icmlkeywords{collaborative machine learning, self-interested agents, incentives}

\vskip 0.3in
]



\printAffiliationsAndNotice{}  
\begin{abstract}
We propose a framework for adaptive data-centric collaborative machine learning among self-interested agents, coordinated by an arbiter. Designed to handle the incremental nature of real-world data, the framework operates in an online manner: at each time step, the arbiter collects a batch of data from agents, trains a machine learning model, and provides each agent with a distinct model reflecting its data contributions. This setup establishes a feedback loop where shared data influence model updates, and the resulting models guide future data-sharing policies. Agents evaluate and partition their data, selecting a partition to share using a stochastic parameterized policy, learned via policy gradient methods to optimize the utility of the received model as defined by agent-specific evaluation functions. On the arbiter side, the expected loss function over the true data distribution is optimized, incorporating agent-specific weights to account for distributional differences arising from diverse sources and selective sharing. A bilevel optimization algorithm jointly learns the model parameters and agent-specific weights. Mean-zero noise, computed using a distortion function that adjusts these agent-specific weights, is introduced to generate distinct agent-specific models, promoting valuable data sharing without requiring separate training. Our framework is underpinned by non-asymptotic analyses, ensuring convergence of the agent-side policy optimization to an approximate stationary point of the evaluation functions and convergence of the arbiter-side optimization to an approximate stationary point of the expected loss function.
\end{abstract}

\section{Introduction}
\label{sec:intro}
Collaborative machine learning (CML) leverages aggregated data from multiple agents to develop more accurate model(s) than any single agent can achieve independently. This approach is particularly valuable in scenarios where data is distributed across various sources such as different institutions or geographic locations. CML comprises two main approaches: data-centric  and model-centric. The data-centric approach enhances training by centrally pooling data, providing access to a comprehensive and unified dataset that utilizes diverse data to gain a deeper understanding of the data distribution. In contrast, the model-centric approach retains data locally, sharing only model parameters or updates. Although this method emphasizes privacy, it may overlook the full complexity of the data distribution, leading to models that are more susceptible to biases.

Data-centric CML methods typically rely on offline learning, where the complete dataset is available from the start, rather than incorporating online algorithms that handle data arriving incrementally. In real-world scenarios, data is generally gathered incrementally rather than all at once, necessitating an online approach. This allows models to continuously update as new data becomes available, ensuring they effectively learn from ongoing inputs. However, instead of processing data as a continuous stream, it is more practical to consider sequential data processing in batches. This mirrors how data is often shared in discrete intervals, such as in systems where data is periodically exchanged or collected, allowing for more efficient processing while managing computational constraints and synchronization needs.

Traditionally, data-centric CML methods view all agents as benevolent, freely sharing their data to support collaborative model development. However, in practice, the agents are often self-interested entities with distinct goals and constraints, such as businesses or organizations. While collaboration can help address the limitations of an agent's own dataset, they may choose to withhold portions of their data due to concerns regarding the quality or relevance of others' data, or to maintain a competitive edge, as sharing too much could risk losing a strategic advantage or compromising their position in research and innovation. Sharing its entire dataset may not be beneficial if the resulting combined model does not serve an agent’s specific objectives. Allowing each agent to strategically decide on its data contributions enables the creation of batch online learning frameworks that cater to individual goals while maximizing collaborative benefits. This approach acknowledges that optimizing solely for collaborative gains, without considering individual objectives, can lead to less effective outcomes in terms of individual satisfaction. By giving agents the power to influence the process, we enhance their engagement and motivation, ultimately fostering a more productive collaborative environment. This approach accommodates the diverse priorities and perspectives driven by each agent's self-interest and unique data-sharing strategies.

Potential applications of batch online data-centric CML with self-interested agents are evident in both healthcare and finance. For instance, hospitals can collaborate on cancer treatment research by leveraging unique patient data, such as demographics and treatment outcomes, to gain a comprehensive understanding of treatment efficacy. This pooling of data overcomes limitations like small sample sizes, resulting in more robust models that enhance patient outcomes. A batch online approach is essential here because new patient data is continuously generated, and models must be updated in real-time with each new batch of data to remain current and reflect the latest treatment trends. Similarly, in finance, CML can enhance credit card fraud detection by addressing the challenge of imbalanced data, where fraudulent transactions are far less common than legitimate ones. As fraud patterns evolve rapidly, a batch online learning method ensures that models can be updated in real-time with each new batch of data, allowing them to quickly adapt to emerging fraud cases and improve detection capabilities. This approach enables institutions to continuously train models on a more diverse set of fraud cases, ensuring the system remains effective as fraud tactics evolve. However, due to self-interests, hospitals and financial institutions may strategically withhold certain proprietary data while selectively sharing others to protect their competitive advantages, ensuring they maintain an edge in their fields.

We consider a framework for batch online data-centric CML among self-interested agents, facilitated by an arbiter. One might interpret the arbiter as a service provider that the collaborating agents have jointly agreed to use. The collaboration unfolds in an online manner, acknowledging that the data in real-world applications are not available all at once. In each online step, the arbiter collects a batch of data from the agents, trains a machine learning model, and subsequently provides each agent with a distinct model that reflects their data contributions. This online approach allows agents to develop strategies or policies for selectively sharing data, optimizing the utility they derive from the models they receive. Our framework comprises two key components, representing the agents and the arbiter, respectively, as outlined in the following paragraphs.

The agent component involves each agent utilizing a constant number of ranked partitions for its data, indicating that while the number of partitions remains unchanged, the size of each partition may vary over time. Different agents may have differing numbers of partitions. Our framework enables each agent to learn a parameterized stochastic policy to decide which partition of data to share with the arbiter at each time step. By incorporating stochasticity, the policy balances exploration and exploitation, allowing agents to occasionally select suboptimal partitions. This exploration is essential for discovering more effective data-sharing strategies, ultimately leading to improved long-term collaborative learning outcomes. At each time step, each agent categorize its data into partitions based on its perceived value, considering both the quality and quantity of the data. An agent may perceive the value of its data based on various factors such as its uniqueness, relevance to the task at hand, quality, and potential impact on the model's performance. Each agent assesses the value of its data independently, and the framework does not impose any specific guidelines for this evaluation. For instance, in a credit card fraud detection scenario, one agent might find a data partition containing predominantly fraudulent transactions to be the most valuable, while another agent might prioritize a data partition containing transactions primarily from high-risk merchant categories, such as online electronics retailers.

Each agent can develop distinct policies that align with their specific objectives. The model provided by the arbiter reflects the relative value of the agent's data concerning the collaborative model, in comparison to the data from other agents. To assess the benefits of sharing a particular partition, agents must evaluate the model they receive, with each agent possibly using a unique evaluation metric suited to their needs. Our framework enables each agent to fine-tune the received model using unshared data, and subsequently evaluate the fine-tuned model to quantify the advantages of sharing that specific data partition. Fine-tuning a model with unshared data gives agents a competitive advantage by allowing them to refine their models using additional data that others may not have access to. This improvement enhances model performance and enables agents to leverage more data, helping them maintain an edge while still participating in collaborative efforts.

In the context of the arbiter, it is essential to note that the data collected from various agents often displays distributional differences that are exacerbated by selective data sharing. In machine learning context, the true data distribution is often unknown and must be inferred from the available data. While we accommodate diverse data-sharing strategies among the agents, we must also consider the potential for malicious agents. To address the aforementioned issues, we propose a data weighting scheme that assigns relative weights to agents based on the value of their data contributions. While learning the collaborative model, each agent's data is weighted instead of simply averaged, ensuring that data from agents with more valuable contributions have a proportionately greater impact.

The weighting scheme also helps correct distributional differences among the data contributed by various agents. We derive these weights by utilizing a separate validation set for each agent, allocating a portion of their data to evaluate their contributions within the context of the collaborative model. This approach is crucial as it acknowledges the unique characteristics of each agent's data while assessing its influence on overall model performance. By validating against these individualized sets, we can accurately capture the distinct impact of each agent's data on the collaborative effort. Furthermore, if agents provide malicious data, the evaluation based on their own data could lead to lower assigned weights, discouraging such behaviour and reinforcing the integrity of the collaborative process. In the context of online learning, it is essential to account for the historical interactions of agents to accurately assess their value. An incremental weighting scheme can effectively incorporate the contributions from previous cycles, ensuring that the evaluation is not solely based on the current time step. This approach acknowledges the ongoing commitment and consistent quality of contributions from agents, leading to more reliable collaborative learning outcomes. We develop a bilevel optimization algorithm that concurrently learns both model parameters and agent-specific weights.

In collaborative data-sharing frameworks, incentivizing agents with tailored models encourages high-quality data contributions. Traditional approaches in data-centric CML often involve training separate models for each agent. In contrast, our method generates distinct models by introducing mean-zero noise to the collaborative model using agent-specific weights, thereby eliminating the need for separate training. These weights, which reflect each agent's relative impact, collectively form a distribution over the agents within the collaborative framework. To further motivate contributions, we apply distortion functions that adjust portions of this distribution, determining whom to nudge further—whether those with the highest, middle, or lowest contributions—and how much nudging is necessary. This targeted adjustment helps customize incentives, encouraging agents to contribute data that enhances the model's effectiveness and generalization capabilities.

\textbf{Related work}\\
Federated Learning (FL) \citep{yang19}, a widely used model-centric approach, addresses the iterative nature of collaborative learning by enabling agents to refine a global model through repeated parameter updates while preserving privacy by avoiding raw data sharing. However, FL often overlooks individual agent-specific goals and the diversity in data distributions across agents \citep{wu24, lin23,zhan22_fed,sim24_fed}. Data-centric CML typically employs offline algorithms that incentivize data sharing through equitable rewards \citep{sim20, sim24}, assuming agents act benevolently. These approaches fail to accommodate the self-interested nature of agents, who might strategically maximize their benefits, or to dynamically adapt to agents’ behaviour.

To address the aforementioned limitations, we propose a framework for data-centric collaborative learning that includes the following key contributions: (i) Self-interested agent modeling: our framework explicitly accounts for the self-interested nature of agents, allowing them to determine what data to share based on their individual criteria, while also addressing distributional differences due to agents' distinct data acquisition practices and self-interest. (ii) Policy learning for data sharing: agents learn a stochastic policy that optimizes their data-sharing strategies, aligning collaboration with their individual goals rather than assuming benevolent behavior. (iii) Incremental data handling: the framework operates online, addressing the incremental nature of data collection in real-world settings. (iv) Non-asymptotic convergence guarantees: we provide rigorous non-asymptotic analyses, ensuring convergence of both the agents' policy optimization and the arbiter's expected loss minimization to approximate stationary points.
%
\section{Preliminaries}
\label{sec:pblm}
We consider a family of predictors $\mathcal{F}\triangleq\{f_\theta:\mathcal{X}\to\mathcal{Y}|\theta\in\R^d\}$, each defined by its parameters $\theta\in\R^d$, and maps each input value $x$ in input space $\mathcal{X}$ to a corresponding target value $y$ in output space $\mathcal{Y}$. The loss function $l_\theta:\mathcal{X}\times\mathcal{Y}\to\R$ measures the discrepancy between the predicted value $f_\theta(x)$ and the true target value $y$.

The goal of each agent is to learn optimal parameters
 $\theta^*$ such that its resulting predictor $f_{\theta^*}$  exhibits effective generalization on unseen data. Specifically, let
\begin{align}
\label{eq:J}
J(\theta)\triangleq \E_{(x,y)\sim \mathscr{D}}\left [l_\theta(x,y)\right ],
\end{align}
where $\mathscr{D}$ represents the data distribution over $\mathcal{X}\times \mathcal{Y}$.
The goal is to solve the following optimization problem:
\begin{align}
\label{eq:opt}
\theta^* \in \argmin_{\theta\in \R^d} J(\theta),
\end{align}
by employing a version of the gradient descent algorithm which involves finding the optimal parameters $\theta^*$ s.t.~$\nabla_\theta J(\theta^*)=0$. To derive an expression for the gradient  $\nabla_\theta J(\theta)$, the following assumption is made about the loss function $l_\theta$:
\begin{assumption}
\label{as:grad_l_bound}
For all $\theta \in \R^d, (x,y) \in \mathcal{X}\times \mathcal{Y}$, $\left\lVert \nabla_\theta l_\theta(x,y) \right\rVert \leq L_l$;\; $0\leq L_l < \infty$.
\end{assumption}
For example, consider the weighted binary log loss function and the sigmoid function as $f_\theta(x)$, defined below:
\begin{align}
   & f_\theta(x) =\nicefrac{1}{1+exp(-\theta^\top x)}. \label{eq:sim_f}\\
    &l_\theta(x,y) \nonumber = -\gamma y\log\left(f_\theta(x) \right)\nonumber\\
    &\;\;-(1-\gamma)(1-y)\log\left(1-f_\theta(x) \right) , \gamma \in (0,1), \textrm{ and}\label{eq:sim_l}\\
    &\nabla_\theta l_\theta(x,y) = \left(\left(\gamma y + (1\!-\!\gamma) (1-y)\right)f_\theta(x) -\gamma y \right)x.\label{eq:sim_grad_l}
\end{align}
From  \eqref{eq:sim_grad_l}, we can see that the loss function in \eqref{eq:sim_l} satisfies Assumption \ref{as:grad_l_bound} if $\left\lVert x \right\rVert$ is bounded. The norm of the feature vector is usually bounded in practice.

Using Assumption \ref{as:grad_l_bound} and the dominated convergence theorem, an expression for $\nabla_\theta J(\theta)$ can be derived as follows:
\begin{align}
\label{eq:grad_J}
\nabla_\theta J(\theta) = \E_{(x,y)\sim \mathscr{D}}\left [\nabla_\theta l_\theta(x,y)\right].
\end{align}
Since the true data distribution
$\mathscr{D}$ is unknown, each agent must estimate the gradient $\nabla_\theta J(\theta)$ from their available data. With limited datasets, agents benefit from collaborating with others for broader data access, which improves model accuracy and enhances generalization. To facilitate this collaboration, we introduce a batch online CML framework, which is described in the following section.

\section{Batch Online CML Framework}
\label{sec:frmk}
We consider collaborative learning among
$N=\{1,\ldots,n\}$ self-interested agents, where the collaboration is facilitated by an arbiter. We process data in batches, where each batch represents a finite subset of the underlying space of input-output pairs, $\mathcal{X}\times\mathcal{Y}$. At each time step, the arbiter collects a batch of data from each agent, trains a machine learning model using the aggregated data, and derives a unique model for each agent commensurate with their individual contributions, then returns the model to the respective agent.
In the following sections, we elaborate on our design approach involving both the agents and the arbiter. Figure \ref{fg:cml} provides an illustration of our  framework.
\begin{figure}
    \centering
    \includegraphics[width=\columnwidth]{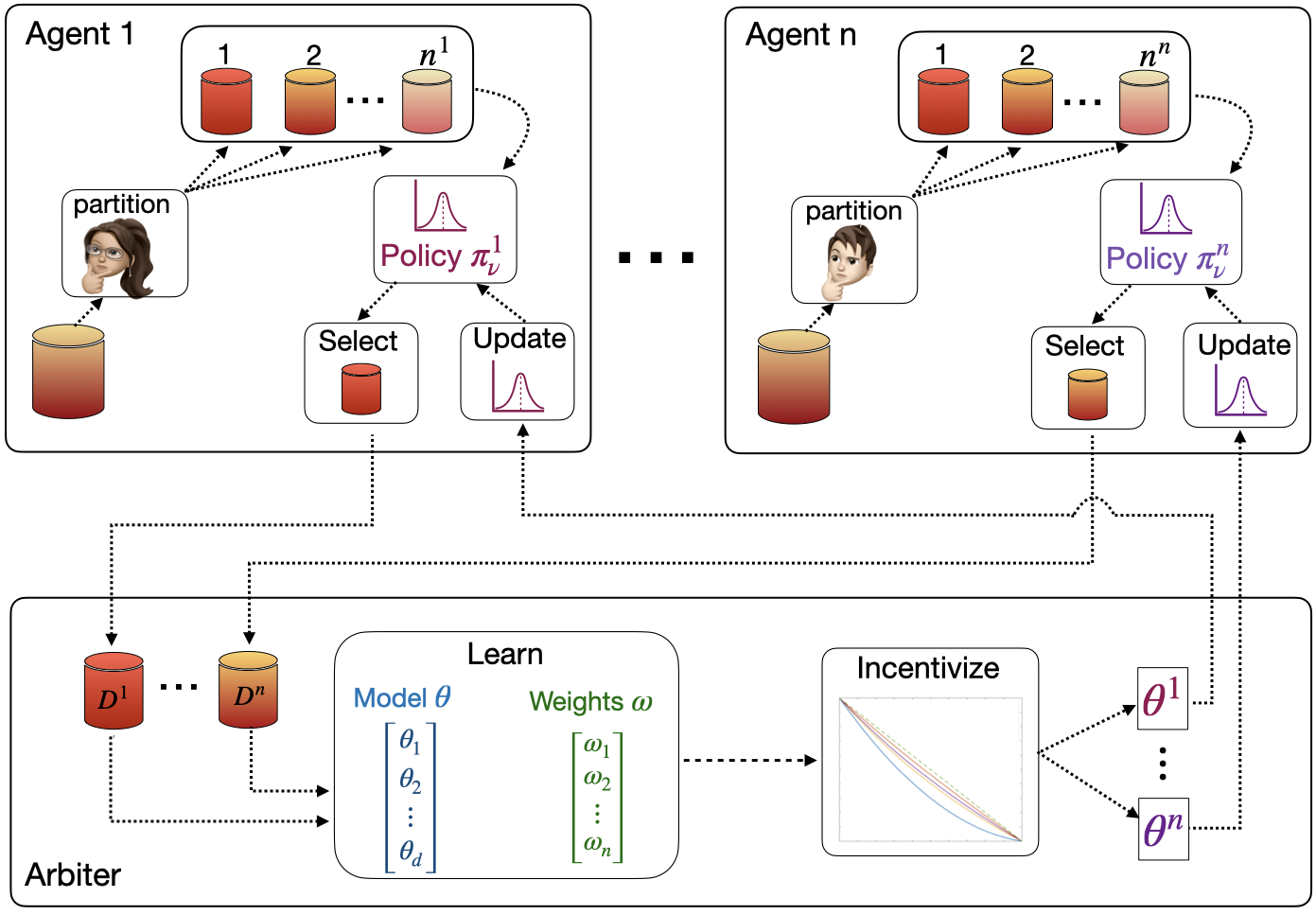}
    \caption{An illustration of the design approach.}
    \label{fg:cml}
\end{figure}

\subsection{Agents}
\label{subsec:agents}

A distribution $\mathcal{D}^i$ represents how the agent obtains batches from the global space $\mathcal{X}\times\mathcal{Y}$, capturing its distinct data acquisition practices, influenced in part by considerations of self-interest.
While $\mathcal{D}^i$ governs the distribution over subsets of data (batches), the global data distribution $\mathscr{D}$ governs the distribution over individual data points. Despite operating at different levels, $\mathcal{D}^i$ over subsets and $\mathscr{D}$ over individual points, the two distributions are intrinsically connected, as the subsets governed by $\mathcal{D}^i$ are composed of data points drawn from the same global data space as the individual points governed by $\mathscr{D}$.

Due to self-interest and uncertainty about the relevance of others' data, agents engage in selective sharing, fearing exposure of sensitive information or loss of competitive advantages. We define a strategy set $S^i = \{1, \ldots, n^i\}$, and agent $i$ divides its data $D\sim\mathcal{D}^i$ into $n^i$ partitions, $\{D_j\}_{j\in S^i}$, prioritizing data based on their own criteria, where $D_1$ contains the most significant data and $D_{n^i}$ the least. The agent has the flexibility to determine its own partitioning criteria, with the framework not imposing any restrictions or directives on how the agent organizes its data.


 We introduce a set of parameterized stochastic policies $\{\pi^i_{\nu}: S^i \to [0,1] \mid \nu \in \mathbb{R}^{n^i}\}$, which guide agent $i$ in selecting an optimal strategy $s^i \in S^i$ for choosing and sharing the data $D_{s^i}$ with the arbiter. An example of such a policy is the softmax policy, which is described below:
\begin{align}
\label{eq:agt_pi}
\pi^i_{\nu}(s)= \nicefrac{exp(\nu(s))}{\sum_{j\in S^i}exp(\nu(j))},\;\forall s\in S^i, \nu \in\R^{n^i}.
\end{align}
Agents, with no knowledge of other agents' policies and relying solely on the model received from the arbiter, face uncertainty about the contributions of other agents. Parameterized stochastic policies enable probabilistic selection of data partitions, balancing exploration and exploitation. Adaptive parameter learning is essential, as it allows the policy to continuously adjust based on the model feedback from the arbiter, improving the agent’s policy and ensuring that the policy remains effective as the behaviour and contributions of other agents evolve.


We optimize the policy parameter $\nu$ by maximizing the expected model performance, evaluated across various strategies from $S^i$. Model performance is evaluated using the function $m_\theta^i$, which measures the effectiveness of model $\theta$ on a validation set. Examples include accuracy, precision, recall, mean-squared error, etc., with agent $i$ selecting an evaluation function based on its preferences.



Specifically, we learn an optimal policy parameter $\nu^*$ by solving the optimization problem
\begin{align}
&\nu^* \in \argmin_{\nu\in\R^{n^i}} G^i(\nu),  \label{eq:agt_opt}\\
&G^i(\nu)= \E_{D\sim \mathcal{D}^i,\, s \sim \pi^i_{\nu}}\big[m^i_{\tilde{\theta}_{s}}(\bar{V}_s)\big], \textrm{ where}\label{eq:agt_G}
\end{align}
\begin{align}
\label{eq:agt_tv}
&D=\{D_j\}_{j\in S^i};\; \forall s \in S^i,\, \bar{D}_{s}=D \setminus D_s,\,p\in(0.5,1),\nonumber\\
& \bar{T}_{s} \subset \bar{D}_{s}\, s.t.\, |\bar{T}_s| = \left\lfloor p| \bar{D}_{s}|\right\rfloor,\,
\bar{V}_{s}= \bar{D}_{s} \setminus \bar{T}_{s}.
\end{align}
In \eqref{eq:agt_tv}, $\bar{T}_{s}$ and $\bar{V}_{s}$ represent the training and validation set, respectively. As we process data sequentially in batches, using a non-fixed validation set that adapts with each batch ensures that the model evaluation adapts with the evolving data, allowing for more accurate policy learning. In \eqref{eq:agt_G}, the model $\tilde{\theta}_{s}$ is obtained by fine-tuning, using $\bar{T}_s$, the parameters $\theta^i$ that agent $i$ received from the arbiter after sharing its dataset $D_{s}$.

We solve the optimization problem in \eqref{eq:agt_opt} using a policy gradient algorithm.
To derive an expression for the gradient $\nabla_{\nu}G(\nu^i)$,
we make the following assumption about the norm of the function $m^i_{\theta}$ and on the policy $\pi^i_{\nu}$:
\begin{assumption}
\label{as:agt_m_bound}
For a given $i$, $\forall D \sim \mathcal{D}^i,\, \forall s \in S^i$, $\big\lvert m^i_{\tilde{\theta}_{s}}(\bar{V}_s) \big\rvert \leq M_m^i$ a.s.; $0\leq M_m^i<\infty$.
\end{assumption}
A few examples of  $m^i_{\theta}$ that satisfy Assumption \ref{as:agt_m_bound} are accuracy, precision, recall and F1-score.
\begin{assumption}
\label{as:agt_policy_1}
For a given $i$, $\forall \nu \in \R^{n^i}$, $\forall s \in S^i$, $\left\lVert  \nabla_{\nu}\log\pi^i_{\nu}(s)\right\rVert \leq M_d^i$; $0 \leq M_d^i <\infty $.
\end{assumption}
For example, if we use softmax policy given in $\eqref{eq:agt_pi}$, then
\begin{align}
 \label{eq:agt_pi_log_grad}
 \nabla_{\nu}\log\pi^i_{\nu}(s) = \mathbf{1}_s - \pi^i_{\nu},
\end{align}
where $\mathbf{1}_s$ is the indicator vector with $1$ at index $s$ and $0$ elsewhere. From  \eqref{eq:agt_pi_log_grad}, we can see that softmax policy in $\eqref{eq:agt_pi}$ satisfies Assumption \ref{as:agt_policy_1}.

We obtain an expression for the gradient $\nabla_{\nu}G(\nu)$ in the lemma below, and the proof is available in Appendix \ref{subsec:agents_pf}.
\begin{lemma}
\label{lm:agt_grad}
Let Assumptions \ref{as:agt_m_bound}--\ref{as:agt_policy_1} hold. Then,
\begin{align}
\label{eq:agt_grad}
&\nabla_{\nu}G^i(\nu) = \E_{D\sim \mathcal{D}^i,\,s \sim \pi^i_{\nu}}\big[ m^i_{\tilde{\theta}_{s}}(\bar{V}_s) \nabla_{\nu} \log\pi^i_{\nu}(s)\big].
\end{align}
\end{lemma}

At time step $t$, agent $i$ has data $D_t=\{D_{t,j}\}_{j\in S^i}$, and selects a strategy $s_t\sim \pi^i_{\nu_t}$, and shares the data partition $D_{t,s_t}$ with the arbiter, and in return receives a unique model parameter $\theta^i_{t+1}$. We construct $\bar{T}_{s_t}$ and $\bar{V}_{s_t}$ as in \eqref{eq:agt_tv}, and fine-tune $\theta^i_t$
to $\tilde{\theta}_{s_t}$. Then, we estimate the gradient $\nabla_{\nu}G(\nu)$ in \eqref{eq:agt_grad} as
\begin{align}
\label{eq:agt_nablahat_G}
\widehat{\nabla}_{\nu}G^i(\nu_t)
= m^i_{\tilde{\theta}_{s_t}}(\bar{V}_{s_t}) \nabla_{\nu} \log\pi^i_{\nu_t}(s_t),
\end{align}
and learn $\nu_{t+1}$ as given below:
\begin{align}
\label{eq:agt_nu_t}
\nu_{t+1} = \nu_{t} - b_t \widehat{\nabla}_\nu G^i(\nu_{t}), \; b_t \in (0,1),
\end{align}
where $\nu_0$ is set arbitrarily.
Algorithm~\ref{alg:agt} outlines the agent-specific component of the CML framework.
\begin{algorithm}
    \caption{CML-Agent}
    \label{alg:agt}
    \begin{algorithmic}[1]
        \STATE \textbf{Initialize}: $\nu_0$;
        \FOR {$t=0,1,\hdots$ }
        \STATE Get $D_t \sim \mathcal{D}^i$;
        \STATE Partition $D_t$ into $\{D_{t,j}\}_{j\in S^i}$;
        \STATE $s_t\sim \pi_{\nu_t}(\cdot)$;
        \STATE Share $D_{t,s_t}$ with arbiter;
        \STATE $\bar{D}_{t,s_t}= D_t \setminus D_{t,s_t}$;
        \STATE $\bar{T}_{t,s_t} \subset \bar{D}_{t,s_t} \,s.t.\, |\bar{T}_{t,s_t}| = \left\lfloor p| \bar{D}_{t,s_t}|\right\rfloor$;
        \STATE $\bar{V}_{t,s_t} = \bar{D}_{t,s_t} \setminus \bar{T}_{t,s_t}$;
        \STATE Divide $\bar{T}_{t,s_t}$ into $n_b^i$ batches $\{T_t^0,\ldots,T_t^{n_b^i-1}\}$;
        \STATE Get $\theta_{t+1}^i$ from the arbiter;
        \STATE $\tilde{\theta}^0=\theta_{t+1}^i$;
         \FOR {$k=0,\hdots, n^i_b-1$ }
         \STATE $\tilde{\theta}^{k+1}=\tilde{\theta}^k - a_{k} \frac{1}{|T_t^k|}\sum\nolimits_{(x,y)\in T_t^k}\nabla_{\tilde{\theta}}l_{\tilde{\theta}^k}(x,y)$;
         \ENDFOR
        \STATE $\tilde{\theta}_{s_t}= \tilde{\theta}^{n^i_b}$;
        \STATE $\nu_{t+1} = \nu_{t} - b_t \widehat{\nabla}_\nu G^i(\nu_{t})$;
        \ENDFOR
    \end{algorithmic}
\end{algorithm}
\subsection{Arbiter}
\label{subsec:arbiter}

Let $\mathcal{D}^{\theta}=\{\mathcal{D}^{\theta^i}\}_{i\in N}$, where  $\mathcal{D}^{\theta^i}$ represents the distribution over the data batches shared by agent $i$. The distribution $\mathcal{D}^{\theta^i}$  deviate from the agent's local data batch distribution $\mathcal{D}^i$ that we discussed in Section \ref{subsec:agents}, because agent $i$ does not share all data in $D\sim \mathcal{D}^i$, but instead selects and shares a subset of $D$, based on its sharing policy $\pi^i$. Consequently, the data received by the arbiter does not directly follow the global data distribution $\mathscr{D}$. Instead it reflects a combination of each agent's local distribution, $\mathcal{D}^i$, and its data-sharing policy $\pi^i$, which depends on the shared model parameters.

Our goal is to solve the optimization problem in \eqref{eq:opt}, by deriving an estimate of the gradient in \eqref{eq:grad_J} by utilizing the batch data that follows the distribution $\mathcal{D}^{\theta}$. Since the underlying distributions are unknown and must be estimated from the available data. To account for distributional differences and effectively model the data, we employ a data weighting scheme that assigns suitable weights to each agents' data. Since the policy $\pi^i$ is learned based on the models received from the arbiter, a feedback loop naturally arises between the data shared by the agents and the models provided by the arbiter. This dynamic interplay motivates the need for a learning algorithm to determine agent-specific weights, ensuring that the model adapts to the evolving data-sharing policies of the agents.

Specifically, we consider parameterized weight functions $\{h_\omega:N\to[0,1]|\omega\in\R^n\}$ that assigns relative weight to each agent $i$, as defined below:
\begin{align}
\label{eq:abt_h}
 h_{\omega}(i)= \nicefrac{exp(\omega(i))}{\sum_{j=1}^{n}exp(\omega(j))},\;\forall i\in N, \omega\in\R^n.
\end{align}
Inorder to jointly learn the model parameters $\theta$ and agent-specific weights $\omega$, we define $L(\theta,\omega,D)$ and $M(\theta,\omega,D)$, as given below:
\begin{align}
&\forall D \sim \mathcal{D^{\theta}}, \,D = \{D^i\}_{i\in N},\nonumber\\
&L(\theta,\omega,D)=\big\langle\big(\E_{\scalebox{0.9}{$\substack{T^i\sim \mathcal{U}(D^i,\lceil p|D^i|\rceil)\\(x^i,y^i)\sim \mathcal{U}(T^i,1)}$}}[l_{\theta}(x^i,y^i)]\big)_{\!i=1}^{\!n},h_\omega   \big\rangle,\label{eq:L}\\
&M(\theta,\omega,D)
=\big\langle\big(\E_{\scalebox{0.9}{$\substack{V^i\sim \mathcal{U}(D^i,\lfloor (1-p)|D^i|\rfloor)\\(x^i,y^i)\sim \mathcal{U}(V^i,1)}$}}[m_{\theta}(x^i,y^i)]\big)_{\!i=1}^{\!n} ,\omega \big\rangle\nonumber\\
&\quad\quad +(\nicefrac{\lambda_\omega}{2})\left\lVert \omega \right\rVert^2,\;\lambda_\omega\in (0,0.5],\,p\in(0.5,1)\label{eq:M}
\end{align}
where $\mathcal{U}(A,a)$ denotes the uniform distribution over subsets of $A$ of size $a$. The notation $(a_i)_{i=1}^{n}$ represents a vector of size $n$, where each element of the vector is denoted by $a_i$.

In \eqref{eq:L}, the function $l_\theta(\cdot)$ is as defined in Section \ref{sec:pblm}, and $m_\theta:\mathcal{X}\times\mathcal{Y}\to\R$ is a function that evaluate the learned model $\theta$. The evaluation function used by the arbiter differ from
the evaluation functions
$m^i_\theta(\cdot)$ used by agents to learn the policy $\pi^i$ in Section \ref{subsec:agents}, where each agent may choose its own evaluation function independently, without requiring consensus on a common evaluation function. The arbiter has no knowledge of the evaluation functions used by agents to learn their policies. Instead, the arbiter utilizes its own evaluation function, such as a sigmoid function or mean-squared error, to determine the agent-specific weights $\omega$. In \eqref{eq:abt_tv}, $T^i$ and $V^i$ represent the training and the validation sets, respectively. By evaluating the model on agent-specific validation sets, we can assess how well the model generalizes to each agent's data using the evaluation function $m_\theta(\cdot,\cdot)$. This enables the adjustment of agent-specific weights to better align with the model's performance on each agent's data. As we process data sequentially in batches, a fixed validation set may not provide an accurate evaluation of model performance due to the changing characteristics of the data. Using a non-fixed validation set that adapts with each batch ensures that model evaluation remains representative of the most recent data.

To obtain an for the gradient $\nabla_{\omega}M(\theta,\omega,D)$, we make the following assumption:
\begin{assumption}
\label{as:abt_m_bound_1}
$\forall \theta \in \R^d,\,\forall (x,y) \in \mathcal{X}\times\mathcal{Y}$, $\left\lvert m_\theta(x,y) \right\rvert \leq M_m$; $0\leq M_m <\infty$.
\end{assumption}
An example of  $m_\theta(\cdot,\cdot)$ that satisfies Assumption \ref{as:abt_m_bound_1} is sigmoid function, which is defined as below.
\begin{align}
    \label{eq:sim_abt_m}
    m_\theta(x,y) =1-\nicefrac{1}{1+\exp\left(-y f_\theta(x)\right)},
\end{align}
where predictor $f_\theta(x)$ is as defined in \eqref{eq:sim_f}.

We obtain an expression for the gradients $\nabla_{\theta} L(\theta,\omega,D)$ and $\nabla_{\omega}M(\theta,\omega,D)$ in the lemma below, and the proof is available in Appendix \ref{subsec:arbiter_pf}.
\begin{lemma}
\label{lm:abt_grad}
Let Assumptions \ref{as:grad_l_bound} and \ref{as:abt_m_bound_1} hold. Then,
\begin{align}
&\nabla_{\theta} L(\theta,\omega,D)
=\big[ \E_{\scalebox{0.9}{$\substack{T^i\sim \mathcal{U}(D^i,\lceil p|D^i|\rceil)\\(x^i,y^i)\sim \mathcal{U}(T^i,1)}$}}[\nabla_\theta l_{\theta}(x^i,y^i)]\big]_{i=1}^{n}
h_\omega, \label{eq:grad_L}\\
&\nabla_{\omega}M(\theta,\omega,D)\nonumber\\
&= \big(\E_{\scalebox{0.9}{$\substack{V^i\sim \mathcal{U}(D^i,\lfloor (1-p)|D^i|\rfloor)\\(x^i,y^i)\sim \mathcal{U}(V^i,1)}$}}[m_{\theta}(x^i,y^i)]\big)_{\!i=1}^{\!n}
+ \lambda_\omega \omega.\label{eq:grad_M}
\end{align}
\end{lemma}
In the above, the notation $[a_i]_{i=1}^{n}$ denotes a matrix with $n$ columns, where each column is represented by $a_i$.
We assume the following to solve the optimization in \eqref{eq:opt}.
\begin{assumption}
\label{as:cent_w*}
 There exists a function $\omega^*:\R^d\to \mathcal{W}\subset\R^n;\,\mathcal{W}=\{\omega\in\R^n:\lVert \omega \rVert\leq M_{\omega^*}, 0 \leq M_{\omega^*} < \infty\}$ s.t. for a given $\theta$, $\omega^*(\theta)$ is the unique solution to $\E_{D\sim\mathcal{D^{\theta}}}\left[\nabla_{\omega}M(\theta,\omega^*(\theta),D)\right]=0$, and $\nabla_\theta J(\theta)= \E_{D\sim\mathcal{D^{\theta}}}\left[\nabla_{\theta} L(\theta,\omega^*(\theta),D)\right]$.
\end{assumption}
Assumptions similar to Assumption \ref{as:cent_w*} are often used in the theoretical analysis of bilevel optimization algorithms \citep{doan23,zeng22}.

Solving the optimization problem in \eqref{eq:opt} by finding a point $\theta^*$, where $\nabla_\theta J(\theta^*)=0$ is equivalent to finding a point $(\theta^*, \omega^*(\theta^*))$ that satisfies
\begin{align}
\label{eq:abt_opt_pt}
&\E_{D\sim\mathcal{D^{\theta}}}\left[\nabla_{\theta} L(\theta^*,\omega^*(\theta^*),D)\right]=0 \textrm{ and } \nonumber\\
&\E_{D \sim \mathcal{D^{\theta}}}\left[\nabla_{\omega}M(\theta^*,\omega^*(\theta^*),D)\right]=0.
\end{align}

At time $t$, the arbiter receives the data $D_t=\{D^i_t\}_{i\in N}$. Then $\forall i \in N$, we construct a training set $T^i_t$, and a validation set $V^i_t$ as given below:
\begin{align}
\label{eq:abt_tv}
T^i_t \subset D^i_t \,s.t.\, |T^i_t| = \left\lfloor p|D^i_t|\right\rfloor,\,
 V^i_t= D^i_t \setminus T^i_t,\,p\in(0.5,1).
\end{align}
We further divide the training set $T^i_t$ into $n_b$ batches $\{T^i_{t,0},\ldots,T^i_{t,n_b-1}\}$, and learn $\theta$ and $\omega$ are as given below:
\begin{align}
&\theta_{t,0} = \theta_t, \;k=0,\nonumber\\
&\begin{rcases*}
\theta_{t,k+1}=\theta_{t,k}-\alpha_{t} \widehat{\nabla}_\theta L(\theta_{t,k},\omega_t,D_t),\\
k=k+1,\\
\end{rcases*}
{n_b \textrm{ times}}\nonumber\\
&\theta_{t+1}= \theta_{t,n_b},\;\theta_0 \textrm{ is set arbitrarily},\; \alpha_{t} \in (0,1) \label{eq:abt_theta_it};\\
&\omega_{t+1}=\omega_{t}-\beta_t \widehat{\nabla}_\omega M(\theta_{t+1},\omega_t,D_t),\;\omega_0=\mathbf{0},\; \beta_{t} \in (0,1)\label{eq:abt_omega_it}
\end{align}
We start with a weight $\omega_{0}=\mathbf{0}$, so that every agent get equal priority when we start our learning.
Here, $\widehat{\nabla}\theta L(\cdot,\cdot,\cdot)$ and $\widehat{\nabla}\omega M(\cdot,\cdot,\cdot)$ are estimators of $\nabla_\theta L(\cdot,\cdot,\cdot)$ and $\nabla_\omega M(\cdot,\cdot,\cdot)$, respectively, as detailed below.
\begin{align}
&\widehat{\nabla}_{\theta}L(\theta_{t,k},\omega_t,D_t)\nonumber\\
&=\big[\nicefrac{1}{|T^i_{t,k}|}{\textstyle\sum\nolimits_{(x^i,y^i)\in T^i_{t,k}}}\nabla_{\theta}l_{\theta_{t,k}}(x^i,y^i)\big]_{i=1}^{n} h_{\omega_t},\label{eq:abt_nablahat_L_theta}\\
&\widehat{\nabla}_{\omega}M(\theta_{t+1},\omega_t,D_t)\nonumber\\
&=\big(
\nicefrac{1}{|V_t^i|}{\textstyle\sum\nolimits_{(x^i,y^i)\in V_t^i}}m_{\theta_{t+1}}(x^i,y^i)\big)_{i=1}^{n}+\lambda_\omega \omega_t.\label{eq:abt_nablahat_M_omega}
\end{align}

We obtain distinct model parameters for each agent $i$ by introducing zero-mean noise $\epsilon(i)$ to the learned model parameter $\theta$, as defined below.
\begin{align}
\label{eq:abt_eps}
\eta^i\sim \mathbb{B}^d(1-g(h_{\omega}(i))),
\end{align}
where $\mathbb{B}^d(r)=\{v\in\R^d|\lVert v \rVert \leq r\}$ is a ball centered at origin with radius $r$. In the above, $g:[0,1]\to[0,1]$ is a distortion function which is non-decreasing with $g(0)=0$ and $g(1)=1$. Similar distortion functions are used in the literature to define risk measures \citep{denneberg1990,nv2023}. We use $g(\cdot)$ to nudge the weights $h_{\omega}(\cdot)$ to incentivize the agents to share more valuable data. A few examples of distortion functions are given in the Table \ref{tb:g} and their plots in Figure \ref{fg:g}.
\begin{table}
\begingroup
\small
\setlength{\tabcolsep}{4.75pt}
\begin{tabular}{ll}
\toprule
Quadratic function & $g(u)=(1+\lambda) u-\lambda u^2,\; 0\leq \lambda\leq 1$\\
Exponential function & $g(u)=\nicefrac{1-\exp(-\lambda u)}{1-\exp(-\lambda )},\; \lambda>0$\\
Square-root function & $g(u)=\nicefrac{\sqrt{1+\lambda u}-1}{\sqrt{1+\lambda }-1},\; \lambda>0$\\
Logarithmic function & $g(u)=\nicefrac{\log(1+\lambda u)}{\log(1+\lambda )}, \;\lambda>0$\\
\bottomrule
\end{tabular}
\endgroup
\caption{Examples of distortion functions}
\label{tb:g}
\end{table}
\begin{figure}
    \centering
    \begin{subfigure}{0.49\columnwidth}
    \includegraphics[width=\columnwidth]{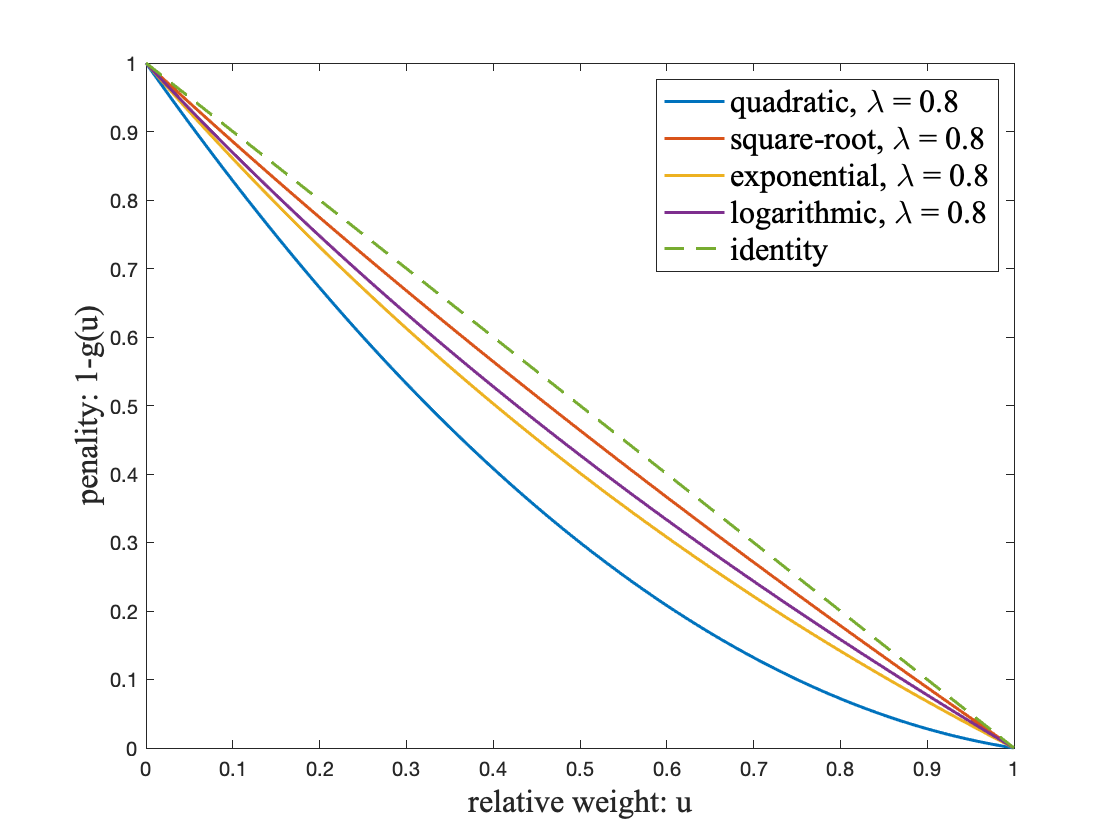}
     \end{subfigure}
     \hfill
    \begin{subfigure}{0.49\columnwidth}
    \includegraphics[width=\columnwidth]{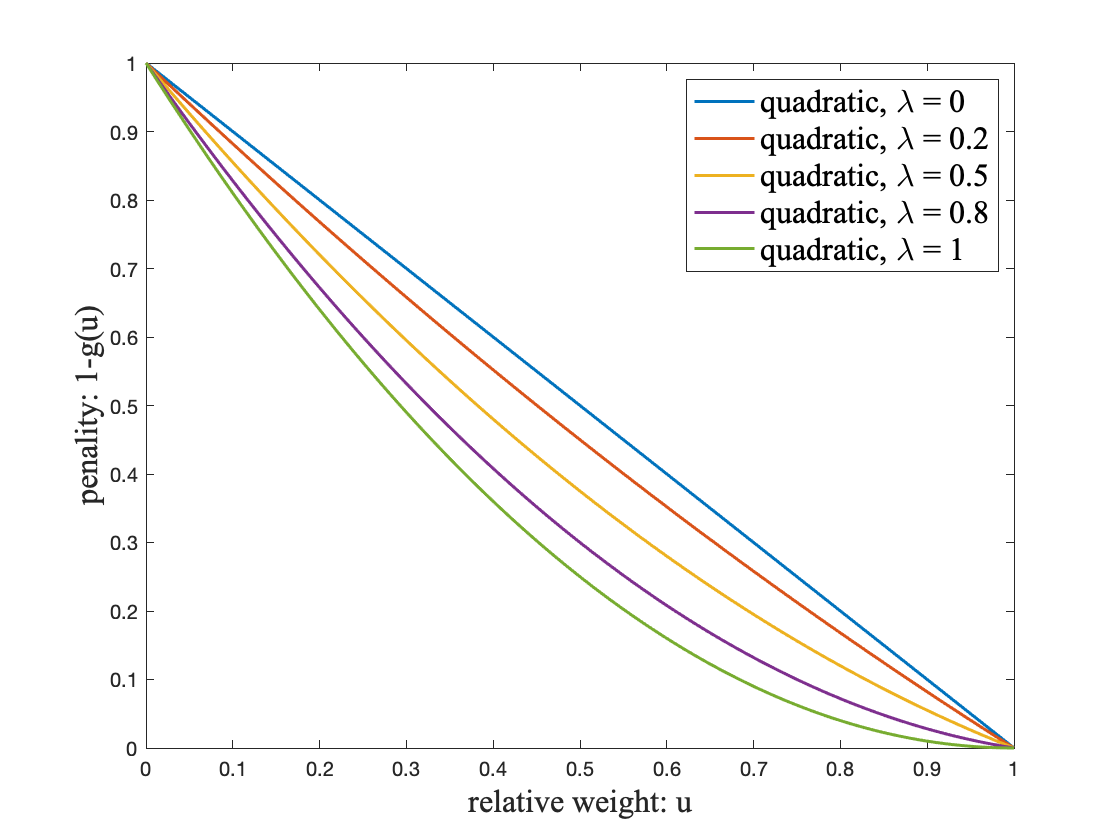}
    \end{subfigure}
    \caption{Examples of distortion functions}
    \label{fg:g}
\end{figure}
By adjusting the parameter $\lambda$, one can determine the degree of incentivization. If we employ the identity function as $g(\cdot)$, we rely solely on the relative weights without incorporating any extra incentives. From Figure \ref{fg:g}, we can see that distortion functions nudge each agent differently based on their relative weight. The chosen distortion function determines the extent of nudging applied to the relative weights. For instance, agents with weights in the middle to highest range may be nudged more significantly compared to those with nearly low weights.

We are sampling a vector uniformly at random from a ball of radius
$1-g(h_{\omega}(i))$ as noise $\eta^i$ which ensure  $\lVert \eta^i\rVert \leq 1-g(h_{\omega}(i))$.
We opt for a ball rather than a sphere for sampling noise to introduce variability in the length of the noise vector, with a maximum length of $1-g(h_{\omega}(i))$.
At time $t$, the arbiter learn the model parameter $\theta_{t+1}$, and calculate a unique $\theta^i_{t+1}$ for each agent $i$ as given below:
\begin{align}
\label{eq:abt_theta_rew}
\theta^i_{t+1} = \theta_{t+1} + \eta^i_{t+1}, \;\eta^i_{t+1}\sim \mathbb{B}^d(1-g(h_{\omega_{t+1}}(i))).
\end{align}
Algorithm~\ref{alg:abt} outlines the arbiter-specific component of the CML framework.
\begin{algorithm}
    \caption{CML-Arbiter}
    \label{alg:abt}
    \begin{algorithmic}[1]
        \STATE \textbf{Initialize}: $\theta_0$, $\omega_0=\mathbf{0}$;
        \FOR {$t=0,1,\hdots$ }
        \STATE Get data $\{D^i_t\}_{i\in N}$ from agents;
        \FOR {$i=1,\ldots,N$}
        \STATE $T^i_t \subset D^i_t \,s.t.\, |T^i_t| = \left\lfloor p|D^i_t|\right\rfloor$,\; $V^i_t= D^i_t \setminus T^i_t$;
        \STATE Divide $T^i_t$ into $n_b$ batches $\{T^i_{t,0},\ldots,T^i_{t,n_b-1}\}$;
        \ENDFOR
        \STATE $\theta_{t,0} = \theta_t$;
         \FOR {$k=0,\hdots, n_b-1$ }
        \STATE $\theta_{t,k+1}=\theta_{t,k}-\alpha_{t} \widehat{\nabla}_\theta L(\theta_{t,k},\omega_t,D_t)$;
         \ENDFOR
        \STATE $\theta_{t+1}= \theta_{t,n_b}$;
        \STATE $\omega_{t+1}=\omega_{t}-\beta_t \widehat{\nabla}_\omega M(\theta_{t+1},\omega_t,D_t)$;
        \FOR {$i=1,\hdots, n$ }
        \STATE $\theta^i_{t+1} = \theta_{t+1} + \eta^i_{t+1}$,\;$\eta^i_{t+1}\sim \mathbb{B}^d(1-g(h_{\omega_{t+1}}(i)))$;
        \STATE Send $\theta^i_{t+1}$ to agent $i$;
        \ENDFOR
        \ENDFOR
    \end{algorithmic}
\end{algorithm}

\section{Convergence Analysis}
\label{sec:conv}
We provide non-asymptotic convergence guarantees for our algorithms toward an $\epsilon$-stationary point, defined as follows:
\begin{definition}(\textbf{$\epsilon$-stationary point})
\label{def:esolution}
Consider an optimization problem $\theta^*\in\textrm{arg}\!\min_{\theta \in \mathbb{R}^d} f(\theta)$.
Fix $\epsilon>0$, and let $ \theta_R$ be the random output of an iterative algorithm designed to solve this problem. Then, $\theta_R$ is called an $\epsilon$-stationary point of the optimization problem, if $\mathbb{E}[\lVert \nabla f \left( \theta_R \right)\rVert^2] \leq \epsilon$, where the expectation is over $R$.
\end{definition}
We make the following additional assumptions for analysis.
\begin{assumption}
\label{as:grad2_l_bound}
For all $\theta \in \R^d, (x,y) \in \mathcal{X}\times \mathcal{Y}$, $\left\lVert \nabla_\theta^2 l_\theta(x,y) \right\rVert \leq L_{l'}$; $0\leq L_{l'}<\infty$.
\end{assumption}
For example, if we use  loss function as given in \eqref{eq:sim_l}, then
\begin{align}
    \nabla_\theta^2 l_\theta(x,y)
    &= \left(\gamma y + (1\!-\!\gamma) (1\!-\!y)\right)f_\theta(x)\left(1\!-\!f_\theta(x) \right)xx^\top\!\label{eq:sim_hess_l}.
\end{align}
From \eqref{eq:sim_hess_l}, we can see that the loss function in \eqref{eq:sim_l} satisfies Assumption \ref{as:grad2_l_bound} if $\left\lVert x \right\rVert$ is bounded. The norm of the feature vector is usually bounded in practice.
\begin{assumption}
\label{as:agt_policy_2}
For a given $i$, $\forall \nu \in \R^{n^i}$, $\forall s \in S^i$, $\left\lVert \nabla_{\nu}^2 \log\pi^i_{\nu}(s) \right\rVert \leq M_h^i$; $0 \leq M_h^i <\infty $.
\end{assumption}
For example, if we use the softmax policy given in \eqref{eq:agt_pi}, then
\begin{align}
  \label{eq:agt_pi_log_hess}
\nabla_{\nu}^2 \log \pi^i_{\nu}(s) = - \textrm{diag}(\pi^i_{\nu}) - \pi^i_{\nu} {\pi^i_{\nu}}^\top.
\end{align}
From \eqref{eq:agt_pi} and \eqref{eq:agt_pi_log_hess}, we can see that the softmax policy given in $\eqref{eq:agt_pi}$ satisfies Assumption \ref{as:agt_policy_2}.
\begin{assumption}
\label{as:abt_m_bound_2}
$\forall \theta \in \R^d$, $\forall (x,y) \in \mathcal{X}\times\mathcal{Y}$, $\left\lVert \nabla_\theta m_\theta(x,y) \right\rVert \leq L_m$; $0\leq L_m <\infty$.
\end{assumption}
For example, if we use  $m_\theta(\cdot,\cdot)$ as in \eqref{eq:sim_abt_m}, then
\begin{align}
    \label{eq:sim_abt_grad_m}
    &\nabla_\theta m_\theta(x,y) \nonumber\\
    &= -xy m_\theta(x,y) \left(1-m_\theta(x,y)\right) f_\theta(x) \left(1\!-\!f_\theta(x)\right).
\end{align}
From \eqref{eq:sim_abt_m} and \eqref{eq:sim_abt_grad_m}, we can see that the validation function in \eqref{eq:sim_abt_m} satisfies Assumption  \ref{as:abt_m_bound_2} if $\left\lVert x \right\rVert$ is bounded. The norm of the feature vector is usually bounded in practice.

We now present our main results regarding the convergence of Algorithms \ref{alg:agt} and \ref{alg:abt}, with proofs provided in Appendices \ref{subsubsec:pf_thm_agt} and \ref{subsubsec:pf_thm_abt}, respectively.
\begin{theorem}[\textbf{Agent $i$}]
\label{thm:agt}
Let Assumptions \ref{as:agt_m_bound}, \ref{as:agt_policy_1} and \ref{as:agt_policy_2} hold.
Let $\{\nu_t\}_{t=0}^{T^i-1}$ be the policy parameters generated by the Algorithm \ref{alg:agt} for $T^i$ iterations. Let $\p(R=t)=\nicefrac{1}{T^i}$, and the step size $b_t=\nicefrac{1}{\sqrt{T^i}},\, \forall t$. Then after $T^i$ iterations of the Algorithm \ref{alg:agt}, we have
\begin{align}
\label{eq:thm_agt}
&\E\big[ \left \lVert \nabla_\nu G^i(\nu_{R})  \right \rVert^2\big]\nonumber\\
&\leq \frac{G^i(\nu_{0})-G^{i^*}}{\sqrt{T^i}}+\frac{\big(M_h^i+M_d^{i^2}\big)M_m^{i^3}M_d^{i^2}}{2\sqrt{T^i}}.
\end{align}
In the above, $G^{i^*}=\argmin_{\nu\in \R^{n_s^i}} G(\nu)$, and the constants $M_d^{i}$, $M_h^i$ and $M_m^{i}$ are as defined in Assumptions \ref{as:agt_policy_1}, \ref{as:agt_policy_2} and \ref{as:agt_m_bound}, respectively.
\end{theorem}
\begin{remark}
\label{rmk:agt}
The result above shows that after $T^i$ iterations,  Algorithm \ref{alg:agt} returns an iterate that satisfies $\E\big[ \big \lVert \nabla_\nu G^i(\nu_{R})  \big \rVert^2\big]=O(\nicefrac{1}{\sqrt{T^i}})$.
Thus, $O(\nicefrac{1}{\epsilon^2})$ iterations of Algorithm \ref{alg:agt} are sufficient to find an $\epsilon$-stationary point.
\end{remark}
\begin{theorem}[\textbf{Arbiter}]
\label{thm:abt}
Let Assumptions \ref{as:grad_l_bound}, \ref{as:grad2_l_bound}, \ref{as:abt_m_bound_1}, \ref{as:abt_m_bound_2} and  \ref{as:cent_w*} hold. Let $\{\{\theta_{t,k}\}_{t=0}^{T-1}\}_{k=0}^{n_b-1}$ be the model parameters generated by the Algorithm \ref{alg:abt} for $T$ iterations. Let $\p(R_1=t)=\nicefrac{1}{T}$ and $\p(R_2=k)=\nicefrac{1}{n_b}$, and the step sizes $\alpha_t=\nicefrac{1}{T^{\nicefrac{3}{5}}};\;\beta_t=\nicefrac{1}{T^{\nicefrac{2}{5}}},\, \forall t$. Then after $T$ iterations of the Algorithm \ref{alg:abt}, we have
\begin{align*}
&\E\big[\left\lVert \nabla J(\theta_{R_1,R_2})\right\rVert^2 \big]
\leq \frac{2\left(J(\theta_{0}) - J^*\right)}{n_bT^{\nicefrac{2}{5}}}
\nonumber\\
&+\Big(\frac{A_1}{ T^{\nicefrac{4}{5}}}+\frac{A_2}{T^{\nicefrac{2}{5}}} +A_3\Big)\sum\nolimits_{i=1}^n\E\Big[\frac{1}{\lfloor (1-p)|D_{R_1}^i|\rfloor}\Big] \nonumber\\
& +\frac{A_4}{ T^{\nicefrac{2}{5}}} +\frac{A_5}{T^{\nicefrac{3}{5}}} + \frac{A_6}{T^{\nicefrac{4}{5}}} +\frac{A_7}{T^{\nicefrac{6}{5}}}+\frac{A_8}{T^{\nicefrac{7}{5}}}+\frac{A_9}{T^{\nicefrac{8}{5}}}+ \frac{A_{10}}{T^{2}},\\
&\textrm{where}\\
&A_1 = \frac{196 nL_l^2M_m^2}{\lambda_\omega},\; A_2 = \frac{84nL_l^2M_m^2}{\lambda_\omega^2}\\
&A_3 = \frac{32{L_l}^2}{\lambda_\omega^3 }\left(7nM_m^2+ \lambda_\omega^3e^2n_b\right),\\
&A_4 = \frac{7L_l^2n\left(7L_l^2 L_m^2n^2n_b^2+ 6\lambda_\omega^3\left(M_m^2n + \lambda_\omega^2M_{\omega^*}^2\right)\right)}{\lambda_\omega^5},\\
&A_5=\frac{7nL_l^2M_{\omega^*}^2}{\lambda_\omega^2},\;
A_6 = \frac{7n^2n_bL_l^4\left(\lambda_\omega^2 L_{l'}+2 L_m^2nn_b\right)}{2\lambda_\omega^4 }\\
&\qquad\qquad +  \frac{98L_l^2n \left(M_m^2 n+ 28\lambda_\omega^2M_{\omega^*}^2\right)}{\lambda_\omega },\\
&A_{7}= \frac{119n^3n_b^2L_l^4L_m^2}{\lambda_\omega^3},\;A_8= \frac{49 n^4n_b^2L_l^6L_m^2\left(1+n_b\right)}{4\lambda_\omega^4},\\
&A_9= \frac{21n^3n_b^2L_l^4L_m^2}{\lambda_\omega^2},\;A_{10}= \frac{49n^3n_b^2L_l^4L_m^2 }{\lambda_\omega}.
\end{align*}
In the above, $J^*=\argmin_{\theta\in\R^d} J(\theta)$, and the constants $L_l$, $L_{l'}$, $M_m$, $L_m$ and $M_{\omega^*}$ are as defined in Assumptions \ref{as:grad_l_bound}, \ref{as:grad2_l_bound}, \ref{as:abt_m_bound_1}, \ref{as:abt_m_bound_2} and \ref{as:cent_w*}, respectively. The parameters $p\in(0.5,1)$ from \eqref{eq:abt_tv}, $\lambda_\omega\in (0,0.5]$ from \eqref{eq:M}, and $n$ is the number of collaborating agents.
\end{theorem}
\begin{remark}
\label{rmk:abt}
The result above shows that after $T$ iterations,  Algorithm \ref{alg:abt} returns an iterate that satisfies $\E[\lVert \nabla J(\theta_{R_1,R_2})\rVert^2]= O(\nicefrac{1}{{T}^{\nicefrac{2}{5}}})+O(\sum_{i=1}^n\E[\nicefrac{1}{\lfloor (1-p)|D_{R_1}^i|\rfloor}])$. Hence, the convergence rate is influenced by the expected data size from the agents.
\end{remark}

\section{Conclusion and Future work}
\label{sec:conc}
We have presented a novel framework for adaptive data-centric collaborative learning, explicitly designed to accommodate the self-interested nature of agents and the incremental nature of data collection. By incorporating a feedback loop between data-sharing policies and model updates, the framework aligns collaborative strategies with agent-specific goals through the optimization of stochastic policies. The arbiter’s bilevel optimization approach addresses distributional differences and ensures efficient utilization of shared data. Additionally, the use of a distortion-based mechanism to generate distinct models fosters meaningful collaboration without added training complexity. Rigorous non-asymptotic analyses guarantee convergence to approximate stationary points for both agent-side and arbiter-side optimization objectives, ensuring the robustness and practicality of the proposed approach in dynamic, real-world environments. This work advances the understanding of collaborative learning frameworks, highlighting the importance of incentivizing agents while respecting their strategic interests and decision-making flexibility.

As future work, incorporating differential privacy mechanisms into the framework could further enhance its practicality, particularly in settings where sensitive data is involved. By introducing rigorous privacy guarantees, the framework could ensure that agents' shared data cannot be exploited to infer sensitive information, thereby encouraging greater participation. Differential privacy could be integrated alongside the current structure to balance the trade-off between data utility and privacy preservation, enabling secure and equitable collaborative learning.

\bibliography{ref}

\begin{thebibliography}{14}
\providecommand{\natexlab}[1]{#1}
\providecommand{\url}[1]{\texttt{#1}}
\expandafter\ifx\csname urlstyle\endcsname\relax
  \providecommand{\doi}[1]{doi: #1}\else
  \providecommand{\doi}{doi: \begingroup \urlstyle{rm}\Url}\fi

\bibitem[Denneberg(1990)]{denneberg1990}
Denneberg, D.
\newblock Distorted probabilities and insurance premiums.
\newblock \emph{Methods of Operations Research}, 63\penalty0 (3):\penalty0
  3--5, 1990.

\bibitem[Doan(2023)]{doan23}
Doan, T.~T.
\newblock Nonlinear two-time-scale stochastic approximation: Convergence and
  finite-time performance.
\newblock \emph{IEEE Transactions on Automatic Control}, 68\penalty0
  (8):\penalty0 4695--4705, 2023.

\bibitem[Gao \& Pavel(2018)Gao and Pavel]{gao18}
Gao, B. and Pavel, L.
\newblock On the properties of the softmax function with application in game
  theory and reinforcement learning, 2018.

\bibitem[Hayes(2005)]{hayes2005large}
Hayes, T.
\newblock A large-deviation inequality for vector-valued martingales.
\newblock \emph{Combinatorics, Probability and Computing}, 2005.

\bibitem[Lin et~al.(2023)Lin, Xu, Ng, Foo, and Low]{lin23}
Lin, X., Xu, X., Ng, S.~K., Foo, C.~S., and Low, B. K.~H.
\newblock Fair yet asymptotically equal collaborative learning.
\newblock In \emph{Proceedings of the 40th International Conference on Machine
  Learning (ICML-23)}, 2023.

\bibitem[Nesterov(2014)]{nesterov_book_1}
Nesterov, Y.
\newblock \emph{Introductory Lectures on Convex Optimization: A Basic Course}.
\newblock Springer Publishing Company, Incorporated, 1 edition, 2014.

\bibitem[Sim et~al.(2020)Sim, Zhang, Chan, and Low]{sim20}
Sim, R. H.~L., Zhang, Y., Chan, M.~C., and Low, B. K.~H.
\newblock Collaborative machine learning with incentive-aware model rewards.
\newblock In \emph{International conference on machine learning}, pp.\
  8927--8936. PMLR, 2020.

\bibitem[Sim et~al.(2024{\natexlab{a}})Sim, Tay, Xu, Zhang, Wu, Lin, Ng, Foo,
  Jaillet, Hoang, and Low]{sim24_fed}
Sim, R. H.~L., Tay, S.~S., Xu, X., Zhang, Y., Wu, Z., Lin, X., Ng, S.-K., Foo,
  C.-S., Jaillet, P., Hoang, T.~N., and Low, B. K.~H.
\newblock Chapter 16 - incentives in federated learning.
\newblock In \emph{Federated Learning}, pp.\  299--309. Academic Press,
  2024{\natexlab{a}}.

\bibitem[Sim et~al.(2024{\natexlab{b}})Sim, Zhang, Hoang, Xu, Low, and
  Jaillet]{sim24}
Sim, R. H.~L., Zhang, Y., Hoang, T.~N., Xu, X., Low, B. K.~H., and Jaillet, P.
\newblock Incentives in private collaborative machine learning.
\newblock In \emph{Proceedings of the 37th International Conference on Neural
  Information Processing Systems}, NIPS '23. Curran Associates Inc.,
  2024{\natexlab{b}}.

\bibitem[Vijayan \& Prashanth(2023)Vijayan and Prashanth]{nv2023}
Vijayan, N. and Prashanth, L.
\newblock A policy gradient approach for optimization of smooth risk measures.
\newblock In \emph{Proceedings of the Thirty-Ninth Conference on Uncertainty in
  Artificial Intelligence}, volume 216 of \emph{PMLR}, pp.\  2168--2178, 2023.

\bibitem[Wu et~al.(2024)Wu, Amiri, Raskar, and Low]{wu24}
Wu, Z., Amiri, M.~M., Raskar, R., and Low, B. K.~H.
\newblock Incentive-aware federated learning with training-time model rewards.
\newblock In \emph{The Twelfth International Conference on Learning
  Representations}, 2024.

\bibitem[Yang et~al.(2019)Yang, Liu, Chen, and Tong]{yang19}
Yang, Q., Liu, Y., Chen, T., and Tong, Y.
\newblock Federated machine learning: Concept and applications.
\newblock \emph{ACM Trans. Intell. Syst. Technol.}, 10\penalty0 (2), 2019.

\bibitem[Zeng et~al.(2024)Zeng, Doan, and Romberg]{zeng22}
Zeng, S., Doan, T.~T., and Romberg, J.
\newblock A two-time-scale stochastic optimization framework with applications
  in control and reinforcement learning.
\newblock \emph{SIAM Journal on Optimization}, 34\penalty0 (1):\penalty0
  946--976, 2024.

\bibitem[Zhan et~al.(2022)Zhan, Zhang, Hong, Wu, Li, and Guo]{zhan22_fed}
Zhan, Y., Zhang, J., Hong, Z., Wu, L., Li, P., and Guo, S.
\newblock A survey of incentive mechanism design for federated learning.
\newblock \emph{IEEE Transactions on Emerging Topics in Computing}, 10\penalty0
  (2):\penalty0 1035--1044, 2022.

\end{thebibliography}
\bibliographystyle{icml2025}
\newpage
\appendix
\onecolumn
\section{Proofs}
\subsection{Agents}
\label{subsec:agents_pf}
\begin{proof}{[\textbf{Lemma \ref{lm:agt_grad}}]}
\begin{align*}
\nabla_{\nu}G^i(\nu)&=
\nabla_{\nu} \E_{D\sim \mathcal{D}^i,\, s \sim \pi^i_{\nu}}\left[m^i_{\tilde{\theta}_{s}}(\bar{V}_s)\right]\\
&=\nabla_{\nu} \sum\limits_{s \in S^i}\E_{D\sim \mathcal{D}^i}\left[m^i_{\tilde{\theta}_{s}}(\bar{V}_s)\right] \pi^i_{\nu}(s)\\
&\stackrel{(a)}{=}\sum\limits_{s \in S^i} \E_{D\sim \mathcal{D}^i}\left[m^i_{\tilde{\theta}_{s}}(\bar{V}_s)\right]\nabla_{\nu} \pi^i_{\nu}(s)\\
&\stackrel{(b)}{=}\sum\limits_{s \in S^i} \E_{D\sim \mathcal{D}^i}\left[m^i_{\tilde{\theta}_{s}}(\bar{V}_s)\right]\nabla_{\nu} \log\pi_{\nu}(s) \pi^i_{\nu}(s)\\
&=\E_{D\sim \mathcal{D}^i,\,s \sim \pi^i_{\nu}}\left[m^i_{\tilde{\theta}_{s}}(\bar{V}_s) \nabla_{\nu} \log\pi^i_{\nu}(s)\right].
\end{align*}
In the above step $(a)$ follows by the application of dominated convergence theorem to interchange the expectation and differentiation operators. The aforementioned application is allowed since $S^i$ is a finite set and $\left\lvert m^i_{\tilde{\theta}_{s}}(\cdot)\right\rvert $ and $\left\lVert\nabla_{\nu} \log\pi^i_{\nu}(\cdot)\right\rVert$ are bounded from Assumption \ref{as:agt_m_bound} and Assumptions \ref{as:agt_policy_1}, respectively. The step $(b)$ follows from the fact that $\nicefrac{\nabla_{\nu} \pi^i_{\nu}(s)}{\pi^i_{\nu}(s)}=\nabla_{\nu}\log\pi^i_{\nu}(s)$.
\end{proof}
We present a few lemmas below, which are subsequently used in the proof of Theorem \ref{thm:agt}.
\begin{lemma}
\label{lm:agt_G_lip}
\begin{align*}
\textrm{For a given } i,\, \forall \nu_1, \nu_2 \in \R^{n^i},\; \left\lVert \nabla_\nu G^i(\nu_1)- \nabla_\nu G^i(\nu_2) \right\rVert\leq M_m^i\left(M_h^i+M_d^{i^2}\right)\left\lVert\nu_1-\nu_2 \right\rVert.
\end{align*}
\end{lemma}
\begin{proof}
\begin{align}
\label{eq:agt_hess1}
\nabla_{\nu}^2G^i(\nu)
&= \nabla_{\nu}^2 \E_{D\sim \mathcal{D}^i,\, s \sim \pi^i_{\nu}}\left[m^i_{\tilde{\theta}_{s}}(\bar{V}_s)\right]\\
&=\nabla_{\nu}^2 \sum\limits_{s \in S^i}\E_{D\sim \mathcal{D}^i}\left[m^i_{\tilde{\theta}_{s}}(\bar{V}_s)\right] \pi^i_{\nu}(s)\\
&\stackrel{(a)}{=}\sum\limits_{s \in S^i} \E_{D\sim \mathcal{D}^i}\left[m^i_{\tilde{\theta}_{s}}(\bar{V}_s)\right]\nabla_{\nu}^2 \pi^i_{\nu}(s)\nonumber\\
&\stackrel{(b)}{=}\sum\limits_{s \in S^i} \E_{D\sim \mathcal{D}^i}\left[m^i_{\tilde{\theta}_{s}}(\bar{V}_s)\right]\left(\nabla_{\nu}^2 \log\pi^i_{\nu}(s) + \nabla_{\nu}\log\pi^i_{\nu}(s)\nabla_{\nu}\log\pi^i_{\nu}(s)^\top\right) \pi^i_{\nu}(s)\nonumber\\
&=\E_{s \sim \pi^i_{\nu},\,D\sim \mathcal{D}^i}\left[m^i_{\tilde{\theta}_{s}}(\bar{V}_s) \left(\nabla_{\nu}^2 \log\pi^i_{\nu}(s) + \nabla_{\nu}\log\pi^i_{\nu}(s)\nabla_{\nu}\log\pi^i_{\nu}(s)^\top\right) \right].
\end{align}
In the above step $(a)$ follows by the application of dominated convergence theorem to interchange the expectation and differentiation operators. The aforementioned application is allowed since $S^i$ is a finite set and $\left\lvert m^i_{\tilde{\theta}_{s}}(\cdot)\right\rvert $, $\left\lVert\nabla_{\nu} \log\pi^i_{\nu}(\cdot)\right\rVert$ and $\left\lVert\nabla_{\nu}^2 \log\pi^i_{\nu}(\cdot)\right\rVert$ are bounded from Assumption \ref{as:agt_m_bound} and Assumptions \ref{as:agt_policy_1}--\ref{as:agt_policy_2}.
The step $(b)$ follows since
$\nicefrac{\nabla_{\nu}^2 \pi^i_{\nu}(s)}{\pi^i_{\nu}(s)} =  \nabla_{\nu}^2 \log\pi^i_{\nu}(s) + \nabla_{\nu}\log\pi^i_{\nu}(s)\nabla_{\nu}\log\pi^i_{\nu}(s)^\top$.

From \eqref{eq:agt_hess1}, we obtain
\begin{align}
\label{eq:agt_hess2}
\left\lVert\nabla_{\nu}^2G^i(\nu)\right\rVert &= \E_{s \sim \pi^i_{\nu^i},\,D\sim \mathcal{D}^i}\left[\left\lVert m^i_{\tilde{\theta}_{s}}(\bar{V}_s) \left(\nabla_{\nu}^2 \log\pi^i_{\nu}(s) + \nabla_{\nu}\log\pi^i_{\nu}(s)\nabla_{\nu}\log\pi^i_{\nu}(s)^\top\right)\right\rVert \right]\nonumber\\
&\stackrel{(a)}{\leq} \E_{s \sim \pi^i_{\nu},\,D\sim \mathcal{D}^i}\left[\left\lVert m^i_{\tilde{\theta}_{s}}(\bar{V}_s)\right\rVert \left(\left\lVert\nabla_{\nu}^2 \log\pi^i_{\nu}(s)\right\rVert + \left\lVert\nabla_{\nu}\log\pi^i_{\nu}(s)\right\rVert^2\right) \right]\nonumber\\
&\stackrel{(b)}{\leq} M_m^i\left(M_h^i+M_d^{i^2}\right).
\end{align}
Finally, the result follows from \eqref{eq:agt_hess2} and \citep[Lemma~1.2.2]{nesterov_book_1}.
\end{proof}
\begin{lemma}
\label{lm:agt_bias}
\begin{align*}
 \E_{s_t \sim \pi^i_{\nu_t},\,D_t\sim \mathcal{D}^i}\left[\widehat{\nabla}_\nu G^i(\nu_{t})\right]=\nabla_\nu G^i(\nu_{t}).
 \end{align*}
\end{lemma}
\begin{proof}
From \eqref{eq:agt_nablahat_G}, we obtain
\begin{align*}
 \E_{s_t \sim \pi^i_{\nu_t},\,D_t\sim \mathcal{D}^i}\left[\widehat{\nabla}_{\nu}G^i(\nu_t)\right]
&= \E_{s_t \sim \pi^i_{\nu_t},\,D_t\sim \mathcal{D}^i}\left[ m^i_{\tilde{\theta}_{s_t}}(\bar{V}_{s_t}) \nabla_{\nu} \log\pi^i_{\nu_t}(s_t)\right]\nonumber\\
&=\nabla_\nu G^i(\nu_{t}).
\end{align*}
\end{proof}
\begin{lemma}
\label{lm:agt_var}
\begin{align*}
 \E_{s_t \sim \pi^i_{\nu_t},\,D_t\sim \mathcal{D}^i}\left[\left\lVert\widehat{\nabla}_\nu G^i(\nu_{t})\right\rVert^2\right]\leq M_m^{i^2} M_d^{i^2}.
 \end{align*}
\end{lemma}
\begin{proof}
From \eqref{eq:agt_nablahat_G}, we obtain
\begin{align*}
 \E_{s_t \sim \pi^i_{\nu_t},\,D_t\sim \mathcal{D}^i}\left[\left\lVert\widehat{\nabla}_{\nu}G^i(\nu_t)\right\rVert^2\right]
&= \E_{s_t \sim \pi^i_{\nu_t},\,D_t\sim \mathcal{D}^i}\left[\left\lVert m^i_{\tilde{\theta}_{s_t}}(\bar{V}_{s_t}) \nabla_{\nu} \log\pi^i_{\nu_t}(s_t)\right\rVert^2\right]\nonumber\\
&\stackrel{(a)}{\leq} \E_{s_t \sim \pi^i_{\nu_t},\,D_t\sim \mathcal{D}^i}\left[\left\lVert m^i_{\tilde{\theta}_{s_t}}(\bar{V}_{s_t})\right\rVert^2 \left\lVert\nabla_{\nu} \log\pi^i_{\nu_t}(s_t)\right\rVert^2\right]\nonumber\\
&\stackrel{(b)}{\leq} M_m^{i^2} M_d^{i^2}.
\end{align*}
In the above, the step $(a)$ follows the fact that $\left\lVert ab\right\rVert\leq \left\lVert a \right\rVert\left\lVert b\right\rVert$. The step $(b)$ follows from Assumption \ref{as:agt_m_bound} and Assumption \ref{as:agt_policy_1}.
\end{proof}
\subsubsection*{Proof [\textbf{Theorem \ref{thm:agt}}]}
\label{subsubsec:pf_thm_agt}
Using the fundamental theorem of calculus, we obtain
    \begin{align}
         &G^i(\nu_{t+1}) -G^i(\nu_{t}) \nonumber\\
        &=\langle \nabla_\nu G^i(\nu_{t}),  \nu_{t+1}-\nu_{t}  \rangle
        + \int_0^1 \left\langle  \nabla_\nu G^i(\nu_{t}+\tau(\nu_{t+1}-\nu_{t}))-\nabla_\nu G^i(\nu_{t}), \nu_{t+1} - \nu_{t} \right\rangle d\tau\nonumber\\
        &\leq\langle \nabla_\nu G^i(\nu_{t}), \nu_{t+1} - \nu_{t} \rangle
        +\int_0^1 \left\lVert\nabla_\nu G^i(\nu_{t}+\tau(\nu_{t+1} - \nu_{t}))-\nabla_\nu G^i(\nu_{t})\right\rVert \left\lVert \nu_{t+1} - \nu_{t} \right\rVert d\tau\nonumber\\
        &\stackrel{(a)}{\leq} \left \langle \nabla_\nu G^i(\nu_{t}), \nu_{t+1} - \nu_{t} \right \rangle
        + M_m^i\left(M_h^i+M_d^{i^2}\right)\left\lVert \nu_{t+1} - \nu_{t} \right\rVert^2  \int_0^1 \tau d\tau \nonumber\\
        &= \left \langle \nabla_\nu G^i(\nu_{t}), \nu_{t+1} - \nu_{t} \right \rangle + \frac{M_m^i\left(M_h^i+M_d^{i^2}\right)}{2}\left\lVert \nu_{t+1} - \nu_{t} \right\rVert^2 \nonumber\\
        &= b_{t} \left \langle \nabla_\nu G^i(\nu_{t}),  -\widehat{\nabla}_\nu G^i(\nu_{t}) \right \rangle + \frac{M_m^i\left(M_h^i+M_d^{i^2}\right)b_{t}^2}{2}\left\lVert  \widehat{\nabla}_\nu G^i(\nu_{t}) \right\rVert^2 \nonumber\\
        &= b_{t} \left \langle \nabla_\nu G^i(\nu_{t}), \nabla_\nu G^i(\nu_{t}) -\widehat{\nabla}_\nu G^i(\nu_{t}) \right \rangle -b_{t} \left \langle \nabla_\nu G^i(\nu_{t}), \nabla_\nu G^i(\nu_{t})  \right \rangle+ \frac{M_m^i\left(M_h^i+M_d^{i^2}\right)b_{t}^2}{2}\left\lVert  \widehat{\nabla}_\nu G^i(\nu_{t}) \right\rVert^2 \nonumber\\
        &= b_{t} \left \langle \nabla_\nu G^i(\nu_{t}), \nabla_\nu G^i(\nu_{t}) -\widehat{\nabla}_\nu G^i(\nu_{t}) \right \rangle -b_{t} \left \lVert \nabla_\nu G^i(\nu_{t})  \right \rVert^2+ \frac{M_m^i\left(M_h^i+M_d^{i^2}\right)b_{t}^2}{2}\left\lVert  \widehat{\nabla}_\nu G^i(\nu_{t}) \right\rVert^2.
        \label{eq:agt_pf1}
    \end{align}
Taking expectations on both sides of \eqref{eq:agt_pf1}, we obtain
    \begin{align}
         &\E\left[G^i(\nu_{t+1})\right] - \E\left[G^i(\nu_{t})\right] \nonumber\\
         &\leq b_{t} \E\left[\left \langle \nabla_\nu G^i(\nu_{t}), \nabla_\nu G^i(\nu_{t}) -\widehat{\nabla}_\nu G^i(\nu_{t}) \right \rangle\right] -b_{t} \E\left[ \left \lVert \nabla_\nu G^i(\nu_{t})  \right \rVert^2\right]+ \frac{M_m^i\left(M_h^i+M_d^{i^2}\right)b_{t}^2}{2}\E\left[\left\lVert  \widehat{\nabla}_\nu G^i(\nu_{t}) \right\rVert^2 \right]\nonumber\\
         &\stackrel{(a)}{=} -b_{t} \E\left[ \left \lVert \nabla_\nu G^i(\nu_{t})  \right \rVert^2\right]+ \frac{M_m^i\left(M_h^i+M_d^{i^2}\right)b_{t}^2}{2}\E\left[\left\lVert  \widehat{\nabla}_\nu G^i(\nu_{t}) \right\rVert^2 \right]\nonumber\\
         &\stackrel{(b)}{\leq} -b_{t} \E\left[ \left \lVert \nabla_\nu G^i(\nu_{t})  \right \rVert^2\right]+ b_{t}^2\frac{\left(M_h^i+M_d^{i^2}\right)M_m^{i^3}M_d^{i^2}}{2}.
        \label{eq:agt_pf2}
    \end{align}
    In the above, the step $(a)$ follows from Lemma \ref{lm:agt_bias}, and the step $(b)$ follows from Lemma \ref{lm:agt_var}.

    Summing \eqref{eq:agt_pf2} from $t=0,\cdots,T^i-1$, we obtain
    \begin{align}
    \sum_{t=0}^{T^i-1}b_{t} \E\left[ \left \lVert \nabla_\nu G^i(\nu_{t})  \right \rVert^2\right] \leq \E\left[G^i(\nu_{0})\right]-\E\left[G^i(\nu_{T^i})\right]+\sum_{t=0}^{T^i-1} b_{t}^2\frac{\left(M_h^i+M_d^{i^2}\right)M_m^{i^3}M_d^{i^2}}{2}.
    \label{eq:agt_pf3}
    \end{align}
We have $\forall t,\, b_t=\nicefrac{1}{\sqrt{T^i}}$. Let $G^{i^*}=\argmin_{\nu\in \R^{n^i}} G(\nu)$. Then,
    \begin{align}
    \sum_{t=0}^{T^i-1}\frac{1}{\sqrt{T^i}} \E\left[ \left \lVert \nabla_\nu G^i(\nu_{t})  \right \rVert^2\right] \leq G^i(\nu_{0})-G^{i^*}+\frac{\left(M_h^i+M_d^{i^2}\right)M_m^{i^3}M_d^{i^2}}{2}.
    \label{eq:agt_pf4}
    \end{align}
    Since $\p(R=t)=\nicefrac{1}{T^i}$, we obtain
     \begin{align}
         \E\left[ \left \lVert \nabla_\nu G^i(\nu_{R})  \right \rVert^2\right]
         &=\frac{\sum_{t=0}^{T^i-1} \E\left[ \left \lVert \nabla_\nu G^i(\nu_{t})  \right \rVert^2\right] }{T^i}
=\frac{\sum_{t=0}^{T^i-1}\frac{1}{\sqrt{T^i}} \E\left[ \left \lVert \nabla_\nu G^i(\nu_{t})  \right \rVert^2\right] }{\sqrt{T^i}}\nonumber\\
&\leq \frac{G^i(\nu_{0})-G^{i^*}}{\sqrt{T^i}}+\frac{\left(M_h^i+M_d^{i^2}\right)M_m^{i^3}M_d^{i^2}}{2\sqrt{T^i}}.
    \label{eq:agt_pf5}
    \end{align}
\hfill{\qed}
\subsection{Arbiter}
\label{subsec:arbiter_pf}
\begin{proof}{[\textbf{Lemma \ref{lm:abt_grad}}]}


We derive expressions for $\nabla_{\theta} L(\theta,\omega,D)$ and $\nabla_{\omega}M(\theta,\omega,D)$ as given below:
\begin{align*}
\nabla_{\theta} L(\theta,\omega,D)
&=\nabla_\theta \big\langle\big(\E_{\scalebox{0.9}{$\substack{T^i\sim \mathcal{U}(D^i,\lceil p|D^i|\rceil)\\(x^i,y^i)\sim \mathcal{U}(T^i,1)}$}}[l_{\theta}(x^i,y^i)]\big)_{i=1}^{n},h_\omega   \big\rangle\nonumber\\
&=\big[\nabla_\theta \E_{\scalebox{0.9}{$\substack{T^i\sim \mathcal{U}(D^i,\lceil p|D^i|\rceil)\\(x^i,y^i)\sim \mathcal{U}(T^i,1)}$}}[l_{\theta}(x^i,y^i)]\big]_{i=1}^{n} h_\omega \nonumber\\
&\stackrel{(a)}{=}\big[ \E_{\scalebox{0.9}{$\substack{T^i\sim \mathcal{U}(D^i,\lceil p|D^i|\rceil)\\(x^i,y^i)\sim \mathcal{U}(T^i,1)}$}}[\nabla_\theta l_{\theta}(x^i,y^i)]\big]_{i=1}^{n} h_\omega.
\end{align*}
In the above, the step $(a)$ follows from Assumption \ref{as:grad_l_bound} by using the dominated convergence theorem.
\begin{align*}
\nabla_{\omega}M(\theta,\omega,D)
&=\nabla_{\omega}\big\langle\big(\E_{\scalebox{0.9}{$\substack{V^i\sim \mathcal{U}(D^i,\lfloor (1-p)|D^i|\rfloor)\\(x^i,y^i)\sim \mathcal{U}(V^i,1)}$}}[m_{\theta}(x^i,y^i)]\big)_{i=1}^{n} ,\omega \big\rangle +\nabla_{\omega}(\nicefrac{\lambda_\omega}{2})\left\lVert \omega \right\rVert^2\nonumber\\
&= \big(\E_{\scalebox{0.9}{$\substack{V^i\sim \mathcal{U}(D^i,\lfloor (1-p)|D^i|\rfloor)\\(x^i,y^i)\sim \mathcal{U}(V^i,1)}$}}[m_{\theta}(x^i,y^i)]\big)_{i=1}^{n} +\lambda_\omega \omega.
\end{align*}

\end{proof}
We present a few lemmas below, which are subsequently used in the proof of Theorem \ref{thm:abt}.
\begin{lemma}
\label{lm:grad_l_lip}
$ \forall \theta_1, \theta_2 \in \R^d,\; \left\lVert  \nabla_\theta l_{\theta_1}(x,y) - \nabla_\theta l_{\theta_2}(x,y)\right\rVert \leq L_{l'} \left\lVert  \theta_1 -\theta_2 \right\rVert,\; \forall (x,y) \in \mathcal{X}\times \mathcal{Y}$.
\end{lemma}
\begin{proof}
The result follows from Assumption \ref{as:grad_l_bound} and \citep[Lemma~1.2.2]{nesterov_book_1}.
\end{proof}
\begin{lemma}
\label{lm:grad_J_lip}
$ \forall \theta_1, \theta_2 \in \R^d,\; \left\lVert  \nabla_\theta J(\theta_1) - \nabla_\theta J(\theta_2)\right\rVert \leq L_{l'} \left\lVert  \theta_1 -\theta_2 \right\rVert$.
\end{lemma}
\begin{proof}
\begin{align*}
\left\lVert  \nabla_\theta J(\theta_1) - \nabla_\theta J(\theta_2)\right\rVert
&\stackrel{(a)}{=}\left\lVert \E_{(x,y)\sim \mathfrak{D}}\left [\nabla_\theta l_{\theta_1}(x,y) -\nabla_\theta l_{\theta_2}(x,y)\right]\right\rVert\\
&\stackrel{(b)}{\leq} \E_{(x,y)\sim \mathfrak{D}}\left [\left\lVert \nabla_\theta l_{\theta_1}(x,y) -\nabla_\theta l_{\theta_2}(x,y)\right\rVert\right]\\
&\stackrel{(c)}{\leq} L_{l'}\E_{(x,y)\sim \mathfrak{D}}\left [\left\lVert \theta_1 -\theta_2\right\rVert\right]\\
&= L_{l'}\left\lVert \theta_1 -\theta_2\right\rVert.
\end{align*}
The step $(a)$ follows from \eqref{eq:grad_J} and step $(b)$ follows since $\lVert \E[X]\rVert \leq \E[\lVert X\rVert]$. The step $(c)$ follows from Lemma \ref{lm:grad_l_lip}.
\end{proof}
\begin{lemma}
\label{lm:m_lip}
$ \forall \theta_1, \theta_2 \in \R^d,\; \left\lvert  m_{\theta_1}(x,y) -m_{\theta_2}(x,y)\right\rvert \leq L_m \left\lVert  \theta_1 -\theta_2 \right\rVert,\; \forall (x,y) \in \mathcal{X}\times \mathcal{Y}$.
\end{lemma}
\begin{proof}
Using fundamental theorem of calculus, we obtain
\begin{align*}
\left\lvert m_{\theta_1}(x,y)-m_{\theta_2}(x,y) \right\rvert&=  \left\lvert\int_0^1 \left\langle \nabla m_{\theta_2+\tau(\theta_1 - \theta_2)}(x,y), \theta_1 - \theta_2 \right\rangle d\tau \right\rvert\nonumber\\
&\leq  \int_0^1 \left\lvert\left\langle \nabla m_{\theta_2+\tau(\theta_1 - \theta_2)}(x,y), \theta_1 - \theta_2 \right\rangle\right\rvert d\tau \nonumber\\
& \stackrel{(a)}{\leq}  \int_0^1 \left\lVert \nabla m_{\theta_2+\tau(\theta_1 - \theta_2)}(x,y) \right\rVert \left\lVert \theta_1 - \theta_2 \right\rVert d\tau \nonumber\\
& \stackrel{(b)}{\leq}  \int_0^1 L_m \left\lVert \theta_1 - \theta_2 \right\rVert d\tau \nonumber\\
&=  L_m \left\lVert \theta_1 - \theta_2 \right\rVert.
\end{align*}
The step $(a)$ follows from the fact that $\lvert \langle a,b\rangle\rvert \leq \lVert a \rVert\lVert b \rVert$.
The step $(b)$ follows from Assumption \ref{as:abt_m_bound_2}.
\end{proof}

\begin{lemma}
\label{lm:cent_softmax_lip}
\begin{align*}
\forall \omega_1, \omega_2 \in \R^n,\; \left \lVert h_{\omega_1}-h_{\omega_2}\right \rVert \leq \left \lVert \omega_1-\omega_2\right \rVert.
\end{align*}
\end{lemma}
\begin{proof}
The result follows from \citep[Proposition~4]{gao18}.
\end{proof}
\begin{lemma}
\label{lm:cent_M_smooth_sc}
$ \forall \theta \in \R^d,\; M(\theta,\omega,D)$ is a smooth strongly convex function in $\omega$ with strong convexity and smoothness parameter $\lambda_\omega$ a.s., i.e.,
\begin{align*}
&\left\langle\nabla_\omega M(\theta,\omega_1,D) -\nabla_\omega M(\theta,\omega_2,D),\omega_1-\omega_2\right\rangle \geq \lambda_\omega \left\lVert \omega_1-\omega_2\right\rVert^2;\\
&\left\lVert\nabla_\omega M(\theta,\omega_1,D) -\nabla_\omega M(\theta,\omega_2,D)\right\rVert \leq \lambda_\omega \left\lVert \omega_1-\omega_2\right\rVert.
\end{align*}
\end{lemma}
\begin{proof}
From \eqref{eq:M}, we have
\begin{align}
\label{eq:nabla2_M}
\nabla^2_{\omega} M(\theta,\omega,D)
&=\nabla^2_{\omega}\big\langle\big(\E_{\scalebox{0.9}{$\substack{V^i\sim \mathcal{U}(D^i,\lfloor (1-p)|D^i|\rfloor)\\(x^i,y^i)\sim \mathcal{U}(V^i,1)}$}}[m_{\theta}(x^i,y^i)]\big)_{i=1}^{n} ,\omega \big\rangle +(\nicefrac{\lambda_\omega}{2}) \nabla^2_{\omega}\left\lVert \omega \right\rVert^2\nonumber\\
&= \lambda_\omega\mathbf{I}_n
\end{align}
From \eqref{eq:M}, we have $\lambda_\omega>0$. From \citep[Theorem~2.1.11 and Lemma~1.2.2]{nesterov_book_1}, we obtain $ \forall \theta \in \R^d,\; M(\theta,\omega,D)$ is a smooth strongly convex function in $\omega$ with strong convexity and smoothness parameter $\lambda_\omega$.
\end{proof}
\begin{lemma}
\label{lm:cent_M_lip}
\begin{align*}
\forall \theta_1,\theta_2 \in \R^d,\, \forall \omega_1,\omega_2 \in \R^n,\;\left\lVert\nabla_\omega M(\theta_1,\omega_1,D) -\nabla_\omega M(\theta_2,\omega_2,D)\right\rVert^2
\leq 2nL_m^2\left\lVert\theta_1 - \theta_2\right\rVert^2+ 2\lambda_\omega^2 \left\lVert\omega_1 - \omega_2\right\rVert^2 a.s.
\end{align*}
\end{lemma}
\begin{proof}
From Lemma \ref{lm:abt_grad}, we obtain
\begin{align}
\label{eq:nabla_M-diff}
&\left\lVert\nabla_{\omega} M(\theta_1,\omega_1,D)-\nabla_{\omega} M(\theta_2,\omega_2,D)\right\rVert^2\nonumber\\
&=\big\lVert \big(\E_{\scalebox{0.9}{$\substack{V^i\sim \mathcal{U}(D^i,\lfloor (1-p)|D^i|\rfloor)\\(x^i,y^i)\sim \mathcal{U}(V^i,1)}$}}[m_{\theta_1}(x^i,y^i)-m_{\theta_2}(x^i,y^i)]\big)_{i=1}^{n} +\lambda_\omega \left(\omega_1-\omega_2 \right)\big\rVert^2\nonumber\\
&\stackrel{(a)}{\leq} 2 \E_{\scalebox{0.9}{$\substack{V^1\sim \mathcal{U}(D^1,\lfloor (1-p)|D^1|\rfloor)\\(x^1,y^1)\sim \mathcal{U}(V^1,1)}, \ldots, \substack{V^n\sim \mathcal{U}(D^n,\lfloor (1-p)|D^n|\rfloor)\\(x^n,y^n)\sim \mathcal{U}(V^n,1)}$}} \left[\left\lVert \big(m_{\theta_1}(x^i,y^i)-m_{\theta_2}(x^i,y^i)\big)_{i=1}^{n}\right\rVert^2\right]+ 2\lambda_\omega^2 \left\lVert\omega_1 - \omega_2\right\rVert^2\nonumber\\
&= 2 \E_{\scalebox{0.9}{$\substack{V^1\sim \mathcal{U}(D^1,\lfloor (1-p)|D^1|\rfloor)\\(x^1,y^1)\sim \mathcal{U}(V^1,1)}, \ldots, \substack{V^n\sim \mathcal{U}(D^n,\lfloor (1-p)|D^n|\rfloor)\\(x^n,y^n)\sim \mathcal{U}(V^n,1)}$}} \left[\sum_{i=1}^{n}\left\lvert m_{\theta_1}(x^i,y^i)-m_{\theta_2}(x^i,y^i)\right\rvert^2\right]+ 2\lambda_\omega^2 \left\lVert\omega_1 - \omega_2\right\rVert^2\nonumber\\
&\stackrel{(b)}{\leq} 2nL_m^2\left\lVert\theta_1 - \theta_2\right\rVert^2+ 2\lambda_\omega^2 \left\lVert\omega_1 - \omega_2\right\rVert^2.
\end{align}
In the above, the step $(a)$ follows from the fact that $\lVert x+y\rVert^2\leq 2\lVert x\rVert^2+2\lVert y\rVert^2$, and $\lVert \E\left[x\right]\rVert^2 \leq\E\left[ \lVert x\rVert^2 \right]$. The step $(b)$ follows from Lemma \ref{lm:m_lip}.
\end{proof}
\begin{lemma}
\label{lm:cent_M_lip-1}
\begin{align*}
\forall \theta \in \R^d,\, \forall \omega_1,\omega_2 \in \R^n,\;\left\lVert\nabla_\omega M(\theta,\omega_1,D) -\nabla_\omega M(\theta,\omega_2,D)\right\rVert^2 = \lambda_\omega^2 \left\lVert\omega_1 - \omega_2\right\rVert^2 a.s.
\end{align*}
\end{lemma}
\begin{proof}
From Lemma \ref{lm:abt_grad}, we obtain
\begin{align}
\label{eq:nabla_M-diff1}
\left\lVert\nabla_{\omega} M(\theta,\omega_1,D)-\nabla_{\omega} M(\theta,\omega_2,D)\right\rVert^2
=\left\lVert \lambda_\omega \left(\omega_1 - \omega_2\right)\right\rVert^2
= \lambda_\omega^2 \left\lVert\omega_1 - \omega_2\right\rVert^2.
\end{align}
\end{proof}
\begin{lemma}
\label{lm:cent_M_lip-2}
\begin{align*}
\forall \omega \in \R^n,\,  \forall \theta_1, \theta_2 \in \R^d,\;\left\lVert\nabla_\omega M(\theta_1,\omega,D) -\nabla_\omega M(\theta_2,\omega,D)\right\rVert^2 = nL_m^2\left\lVert\theta_1 - \theta_2\right\rVert^2 a.s.
\end{align*}
\end{lemma}
\begin{proof}
From Lemma \ref{lm:abt_grad}, we obtain

\begin{align}
\label{eq:nabla_M-diff2}
&\left\lVert\nabla_{\omega} M(\theta_1,\omega,D)-\nabla_{\omega} M(\theta_2,\omega,D)\right\rVert^2\nonumber\\
&=\big\lVert \big(\E_{\scalebox{0.9}{$\substack{V^i\sim \mathcal{U}(D^i,\lfloor (1-p)|D^i|\rfloor)\\(x^i,y^i)\sim \mathcal{U}(V^i,1)}$}}[m_{\theta_1}(x^i,y^i)-m_{\theta_2}(x^i,y^i)]\big)_{i=1}^{n} \big\rVert^2\nonumber\\
&\stackrel{(a)}{\leq}\E_{\scalebox{0.9}{$\substack{V^1\sim \mathcal{U}(D^1,\lfloor (1-p)|D^1|\rfloor)\\(x^1,y^1)\sim \mathcal{U}(V^1,1)}, \ldots, \substack{V^n\sim \mathcal{U}(D^n,\lfloor (1-p)|D^n|\rfloor)\\(x^n,y^n)\sim \mathcal{U}(V^n,1)}$}} \left[\left\lVert \big(m_{\theta_1}(x^i,y^i)-m_{\theta_2}(x^i,y^i)\big)_{i=1}^{n}\right\rVert^2\right]\nonumber\\
&=\E_{\scalebox{0.9}{$\substack{V^1\sim \mathcal{U}(D^1,\lfloor (1-p)|D^1|\rfloor)\\(x^1,y^1)\sim \mathcal{U}(V^1,1)}, \ldots, \substack{V^n\sim \mathcal{U}(D^n,\lfloor (1-p)|D^n|\rfloor)\\(x^n,y^n)\sim \mathcal{U}(V^n,1)}$}} \left[\sum_{i=1}^{n}\left\lvert m_{\theta_1}(x^i,y^i)-m_{\theta_2}(x^i,y^i)\right\rvert^2\right]\nonumber\\
&\stackrel{(b)}{\leq} nL_m^2\left\lVert\theta_1 - \theta_2\right\rVert^2.
\end{align}
In the above, the step $(a)$ follows from the fact that  $\lVert \E\left[x\right]\rVert^2 \leq\E\left[ \lVert x\rVert^2 \right]$. The step $(b)$ follows from Lemma \ref{lm:m_lip}.
\end{proof}
\begin{lemma}
\label{lm:cent_M_bnd}
\begin{align*}
\E\left[\left\lVert \nabla_\omega M(\theta_{t},\omega^*(\theta_t),D_t) \right\rVert^2\right]\leq 2nM_m^2 + 2\lambda_\omega^2M_{\omega^*}^2.
\end{align*}
\end{lemma}
\begin{proof}
\begin{align}
&\E\left[\left\lVert \widehat{\nabla}_\omega M(\theta_{t},\omega^*(\theta_t),D_t) \right\rVert^2\right]\nonumber\\
&= \E\left[\left\lVert \big(
\nicefrac{1}{|V_t^i|}\sum\nolimits_{(x^i,y^i)\in V_t^i}m_{\theta_{t}}(x^i,y^i)\big)_{i=1}^{n}+\lambda_\omega \omega^*(\theta_t)\right\rVert^2\right]\nonumber\\
&\stackrel{(a)}{\leq} 2\E\left[\left\lVert \big(
\nicefrac{1}{|V_t^i|}\sum\nolimits_{(x^i,y^i)\in V_t^i}m_{\theta_{t}}(x^i,y^i)\big)_{i=1}^{n}\right\rVert^2\right]+ 2\E\left[\left\lVert \lambda_\omega \omega^*(\theta_t)\right\rVert^2\right]\nonumber\\
&= 2\E\left[\sum_{i=1}^n \left\lvert
\nicefrac{1}{|V_t^i|}\sum\nolimits_{(x^i,y^i)\in V_t^i}m_{\theta_{t}}(x^i,y^i)\right\rvert^2\right]
+ 2\E\left[\left\lVert \lambda_\omega \omega^*(\theta_t)\right\rVert^2\right]\nonumber\\
&\stackrel{(b)}{\leq} 2\E\left[\sum_{i=1}^n \nicefrac{1}{|V_t^i|}\sum\nolimits_{(x^i,y^i)\in V_t^i}\left\lvert
m_{\theta_{t}}(x^i,y^i)\right\rvert^2\right]
+ 2\lambda_\omega^2\E\left[\left\lVert  \omega^*(\theta_t)\right\rVert^2\right]\nonumber\\
&\stackrel{(c)}{\leq} 2nM_m^2
+ 2\lambda_\omega^2M_{\omega^*}^2.
\end{align}
In the above, the step $(a)$ follows since $\E\left[\left\lVert a+b \right\rVert^2\right]\leq 2\E\left[\left\lVert a \right\rVert^2\right]+\E\left[\left\lVert b \right\rVert^2\right]$. The step $(b)$ follows since $\left\lVert \sum_{i=1}^n a_i \right\rVert^2\leq n \sum_{i=1}^n \left\lVert  a_i \right\rVert^2$. The step $(c)$ follows from Assumption \ref{as:abt_m_bound_1} and Assumption \ref{as:cent_w*}.
\end{proof}
\begin{lemma}
\label{lm:cent_M_mse}
\begin{align*}
\E\left[\left\lVert \nabla_\omega M(\theta_{t+1},\omega_t,D)-\widehat{\nabla}_\omega M(\theta_{t+1},\omega_t,D_t) \right\rVert^2\right] \leq 4M_m^2\sum_{i=1}^n\E\left[\frac{1}{\lfloor(1-p)|D_t^i|\rfloor}\right]
\end{align*}
\end{lemma}
\begin{proof}
\begin{align}
\label{eq:M_mse1}
&\E\left[\left\lVert \nabla_\omega M(\theta_{t+1},\omega_t,D)-\widehat{\nabla}_\omega M(\theta_{t+1},\omega_t,D_t) \right\rVert^2\right]\nonumber\\
&=\E\Big[\big\lVert \big(\E_{\scalebox{0.9}{$\substack{V^i\sim \mathcal{U}(D^i,\lfloor (1-p)|D^i|\rfloor)\\(x^i,y^i)\sim \mathcal{U}(V^i,1)}$}}[m_{\theta_{t+1}}(x^i,y^i)]\big)_{i=1}^{n}
- \big(
\nicefrac{1}{|V_t^i|}\sum\nolimits_{(x^i,y^i)\in V_t^i}m_{\theta_{t}}(x^i,y^i)\big)_{i=1}^{n}
\big\rVert^2\Big]\nonumber\\
&=\E\Big[\big\lVert \big(\E_{\scalebox{0.9}{$\substack{V^i\sim \mathcal{U}(D^i,\lfloor (1-p)|D^i|\rfloor)\\(x^i,y^i)\sim \mathcal{U}(V^i,1)}$}}[m_{\theta_{t+1}}(x^i,y^i)]-
\nicefrac{1}{|V_t^i|}\sum\nolimits_{(x^i,y^i)\in V_t^i}m_{\theta_{t}}(x^i,y^i)\big)_{i=1}^{n}
\big\rVert^2\Big]\nonumber\\
&=\E\Big[\sum_{i=1}^{n}\big\lvert \E_{\scalebox{0.9}{$\substack{V^i\sim \mathcal{U}(D^i,\lfloor (1-p)|D^i|\rfloor)\\(x^i,y^i)\sim \mathcal{U}(V^i,1)}$}}[m_{\theta_{t+1}}(x^i,y^i)]-
\nicefrac{1}{|V_t^i|}\sum\nolimits_{(x^i,y^i)\in V_t^i}m_{\theta_{t}}(x^i,y^i)
\big\rvert^2\Big]\nonumber\\
&=\E\Big[\sum_{i=1}^{n}\big\lvert \E_{V^i\sim \mathcal{U}(D^i,\lfloor (1-p)|D^i|\rfloor)}\big[ \E_{(x^i,y^i)\sim \mathcal{U}(V^i,1)}[m_{\theta_{t+1}}(x^i,y^i)]-
\nicefrac{1}{|V_t^i|}\sum\nolimits_{(x^i,y^i)\in V_t^i}m_{\theta_{t}}(x^i,y^i)\big]
\big\rvert^2\Big]\nonumber\\
&\stackrel{(a)}{\leq}\sum_{i=1}^{n}\E\Big[ \E_{V^i\sim \mathcal{U}(D^i,\lfloor (1-p)|D^i|\rfloor)}\big[ \big\lvert\E_{(x^i,y^i)\sim \mathcal{U}(V^i,1)}[m_{\theta_{t+1}}(x^i,y^i)]-
\nicefrac{1}{|V_t^i|}\sum\nolimits_{(x^i,y^i)\in V_t^i}m_{\theta_{t}}(x^i,y^i)
\big\rvert^2\big]\Big].
\end{align}
In the above, the step $(a)$ follows from the fact that $|\E[x]|^2\leq \E[|x|]^2$. From Assumption \ref{as:abt_m_bound_1}, we have $\left\lvert m_\theta(x,y) \right\rvert \leq M_m$. Let $V_{t_j}^i$ represents first $j$ elements of $V_t^i$, then\\ $\big\{j\big( \E_{(x^i,y^i)\sim \mathcal{U}(V^i,1)}\left[m_{\theta_{t+1}}(x^i,y^i)\right]- \frac{1}{|V_{t_j}^i|}\sum_{(x^i,y^i)\in V_{t_j}^i}m_{\theta_{t+1}}(x^i,y^i)\big)\big\}_{j=1}^{|V_t^i|}$ is a set of partial sums of bounded mean zero r.v.s and hence they are martingales. Using Azuma-Hoeffding inequality, we obtain
\begin{align}
\label{eq:M_mse2}
\p\Big(\big\lvert \E_{(x^i,y^i)\sim \mathcal{U}(V^i,1)}\left[m_{\theta_{t+1}}(x^i,y^i)\right]-
\frac{1}{|V_t^i|}\sum\nolimits_{(x^i,y^i)\in V_t^i}m_{\theta_{t+1}}(x^i,y^i)\big\rvert > \epsilon \Big) \leq 2 \exp{\frac{-|V_t^i|\epsilon^2}{2M_m^2}} ,
\end{align}
and
\begin{align}
\label{eq:M_mse3}
&\E_{V^i\sim \mathcal{U}(D^i,\lfloor (1-p)|D^i|\rfloor)}\Big[\big\lvert \E_{(x^i,y^i)\sim \mathcal{U}(V^i,1)}\left[m_{\theta_{t+1}}(x^i,y^i)\right]-
\frac{1}{|V_t^i|}\sum\nolimits_{(x^i,y^i)\in V_t^i}m_{\theta_{t+1}}(x^i,y^i)\big\rvert^2\Big]\nonumber\\
&=\int_0^{\infty} \p\Big(\big\lvert \E_{(x^i,y^i)\sim \mathcal{U}(V^i,1)}\left[m_{\theta_{t+1}}(x^i,y^i)\right]-
\frac{1}{|V_t^i|}\sum\nolimits_{(x^i,y^i)\in V_t^i}m_{\theta_{t+1}}(x^i,y^i)\big\rvert > \sqrt{\epsilon} \Big)\,d\epsilon\nonumber\\
&\stackrel{(a)}{\leq} \int_0^{\infty} 2 \exp{\frac{-|V_t^i|\epsilon}{2M_m^2}} \,d\epsilon \nonumber\\
&\stackrel{(b)}{=} \frac{4M_m^2}{|V_t^i|}.
\end{align}
In the above, the step $(a)$ follows from \eqref{eq:M_mse2}, and the step $(b)$ follows from the fact that $\int_0^\infty \exp(-a\epsilon)\,d\epsilon=\frac{1}{a},\;\forall a>0$.

Applying \eqref{eq:M_mse3} in \eqref{eq:M_mse1}, we obtain
\begin{align}
\label{eq:M_mse4}
\E\left[\left\lVert \nabla_\omega M(\theta_{t+1},\omega_t,D)-\widehat{\nabla}_\omega M(\theta_{t+1},\omega_t,D_t) \right\rVert^2\right]
&\leq \sum_{i=1}^n\E\left[\frac{4M_m^2}{|V_t^i|}\right]\nonumber\\
&= 4M_m^2\sum_{i=1}^n\E\left[\frac{1}{\lfloor(1-p)|D_t^i|\rfloor}\right].
\end{align}
\end{proof}
\begin{lemma}
\label{lm:cent_w_diff_bound}
$\forall \theta_1,\theta_2 \in \R^d,\, \forall \omega_1,\omega_2 \in \R^n$,
\begin{align*}
 \left\lVert\omega_1-\omega_2\right\rVert^2
\leq \frac{2}{\lambda_\omega}\left\langle\nabla_{\omega} M(\theta_1,\omega_1,D)-\nabla_{\omega} M(\theta_2,\omega_2,D), \omega_1-\omega_2 \right\rangle + \frac{nL_m^2}{\lambda_\omega^2}\left\lVert \theta_1-\theta_2\right\rVert^2 a.s.
\end{align*}
\end{lemma}
\begin{proof}
From Lemma \ref{lm:cent_M_smooth_sc} and \citep[Definition~2.1.2]{nesterov_book_1}, we obtain
\begin{align}
M(\theta_1,\omega_2,D) \geq M(\theta_1,\omega_1,D)+ \left\langle\nabla_{\omega} M(\theta_1,\omega_1,D), \omega_2-\omega_1 \right\rangle+\frac{\lambda_\omega}{2}\left\lVert \omega_2-\omega_1\right\rVert^2; \label{eq:M_sc1}\\
M(\theta_2,\omega_1,D) \geq M(\theta_2,\omega_2,D)+ \left\langle\nabla_{\omega} M(\theta_2,\omega_2,D), \omega_1-\omega_2 \right\rangle+\frac{\lambda_\omega}{2}\left\lVert \omega_1-\omega_2\right\rVert^2. \label{eq:M_sc2}
\end{align}
Adding \eqref{eq:M_sc1} and \eqref{eq:M_sc2}, we obtain
\begin{align}
\label{eq:M_sc3}
& \left\langle\nabla_{\omega} M(\theta_1,\omega_1,D)-\nabla_{\omega} M(\theta_2,\omega_2,D), \omega_2-\omega_1 \right\rangle +\lambda_\omega\left\lVert \omega_1-\omega_2\right\rVert^2 \nonumber\\
 &\leq M(\theta_1,\omega_2,D)-M(\theta_2,\omega_2,D)-\left(M(\theta_1,\omega_1,D)-M(\theta_2,\omega_1,D)\right)\nonumber\\
 &=\big\langle\omega_2-\omega_1, \big(\E_{\scalebox{0.9}{$\substack{V^i\sim \mathcal{U}(D^i,\lfloor (1-p)|D^i|\rfloor)\\(x^i,y^i)\sim \mathcal{U}(V^i,1)}$}}[m_{\theta_1}(x^i,y^i)-m_{\theta_2}(x^i,y^i)]\big)_{i=1}^{n} \big\rangle \nonumber\\
 &\stackrel{(a)}{\leq}\frac{\lambda_\omega}{2}\left\lVert\omega_1-\omega_2\right\rVert^2
 +\frac{1}{2\lambda_\omega}\E_{\scalebox{0.9}{$\substack{V^1\sim \mathcal{U}(D^1,\lfloor (1-p)|D^1|\rfloor)\\(x^1,y^1)\sim \mathcal{U}(V^1,1)}, \ldots, \substack{V^n\sim \mathcal{U}(D^n,\lfloor (1-p)|D^n|\rfloor)\\(x^n,y^n)\sim \mathcal{U}(V^n,1)}$}}\left[\left\lVert (m_{\theta_1}(x^i,y^i)-m_{\theta_2}(x^i,y^i))_{i=1}^n\right\rVert^2\right]\nonumber\\
&=\frac{\lambda_\omega}{2}\left\lVert\omega_1-\omega_2\right\rVert^2
 +\frac{1}{2\lambda_\omega}\E_{\scalebox{0.9}{$\substack{V^1\sim \mathcal{U}(D^1,\lfloor (1-p)|D^1|\rfloor)\\(x^1,y^1)\sim \mathcal{U}(V^1,1)}, \ldots, \substack{V^n\sim \mathcal{U}(D^n,\lfloor (1-p)|D^n|\rfloor)\\(x^n,y^n)\sim \mathcal{U}(V^n,1)}$}}\left[\sum_{i=1}^{n}\left\lvert m_{\theta_1}(x^i,y^i)-m_{\theta_2}(x^i,y^i)\right\rvert^2\right]\nonumber\\
 &\stackrel{b}{\leq}\frac{\lambda_\omega}{2}\left\lVert\omega_1-\omega_2\right\rVert^2
 +\frac{L_m^2}{2\lambda_\omega}\E_{\scalebox{0.9}{$\substack{V^1\sim \mathcal{U}(D^1,\lfloor (1-p)|D^1|\rfloor)\\(x^1,y^1)\sim \mathcal{U}(V^1,1)}, \ldots, \substack{V^n\sim \mathcal{U}(D^n,\lfloor (1-p)|D^n|\rfloor)\\(x^n,y^n)\sim \mathcal{U}(V^n,1)}$}}\left[\sum_{i=1}^{n}\left\lVert \theta_1-\theta_2\right\rVert^2\right]\nonumber\\
 &=\frac{\lambda_\omega}{2}\left\lVert\omega_1-\omega_2\right\rVert^2
 +\frac{nL_m^2}{2\lambda_\omega}\left\lVert \theta_1-\theta_2\right\rVert^2.
\end{align}
In the above, the step $(a)$ follows since $<a,b>\leq \frac{\lVert a \rVert^2}{2\delta} + \frac{\delta\lVert b \rVert^2}{2},\;\delta>0$, and from the fact that $\left\lVert\E[x]\right\rVert^2\leq\E[\left\lVert x \right\rVert^2]$. The step $(b)$ follows from Lemma \ref{lm:m_lip}.

Re-arranging \eqref{eq:M_sc3}, we obtain
\begin{align}
\left\lVert\omega_1-\omega_2\right\rVert^2
\leq \frac{2}{\lambda_\omega}\left\langle\nabla_{\omega} M(\theta_1,\omega_1,D)-\nabla_{\omega} M(\theta_2,\omega_2,D), \omega_1-\omega_2 \right\rangle + \frac{nL_m^2}{\lambda_\omega^2}\left\lVert \theta_1-\theta_2\right\rVert^2.
\end{align}
\end{proof}
\begin{lemma}
\label{lm:cent_w_diff_bound-1}
\begin{align*}
\forall \theta \in \R^d,\, \forall \omega_1,\omega_2 \in \R^n,\; \left\lVert\omega_1-\omega_2\right\rVert^2
\leq \frac{1}{\lambda_\omega}\left\langle\nabla_{\omega} M(\theta,\omega_1,D)-\nabla_{\omega} M(\theta,\omega_2,D), \omega_1-\omega_2 \right\rangle \,a.s.
\end{align*}
\end{lemma}
\begin{proof}
From Lemma \ref{lm:cent_M_smooth_sc} and \citep[Definition~2.1.2]{nesterov_book_1}, we obtain
\begin{align}
M(\theta,\omega_2,D) \geq M(\theta,\omega_1,D)+ \left\langle\nabla_{\omega} M(\theta,\omega_1,D), \omega_2-\omega_1 \right\rangle+\frac{\lambda_\omega}{2}\left\lVert \omega_2-\omega_1\right\rVert^2; \label{eq:M_sc1-1}\\
M(\theta,\omega_1,D) \geq M(\theta,\omega_2,D)+ \left\langle\nabla_{\omega} M(\theta,\omega_2,D), \omega_1-\omega_2 \right\rangle+\frac{\lambda_\omega}{2}\left\lVert \omega_1-\omega_2\right\rVert^2. \label{eq:M_sc2-1}
\end{align}
Adding \eqref{eq:M_sc1-1} and \eqref{eq:M_sc2-1}, we obtain
\begin{align}
\label{eq:M_sc3-1}
& \left\langle\nabla_{\omega} M(\theta,\omega_1,D)-\nabla_{\omega} M(\theta,\omega_2,D), \omega_2-\omega_1 \right\rangle +\lambda_\omega\left\lVert \omega_1-\omega_2\right\rVert^2 \leq 0.
\end{align}

Re-arranging \eqref{eq:M_sc3-1}, we obtain
\begin{align}
\left\lVert\omega_1-\omega_2\right\rVert^2
\leq \frac{1}{\lambda_\omega}\left\langle\nabla_{\omega} M(\theta,\omega_1,D)-\nabla_{\omega} M(\theta,\omega_2,D), \omega_1-\omega_2 \right\rangle.
\end{align}
\end{proof}
\begin{lemma}
\label{lm:cent_w_rec}
For any $\delta>0$,
\begin{align*}
\E\left[\left\lVert \omega^*(\theta_{t,k+1})- \omega_{t} \right\rVert^2 \right]
\leq \left(1+\delta\right)\E\left[\left\lVert\omega^*(\theta_{t,k})- \omega_{t} \right\rVert^2\right]+ \left(1+\frac{1}{\delta}\right)\frac{nL_m^2\alpha_{t}^2}{\lambda_\omega^2}\E\left[\left\lVert \widehat{\nabla}_\theta L(\theta_{t,k},\omega_t,D_t)\right\rVert^2\right].
\end{align*}
\end{lemma}
\begin{proof}
\begin{align}
\label{eq:w_rec-1}
&\left\lVert \omega^*(\theta_{t,k+1})- \omega_{t} \right\rVert^2 \nonumber\\
&=\left\lVert \omega^*(\theta_{t,k+1})-\omega^*(\theta_{t,k})+\omega^*(\theta_{t,k})- \omega_{t} \right\rVert^2 \nonumber\\
&=\left\lVert\omega^*(\theta_{t,k})- \omega_{t} \right\rVert^2+\left\lVert \omega^*(\theta_{t,k+1})-\omega^*(\theta_{t,k})\right\rVert^2+2\left\langle \omega^*(\theta_{t,k})- \omega_{t}, \omega^*(\theta_{t,k+1})-\omega^*(\theta_{t,k})\right\rangle \nonumber\\
&\stackrel{(a)}{\leq}\left(1+\delta\right)\left\lVert\omega^*(\theta_{t,k})- \omega_{t} \right\rVert^2+ \left(1+\frac{1}{\delta}\right)\left\lVert \omega^*(\theta_{t,k+1})-\omega^*(\theta_{t,k})\right\rVert^2 \nonumber\\
&\stackrel{(b)}{\leq}\left(1+\delta\right)\left\lVert\omega^*(\theta_{t,k})- \omega_{t} \right\rVert^2\nonumber\\
&\quad+ \left(1+\frac{1}{\delta}\right)\frac{2}{\lambda_\omega}\left\langle \nabla_{\omega} M(\theta_{t,k+1},\omega^*(\theta_{t,k+1}),D)-\nabla_{\omega} M(\theta_{t,k},\omega^*(\theta_{t,k}),D),\omega^*(\theta_{t,k+1})-\omega^*(\theta_{t,k}) \right\rangle \nonumber\\
&\quad+ \left(1+\frac{1}{\delta}\right)\frac{nL_m^2}{\lambda_\omega^2}\left\lVert \theta_{t,k+1}-\theta_{t,k}\right\rVert^2.
\end{align}
In the above, the step $(a)$ follows since $<a,b>\leq \frac{\lVert a \rVert^2}{2\delta} + \frac{\delta\lVert b \rVert^2}{2},\;\delta>0$., and the step $(b)$ follows from Lemma \ref{lm:cent_w_diff_bound}.

Taking expectation on both sides of \eqref{eq:w_rec-1}, we obtain
\begin{align}
\label{eq:w_rec-2}
&\E\left[\left\lVert \omega^*(\theta_{t,k+1})- \omega_{t} \right\rVert^2 \right]\nonumber\\
&\leq\left(1+\delta\right)\E\left[\left\lVert\omega^*(\theta_{t,k})- \omega_{t} \right\rVert^2\right]\nonumber\\
&\quad+ \left(1+\frac{1}{\delta}\right)\frac{2}{\lambda_\omega}\E\left[\left\langle\nabla_{\omega} M(\theta_{t,k+1},\omega^*(\theta_{t,k+1}),D)-\nabla_{\omega} M(\theta_{t,k},\omega^*(\theta_{t,k}),D), \omega^*(\theta_{t,k+1})-\omega^*(\theta_{t,k}) \right\rangle \right]\nonumber\\
&\quad+ \left(1+\frac{1}{\delta}\right)\frac{nL_m^2}{\lambda_\omega^2}\E\left[\left\lVert \theta_{t,k+1}-\theta_{t,k}\right\rVert^2\right]\nonumber\\
&\stackrel{(a)}{=}\left(1+\delta\right)\E\left[\left\lVert\omega^*(\theta_{t,k})- \omega_{t} \right\rVert^2\right]+ \left(1+\frac{1}{\delta}\right)\frac{nL_m^2}{\lambda_\omega^2}\E\left[\left\lVert \theta_{t,k+1}-\theta_{t,k}\right\rVert^2\right]\nonumber\\
&\stackrel{(b)}{=}\left(1+\delta\right)\E\left[\left\lVert\omega^*(\theta_{t,k})- \omega_{t} \right\rVert^2\right]+ \left(1+\frac{1}{\delta}\right)\frac{nL_m^2\alpha_{t}^2}{\lambda_\omega^2}\E\left[\left\lVert \widehat{\nabla}_\theta L(\theta_{t,k},\omega_t,D_t)\right\rVert^2\right].
\end{align}
In the above, the step $(a)$ follows since $\E_{D\sim \mathcal{D^{\theta}}}\left[\nabla_{\omega}M(\theta,\omega^*(\theta),D)\right]=0$ from Assumption \ref{as:cent_w*}, and the step $(b)$ follows from \eqref{eq:abt_theta_it}.
\end{proof}
\begin{lemma}
\begin{align*}
\E\left[\left\lVert \omega_{t+1}-\omega^*(\theta_{t+1})  \right\rVert^2\right]
&\leq \left(1- \beta_t\lambda_\omega +\beta_t^2\lambda_\omega^2\right)\E\left[\left\lVert \omega_{t}-\omega^*(\theta_{t})\right\rVert^2 \right]
+\left(3\beta_t^2+7\beta_t^3\lambda_\omega\right) \E\left[\left\lVert\nabla_\omega M(\theta_{t},\omega^*(\theta_{t}),D)\right\rVert^2\right]\nonumber\\
 &\quad+\left(3\beta_t^2+\frac{8\beta_t}{\lambda_\omega}+7\beta_t^3\lambda_\omega \right)\E\left[\left\lVert \widehat{\nabla}_\omega M(\theta_{t+1},\omega_t,D_t)-\nabla_\omega M(\theta_{t+1},\omega_t,D) \right\rVert^2\right] \nonumber\\
&\quad+\left( 1+ 7\beta_t^3\lambda_\omega^3+3\beta_t^2\lambda_\omega^2+\frac{7}{\beta_t\lambda_\omega}+17\beta_t\lambda_\omega\right)\frac{nL_m^2}{\lambda_\omega^2}\E\left[\left\lVert \theta_{t+1}-\theta_t\right\rVert^2\right].
\end{align*}
\end{lemma}
\begin{proof}
\begin{align}
\label{eq:wt-w*_1}
&\left\lVert \omega_{t+1}-\omega^*(\theta_{t+1})  \right\rVert^2\nonumber\\
&=\left\lVert \omega_{t}-\beta_t \widehat{\nabla}_\omega M(\theta_{t+1},\omega_t,D_t) -\omega^*(\theta_{t+1}) \right\rVert^2 \nonumber\\
&=\left\lVert \left(\omega_{t}-\omega^*(\theta_{t})\right)-\left(\omega^*(\theta_{t+1}) -\omega^*(\theta_{t}) \right)-\beta_t\left( \nabla_\omega M(\theta_{t+1},\omega_t,D)-\nabla_\omega M(\theta_{t+1},\omega^*(\theta_{t}),D)\right)\right. \nonumber\\
&\quad-\beta_t \left(\nabla_\omega M(\theta_{t},\omega^*(\theta_{t}),D)+\left(\nabla_\omega M(\theta_{t+1},\omega^*(\theta_{t}),D)-\nabla_\omega M(\theta_{t},\omega^*(\theta_{t}),D)\right) \right.\nonumber\\
&\quad\left.\left.+\left( \widehat{\nabla}_\omega M(\theta_{t+1},\omega_t,D_t)-\nabla_\omega M(\theta_{t+1},\omega_t,D)\right)\right) \right\rVert^2 \nonumber\\
&\stackrel{(a)}{\leq}\left\lVert \omega_{t}-\omega^*(\theta_{t})\right\rVert^2+\left\lVert\omega^*(\theta_{t+1}) -\omega^*(\theta_{t}) \right\rVert^2 \nonumber\\
&\quad+\beta_t^2 \left\lVert \nabla_\omega M(\theta_{t+1},\omega_t,D)-\nabla_\omega M(\theta_{t+1},\omega^*(\theta_{t}),D)\right\rVert^2\nonumber\\
&\quad+3\beta_t^2 \left\lVert\nabla_\omega M(\theta_{t},\omega^*(\theta_{t}),D)\right\rVert^2
 +3\beta_t^2 \left\lVert\nabla_\omega M(\theta_{t+1},\omega^*(\theta_{t}),D)-\nabla_\omega M(\theta_{t},\omega^*(\theta_{t}),D)\right\rVert^2 \nonumber\\
 &\quad+3\beta_t^2 \left\lVert \widehat{\nabla}_\omega M(\theta_{t+1},\omega_t,D_t)-\nabla_\omega M(\theta_{t+1},\omega_t,D) \right\rVert^2 \nonumber\\
 &\quad-2\left\langle \omega_{t}-\omega^*(\theta_{t}), \omega^*(\theta_{t+1}) -\omega^*(\theta_{t})\right\rangle
-2\beta_t\left\langle \omega_{t}-\omega^*(\theta_{t}), \nabla_\omega M(\theta_{t},\omega^*(\theta_{t}),D)\right\rangle \nonumber\\
&\quad-2\beta_t\left\langle \omega_{t}-\omega^*(\theta_{t}), \nabla_\omega M(\theta_{t+1},\omega_t,D)-\nabla_\omega M(\theta_{t+1},\omega^*(\theta_{t}),D) \right\rangle\nonumber\\
&\quad-2\beta_t\left\langle \omega_{t}-\omega^*(\theta_{t}), \nabla_\omega M(\theta_{t+1},\omega^*(\theta_{t}),D)-\nabla_\omega M(\theta_{t},\omega^*(\theta_{t}),D) \right\rangle\nonumber\\
&\quad-2\beta_t\left\langle \omega_{t}-\omega^*(\theta_{t}), \widehat{\nabla}_\omega M(\theta_{t+1},\omega_t,D_t)-\nabla_\omega M(\theta_{t+1},\omega_t,D)\right\rangle \nonumber\\
&\quad-2\beta_t\left\langle \omega^*(\theta_{t+1}) -\omega^*(\theta_{t}), \nabla_\omega M(\theta_{t},\omega^*(\theta_{t}),D)\right\rangle \nonumber\\
&\quad-2\beta_t\left\langle \omega^*(\theta_{t+1}) -\omega^*(\theta_{t}), \nabla_\omega M(\theta_{t+1},\omega_t,D)-\nabla_\omega M(\theta_{t+1},\omega^*(\theta_{t}),D) \right\rangle\nonumber\\
&\quad-2\beta_t\left\langle  \omega^*(\theta_{t+1}) -\omega^*(\theta_{t}), \nabla_\omega M(\theta_{t+1},\omega^*(\theta_{t}),D)-\nabla_\omega M(\theta_{t},\omega^*(\theta_{t}),D) \right\rangle\nonumber\\
&\quad-2\beta_t\left\langle \omega^*(\theta_{t+1}) -\omega^*(\theta_{t}), \widehat{\nabla}_\omega M(\theta_{t+1},\omega_t,D_t)-\nabla_\omega M(\theta_{t+1},\omega_t,D)\right\rangle\nonumber\\
&\quad-2\beta_t^2\left\langle  \nabla_\omega M(\theta_{t+1},\omega_t,D)-\nabla_\omega M(\theta_{t+1},\omega^*(\theta_{t}),D), \nabla_\omega M(\theta_{t},\omega^*(\theta_{t}),D)\right\rangle\nonumber\\
&\quad-2\beta_t^2\left\langle  \nabla_\omega M(\theta_{t+1},\omega_t,D)-\nabla_\omega M(\theta_{t+1},\omega^*(\theta_{t}),D), \nabla_\omega M(\theta_{t+1},\omega^*(\theta_{t}),D)-\nabla_\omega M(\theta_{t},\omega^*(\theta_{t}),D)\right\rangle\nonumber\\
&\quad-2\beta_t^2\left\langle  \nabla_\omega M(\theta_{t+1},\omega_t,D)-\nabla_\omega M(\theta_{t+1},\omega^*(\theta_{t}),D), \widehat{\nabla}_\omega M(\theta_{t+1},\omega_t,D_t)-\nabla_\omega M(\theta_{t+1},\omega_t,D)\right\rangle\nonumber\\
&\stackrel{(b)}{\leq}\left\lVert \omega_{t}-\omega^*(\theta_{t})\right\rVert^2+\left\lVert\omega^*(\theta_{t+1}) -\omega^*(\theta_{t}) \right\rVert^2 \nonumber\\
&\quad+\beta_t^2 \left\lVert \nabla_\omega M(\theta_{t+1},\omega_t,D)-\nabla_\omega M(\theta_{t+1},\omega^*(\theta_{t}),D)\right\rVert^2\nonumber\\
 &\quad+3\beta_t^2 \left\lVert\nabla_\omega M(\theta_{t},\omega^*(\theta_{t}),D)\right\rVert^2
 +3\beta_t^2 \left\lVert\nabla_\omega M(\theta_{t+1},\omega^*(\theta_{t}),D)-\nabla_\omega M(\theta_{t},\omega^*(\theta_{t}),D)\right\rVert^2 \nonumber\\
 &\quad+3\beta_t^2 \left\lVert \widehat{\nabla}_\omega M(\theta_{t+1},\omega_t,D_t)-\nabla_\omega M(\theta_{t+1},\omega_t,D) \right\rVert^2 \nonumber\\
 &\quad+\frac{\beta_t\lambda_\omega}{7}\left\lVert \omega_{t}-\omega^*(\theta_{t}) \right\rVert^2+ \frac{7}{\beta_t\lambda_\omega}\left\lVert \omega^*(\theta_{t+1}) -\omega^*(\theta_{t})\right\rVert^2\nonumber\\
&\quad - 2\beta_t\lambda_\omega \left\lVert\omega_t-\omega^*(\theta_t)\right\rVert^2\nonumber\\
&\quad+\frac{\beta_t\lambda_\omega}{7}\left\lVert \omega_{t}-\omega^*(\theta_{t})\right\rVert^2+\frac{7\beta_t}{\lambda_\omega}\left\lVert \nabla_\omega M(\theta_{t+1},\omega^*(\theta_{t}),D)-\nabla_\omega M(\theta_{t},\omega^*(\theta_{t}),D)\right\rVert^2 \nonumber\\
&\quad+\frac{\beta_t\lambda_\omega}{7}\left\lVert \omega_{t}-\omega^*(\theta_{t})\right\rVert^2+\frac{7\beta_t}{\lambda_\omega}\left\lVert \widehat{\nabla}_\omega M(\theta_{t+1},\omega_t,D_t)-\nabla_\omega M(\theta_{t+1},\omega_t,D)\right\rVert^2 \nonumber\\
&\quad-2\beta_t\left\langle \omega_{t}-\omega^*(\theta_{t}), \nabla_\omega M(\theta_{t},\omega^*(\theta_{t}),D)\right\rangle -2\beta_t\left\langle \omega^*(\theta_{t+1}) -\omega^*(\theta_{t}), \nabla_\omega M(\theta_{t},\omega^*(\theta_{t}),D)\right\rangle \nonumber\\
&\quad+7\beta_t\lambda_\omega\left\lVert \omega^*(\theta_{t+1}) -\omega^*(\theta_{t})\right\rVert^2+ \frac{\beta_t}{7\lambda_\omega}\left\lVert\nabla_\omega M(\theta_{t+1},\omega_t,D)-\nabla_\omega M(\theta_{t+1},\omega^*(\theta_{t}),D) \right\rVert^2\nonumber\\
&\quad+\beta_t\lambda_\omega\left\lVert \omega^*(\theta_{t+1}) -\omega^*(\theta_{t})\right\rVert^2+ \frac{\beta_t}{\lambda_\omega}\left\lVert\nabla_\omega M(\theta_{t+1},\omega^*(\theta_{t}),D)-\nabla_\omega M(\theta_{t},\omega^*(\theta_{t}),D) \right\rVert^2\nonumber\\
&\quad+\beta_t \lambda_\omega\left\lVert \omega^*(\theta_{t+1}) -\omega^*(\theta_{t})\right\rVert^2+ \frac{\beta_t}{\lambda_\omega}\left\lVert\widehat{\nabla}_\omega M(\theta_{t+1},\omega_t,D_t)-\nabla_\omega M(\theta_{t+1},\omega_t,D)\right\rVert^2\nonumber\\
&\quad+\frac{3\beta_t}{7\lambda_\omega}\left\lVert  \nabla_\omega M(\theta_{t+1},\omega_t,D)-\nabla_\omega M(\theta_{t+1},\omega^*(\theta_{t}),D)\right\rVert^2 \nonumber\\
&\quad+7\beta_t^3\lambda_\omega\left\lVert\nabla_\omega M(\theta_{t},\omega^*(\theta_{t}),D)\right\rVert^2\nonumber\\
&\quad+7\beta_t^3\lambda_\omega \left\lVert\nabla_\omega M(\theta_{t+1},\omega^*(\theta_{t}),D)-\nabla_\omega M(\theta_{t},\omega^*(\theta_{t}),D)\right\rVert^2\nonumber\\
&\quad+7\beta_t^3\lambda_\omega \left\lVert\widehat{\nabla}_\omega M(\theta_{t+1},\omega_t,D_t)-\nabla_\omega M(\theta_{t+1},\omega_t,D)\right\rVert^2\nonumber\\
&=\left(1- 2\beta_t\lambda_\omega +\frac{3\beta_t\lambda_\omega}{7}\right)\left\lVert \omega_{t}-\omega^*(\theta_{t})\right\rVert^2+ \left( 1+ \frac{7}{\beta_t\lambda_\omega}+9\lambda_\omega\beta_t\right)\left\lVert\omega^*(\theta_{t+1}) -\omega^*(\theta_{t}) \right\rVert^2 \nonumber\\
&\quad+\left(3\beta_t^2+7\beta_t^3\lambda_\omega \right)\left\lVert\nabla_\omega M(\theta_{t},\omega^*(\theta_{t}),D)\right\rVert^2\nonumber\\
 &\quad+\left(\beta_t^2+\frac{4\beta_t}{7\lambda_\omega}\right) \left\lVert \nabla_\omega M(\theta_{t+1},\omega_t,D)-\nabla_\omega M(\theta_{t+1},\omega^*(\theta_{t}),D)\right\rVert^2\nonumber\\
 &\quad+\left(3\beta_t^2+\frac{8\beta_t}{\lambda_\omega}+7\beta_t^3\lambda_\omega\right) \left\lVert \nabla_\omega M(\theta_{t+1},\omega^*(\theta_{t}),D)-\nabla_\omega M(\theta_{t},\omega^*(\theta_{t}),D)\right\rVert^2\nonumber\\
 &\quad+\left(3\beta_t^2+\frac{8\beta_t}{\lambda_\omega} +7\beta_t^3\lambda_\omega\right)\left\lVert \widehat{\nabla}_\omega M(\theta_{t+1},\omega_t,D_t)-\nabla_\omega M(\theta_{t+1},\omega_t,D) \right\rVert^2 \nonumber\\
&\quad-2\beta_t\left\langle \omega_{t}-\omega^*(\theta_{t}), \nabla_\omega M(\theta_{t},\omega^*(\theta_{t}),D)\right\rangle -2\beta_t\left\langle \omega^*(\theta_{t+1}) -\omega^*(\theta_{t}), \nabla_\omega M(\theta_{t},\omega^*(\theta_{t}),D)\right\rangle \nonumber\\
&\stackrel{(c)}{\leq}\left(1- 2\beta_t\lambda_\omega +\frac{3\beta_t\lambda_\omega}{7}\right)\left\lVert \omega_{t}-\omega^*(\theta_{t})\right\rVert^2+ \left( 1+ \frac{7}{\beta_t\lambda_\omega}+9\lambda_\omega\beta_t\right)\left\lVert\omega^*(\theta_{t+1}) -\omega^*(\theta_{t}) \right\rVert^2 \nonumber\\
&\quad+\left(3\beta_t^2+7\beta_t^3\lambda_\omega\right) \left\lVert\nabla_\omega M(\theta_{t},\omega^*(\theta_{t}),D)\right\rVert^2\nonumber\\
 &\quad+\left(\beta_t^2\lambda_\omega^2+\frac{4\beta_t\lambda_\omega}{7}\right) \left\lVert \omega_t-\omega^*(\theta_{t})\right\rVert^2+\left(3\beta_t^2+\frac{8\beta_t}{\lambda_\omega}+7\beta_t^3\lambda_\omega\right) nL_m^2\left\lVert \theta_{t+1}-\theta_{t}\right\rVert^2\nonumber\\
 &\quad+\left(3\beta_t^2+\frac{8\beta_t}{\lambda_\omega} +7\beta_t^3\lambda_\omega\right)\left\lVert \widehat{\nabla}_\omega M(\theta_{t+1},\omega_t,D_t)-\nabla_\omega M(\theta_{t+1},\omega_t,D) \right\rVert^2 \nonumber\\
&\quad-2\beta_t\left\langle \omega_{t}-\omega^*(\theta_{t}), \nabla_\omega M(\theta_{t},\omega^*(\theta_{t}),D)\right\rangle -2\beta_t\left\langle \omega^*(\theta_{t+1}) -\omega^*(\theta_{t}), \nabla_\omega M(\theta_{t},\omega^*(\theta_{t}),D)\right\rangle \nonumber\\
&\stackrel{(d)}{\leq}\left(1- \beta_t\lambda_\omega +\beta_t^2\lambda_\omega^2\right)\left\lVert \omega_{t}-\omega^*(\theta_{t})\right\rVert^2+\left(3\beta_t^2+7\beta_t^3\lambda_\omega\right) \left\lVert\nabla_\omega M(\theta_{t},\omega^*(\theta_{t}),D)\right\rVert^2\nonumber\\
 &\quad+\left(3\beta_t^2+\frac{8\beta_t}{\lambda_\omega}+7\beta_t^3\lambda_\omega \right)\left\lVert \widehat{\nabla}_\omega M(\theta_{t+1},\omega_t,D_t)-\nabla_\omega M(\theta_{t+1},\omega_t,D) \right\rVert^2 \nonumber\\
&\quad-2\beta_t\left\langle \omega_{t}-\omega^*(\theta_{t}), \nabla_\omega M(\theta_{t},\omega^*(\theta_{t}),D)\right\rangle -2\beta_t\left\langle \omega^*(\theta_{t+1}) -\omega^*(\theta_{t}), \nabla_\omega M(\theta_{t},\omega^*(\theta_{t}),D)\right\rangle \nonumber\\
&\quad+\left( \frac{2}{\lambda_\omega}+ \frac{14}{\beta_t\lambda_\omega^2}+18\beta_t\right)\left\langle\nabla_{\omega} M(\theta_{t+1},\omega^*(\theta_{t+1}),D)-\nabla_{\omega} M(\theta_t,\omega^*(\theta_{t}),D), \omega^*(\theta_{t+1})-\omega^*(\theta_{t}) \right\rangle \nonumber\\
&\quad+ \left( 1+ 7\beta_t^3\lambda_\omega^3+3\beta_t^2\lambda_\omega^2+\frac{7}{\beta_t\lambda_\omega}+17\beta_t\lambda_\omega\right)\frac{nL_m^2}{\lambda_\omega^2}\left\lVert \theta_{t+1}-\theta_t\right\rVert^2.
\end{align}
The step $(a)$ follows the facts that $\left\lVert a-b-c-d \right\rVert^2 = \left\lVert a \right\rVert^2+\left\lVert b \right\rVert^2+\left\lVert c \right\rVert^2+\left\lVert d \right\rVert^2 -2\left\langle a,b \right\rangle -2\left\langle a,c \right\rangle-2\left\langle a,d \right\rangle -2\left\langle b,c \right\rangle-2\left\langle b,d \right\rangle-2\left\langle c,d \right\rangle$, $<a,b+c+d>=<a,b>+<a,c>+<a,d>$, and $\left\lVert a+b+c \right\rVert^2 \leq 3\left\lVert a \right\rVert^2+3\left\lVert b \right\rVert^2+3\left\lVert c \right\rVert^2$. The step $(b)$ follows from Lemma \ref{lm:cent_w_diff_bound-1}, and $\forall a,b,\,<a,b>\leq \frac{\lVert a \rVert^2}{2\delta} + \frac{\delta\lVert b \rVert^2}{2},\,\delta>0$. The step $(c)$ follows from Lemma \ref{lm:cent_M_lip-1} and \ref{lm:cent_M_lip-2}, and the step $(d)$ follows from Lemma \ref{lm:cent_w_diff_bound}.

Taking expectation on both sides of \eqref{eq:wt-w*_1}, we obtain
\begin{align}
\label{eq:wt-w*_2}
&\E\left[\left\lVert \omega_{t+1}-\omega^*(\theta_{t+1})  \right\rVert^2\right]\nonumber\\
&\leq \left(1- \beta_t\lambda_\omega+\beta_t^2\lambda_\omega^2\right)\E\left[\left\lVert \omega_{t}-\omega^*(\theta_{t})\right\rVert^2 \right]
+\left(3\beta_t^2+7\beta_t^3\lambda_\omega\right) \E\left[\left\lVert\nabla_\omega M(\theta_{t},\omega^*(\theta_{t}),D)\right\rVert^2\right]\nonumber\\
 &\quad+\left(3\beta_t^2+\frac{8\beta_t}{\lambda_\omega}+7\beta_t^3\lambda_\omega \right)\E\left[\left\lVert \widehat{\nabla}_\omega M(\theta_{t+1},\omega_t,D_t)-\nabla_\omega M(\theta_{t+1},\omega_t,D) \right\rVert^2\right] \nonumber\\
&\quad-2\beta_t\E\left[\left\langle \omega_{t}-\omega^*(\theta_{t}), \nabla_\omega M(\theta_{t},\omega^*(\theta_{t}),D)\right\rangle\right] -2\beta_t\E\left[\left\langle \omega^*(\theta_{t+1}) -\omega^*(\theta_{t}), \nabla_\omega M(\theta_{t},\omega^*(\theta_{t}),D)\right\rangle \right]\nonumber\\
&\quad+\left( \frac{2}{\lambda_\omega}+ \frac{14}{\beta_t\lambda_\omega^2}+18\beta_t\right)\E\left[\left\langle\nabla_{\omega} M(\theta_{t+1},\omega^*(\theta_{t+1}),D)-\nabla_{\omega} M(\theta_t,\omega^*(\theta_{t}),D), \omega^*(\theta_{t+1})-\omega^*(\theta_{t}) \right\rangle \right]\nonumber\\
&\quad+ \left( 1+ 7\beta_t^3\lambda_\omega^3+3\beta_t^2\lambda_\omega^2+\frac{7}{\beta_t\lambda_\omega}+17\beta_t\lambda_\omega\right)\frac{nL_m^2}{\lambda_\omega^2}\E\left[\left\lVert \theta_{t+1}-\theta_t\right\rVert^2\right]\nonumber\\
&\stackrel{(a)}{=} \left(1- \beta_t\lambda_\omega +\beta_t^2\lambda_\omega^2\right)\E\left[\left\lVert \omega_{t}-\omega^*(\theta_{t})\right\rVert^2 \right]
+\left(3\beta_t^2+7\beta_t^3\lambda_\omega\right) \E\left[\left\lVert\nabla_\omega M(\theta_{t},\omega^*(\theta_{t}),D)\right\rVert^2\right]\nonumber\\
 &\quad+\left(3\beta_t^2+\frac{8\beta_t}{\lambda_\omega}+7\beta_t^3\lambda_\omega \right)\E\left[\left\lVert \widehat{\nabla}_\omega M(\theta_{t+1},\omega_t,D_t)-\nabla_\omega M(\theta_{t+1},\omega_t,D) \right\rVert^2\right] \nonumber\\
&\quad+\left( 1+ 7\beta_t^3\lambda_\omega^3+3\beta_t^2\lambda_\omega^2+\frac{7}{\beta_t\lambda_\omega}+17\beta_t\lambda_\omega\right)\frac{nL_m^2}{\lambda_\omega^2}\E\left[\left\lVert \theta_{t+1}-\theta_t\right\rVert^2\right].
\end{align}
In the above, the step $(a)$ follows since $\E_{D\sim \mathcal{D^{\theta}}}\left[\nabla_{\omega}M(\theta,\omega^*(\theta),D)\right]=0$ from Assumption \ref{as:cent_w*}.
\end{proof}
\begin{lemma}
	\label{lm:hayes}
 Let $\{\hat{X}_i\}_{i=1}^n$ be i.i.d, vector-valued r.v.s., such that $\chi=\E\left[\hat{X}_i\right]$, $\forall \hat{X}_i$. Let $S_n=\frac{1}{n}\sum_{i=1}^n \hat{X}_i$. Assume $\forall i$, $\lVert \hat{X}_i \rVert \leq M$ a.s. Then,
 \begin{align*}
 \forall \epsilon>0,\; \p\left(\left\lVert S_n - \chi \right\rVert \geq \epsilon \right)\leq 2e^2 \exp \left( \frac{-n\epsilon^2}{8M^2}\right).
 \end{align*}
\end{lemma}
\begin{proof}
Let
\begin{align*}
Y_{n'}= \begin{cases}\frac{1}{2M}\sum\limits_{i=1}^{n'} \left(\hat{X}_i - \chi\right),& \textrm{ for }n'=\{1,\cdots,n\}\\
0,&\textrm{ for }n'=0.
\end{cases}
\end{align*}
Then $\{Y_{n'}\}_{n'=1}^{n}$ is a set of partial sums of bounded mean zero r.v.s. Hence it is a martingale, and $\forall n'>0$,
\begin{align*}
\left\lVert Y_{n'}-Y_{n'-1} \right\rVert &= \frac{1}{2M}\left\lVert\sum_{i=1}^{n'}\left(\hat{X}_i - \chi \right) - \sum_{i=1}^{n'-1}\left(\hat{X}_i - \chi \right) \right\rVert\\
&= \frac{1}{2M}\left\lVert\hat{X}_{n'} - \chi \right\rVert\\
&\leq \frac{1}{2M} 2M=1.
\end{align*}
Now,
\begin{align*}
\p\left(\left\lVert S_n - \chi \right\rVert \geq \epsilon\right)&=\p\left(\left\lVert Y_n\right\rVert\geq \frac{n\epsilon}{2M} \right)\\
&\stackrel{(a)}{\leq} 2e^2 \exp \left( \frac{-n\epsilon^2}{8M^2}\right).
\end{align*}
In the above, the step $(a)$ follows from \citep[Theorem 1.8]{hayes2005large}, and the fact that every martingale is a very weak martingale (cf. \citep[Definition 1.3]{hayes2005large}).
\end{proof}
\begin{lemma}
\label{lm:cent_mse_gard_L}
\begin{align*}
\E\left[\left \lVert \nabla_\theta L(\theta_{t,k},\omega^*(\theta_{t,k}),D) -\widehat{\nabla}_\theta L(\theta_{t,k},\omega^*(\theta_{t,k}),D_t)  \right \rVert^2 \right]\leq 16{L_l}^2e^2n_b\sum_{i=1}^n \E\left[\frac{1}{\left\lfloor p|D^i_t|\right\rfloor}\right]
\end{align*}
\end{lemma}
\begin{proof}
\begin{align}
\label{eq:cent_mse_gard_L1}
&\E\left[\left \lVert \nabla_\theta L(\theta_{t,k},\omega^*(\theta_{t,k}),D) -\widehat{\nabla}_\theta L(\theta_{t,k},\omega^*(\theta_{t,k}),D_t)  \right \rVert^2 \right]\nonumber\\
&= \E\Big[\big \lVert
\big[ \E_{\scalebox{0.9}{$\substack{T^i\sim \mathcal{U}(D^i,\lceil p|D^i|\rceil)\\(x^i,y^i)\sim \mathcal{U}(T^i,1)}$}}[\nabla_\theta l_{\theta}(x^i,y^i)]-\frac{1}{|T^i_{t,k}|}\sum\nolimits_{(x^i,y^i)\in T^i_{t,k}}\nabla_{\theta}l_{\theta_{t,k}}(x^i,y^i)\big]_{i=1}^{n} h_{\omega^*(\theta_{t,k})}\big \rVert^2 \Big]\nonumber\\
 &\stackrel{(a)}{\leq} \E\Big[\sum_{i=1}^n\big \lVert \E_{\scalebox{0.9}{$\substack{T^i\sim \mathcal{U}(D^i,\lceil p|D^i|\rceil)\\(x^i,y^i)\sim \mathcal{U}(T^i,1)}$}}[\nabla_\theta l_{\theta_{t,k}}(x^i,y^i)]- \frac{1}{|T^i_{t,k}|}\sum\nolimits_{(x^i,y^i)\in T^i_{t,k}}\nabla_{\theta}l_{\theta_{t,k}}(x^i,y^i)\big\rVert^2
 \left \lVert h_{\omega^*(\theta_{t,k})}\right \rVert^2 \Big]\nonumber\\
 &\stackrel{(b)}{\leq} \sum_{i=1}^n\E\Big[\big \lVert \E_{\scalebox{0.9}{$\substack{T^i\sim \mathcal{U}(D^i,\lceil p|D^i|\rceil)\\(x^i,y^i)\sim \mathcal{U}(T^i,1)}$}}[\nabla_\theta l_{\theta_{t,k}}(x^i,y^i)]- \frac{1}{|T^i_{t,k}|}\sum\nolimits_{(x^i,y^i)\in T^i_{t,k}}\nabla_{\theta}l_{\theta_{t,k}}(x^i,y^i)\big\rVert^2
 \Big]\nonumber\\
 &=\sum_{i=1}^n\E\Big[\big \lVert \E_{T^i\sim \mathcal{U}(D^i,\lceil p|D^i|\rceil)}\big[\E_{(x^i,y^i)\sim \mathcal{U}(T^i,1)}[\nabla_\theta l_{\theta_{t,k}}(x^i,y^i)]- \frac{1}{|T^i_{t,k}|}\sum\nolimits_{(x^i,y^i)\in T^i_{t,k}}\nabla_{\theta}l_{\theta_{t,k}}(x^i,y^i)\big]\big\rVert^2
 \Big]\nonumber\\
  &\stackrel{(c)}{\leq}\sum_{i=1}^n\E\Big[ \E_{T^i\sim \mathcal{U}(D^i,\lceil p|D^i|\rceil)}\big[\big \lVert\E_{(x^i,y^i)\sim \mathcal{U}(T^i,1)}[\nabla_\theta l_{\theta_{t,k}}(x^i,y^i)]- \frac{1}{|T^i_{t,k}|}\sum\nolimits_{(x^i,y^i)\in T^i_{t,k}}\nabla_{\theta}l_{\theta_{t,k}}(x^i,y^i)\big\rVert^2\big]
 \Big]\nonumber\\
  &= \sum_{i=1}^n\E\Big[ \E_{T^i\sim \mathcal{U}(D^i,\lceil p|D^i|\rceil)}\big[\big \lVert \E_{(x^i,y^i)\sim \mathcal{U}(T^i,1)}\left[\nabla_\theta l_{\theta_{t,k}}(x^i,y^i)\right]- \frac{1}{|T^i_{t,k}|}\sum\nolimits_{(x^i,y^i)\in T^i_{t,k}}\nabla_{\theta}l_{\theta_{t,k}}(x^i,y^i)\big\rVert^2
 \big]\Big]
\end{align}
The step $(a)$ follows from the fact that for a vector $a$ and a matrix $A$, $\lVert Aa \rVert^2 \leq  \Vert a \rVert^2\left(\sum_{i=1}^n\lVert  A_i \rVert^2\right)$, where $A_i$ is the $i^{th}$ column of $A$. The step $(b)$ follows since $\forall \theta \in \R^d,\; \left \lVert h_{\omega^*(\theta)}\right \rVert \leq 1$. The step $(c)$ follows from the fact that $|\E[x]|^2\leq \E[|x|]^2$

From Assumption \ref{as:grad_l_bound}, we have $\nabla_{\theta}l_{\theta}(x,y)<=L_l$. Using Lemma \ref{lm:hayes}, we obtain
$\forall i$ and $\forall \theta_{t,k}$,
\begin{align}
\label{eq:mse_p}
&\p\Big(\big\lVert \E_{(x^i,y^i)\sim \mathcal{U}(T^i,1)}\left[\nabla_\theta l_{\theta_{t,k}}(x^i,y^i)\right]- \frac{1}{|T^i_{t,k}|}\sum\nolimits_{(x^i,y^i)\in T^i_{t,k}}\nabla_{\theta}l_{\theta_{t,k}}(x^i,y^i) \big\rVert \geq \epsilon\Big)\leq 2e^2 \exp \left( \frac{-b^i_{t,k}\epsilon^2}{8{L_l}^2}\right),
\end{align}
and
\begin{align}
\label{eq:mse_e}
&\E_{T^i\sim \mathcal{U}(D^i,\lceil p|D^i|\rceil)}\Big[\big \lVert \E_{(x^i,y^i)\sim \mathcal{U}(T^i,1)}\left[\nabla_\theta l_{\theta_{t,k}}(x^i,y^i)\right]- \frac{1}{|T^i_{t,k}|}\sum\limits_{(x^i,y^i)\in T^i_{t,k}}\nabla_{\theta}l_{\theta_{t,k}}(x^i,y^i)\big\rVert^2
 \Big]\nonumber\\
 &=\int_0^\infty \p\Big(\big\lVert \E_{(x^i,y^i)\sim \mathcal{U}(T^i,1)}\left[\nabla_\theta l_{\theta_{t,k}}(x^i,y^i)\right]- \frac{1}{|T^i_{t,k}|}\sum\limits_{(x^i,y^i)\in T^i_{t,k}}\nabla_{\theta}l_{\theta_{t,k}}(x^i,y^i) \big\rVert \geq \sqrt{\epsilon}\Big)\, d\epsilon\nonumber\\
 &\stackrel{(a)}{\leq} \int_0^\infty 2e^2 \exp \left( \frac{-T^i_{t,k}\epsilon}{8{L_l}^2}\right) \,d\epsilon\nonumber\\
 &\stackrel{(b)}{=}\frac{16{L_l}^2e^2}{b^i_{t,k}},
\end{align}
where the step $(a)$ follows from \eqref{eq:mse_p}, and  the step $(b)$ follows from the fact that $\int_0^\infty \exp(-a\epsilon)\,d\epsilon=\frac{1}{a},\;\forall a>0$.

Applying \eqref{eq:mse_e} in \eqref{eq:cent_mse_gard_L1}, we obtain
\begin{align}
\label{eq:cent_mse_gard_L2}
\E\left[\left \lVert \nabla_\theta L(\theta_{t,k},\omega^*(\theta_{t,k}),D) -\widehat{\nabla}_\theta L(\theta_{t,k},\omega^*(\theta_{t,k}),D_t)  \right \rVert^2 \right]
&\leq 16{L_l}^2e^2\sum_{i=1}^n \E\left[\frac{1}{T^i_{t,k}}\right]
= 16{L_l}^2e^2\sum_{i=1}^n \E\left[\frac{n_b}{|T^i_t|}\right]\nonumber\\
&= 16{L_l}^2e^2n_b\sum_{i=1}^n \E\left[\frac{1}{\left\lfloor p|D^i_t|\right\rfloor}\right].
\end{align}
\end{proof}
\begin{lemma}
\label{lm:cent_hat_nabla_L_diff}
\begin{align*}
 \forall t,k, \;\E\left[\left \lVert \widehat{\nabla}_\theta L(\theta_{t,k},\omega^*(\theta_{t,k}),D_t) -\widehat{\nabla}_\theta L(\theta_{t,k},\omega_t,D_t)\right \rVert^2 \right] \leq nL_l^2\E\left[\left \lVert \omega^*(\theta_{t,k})-\omega_t \right \rVert^2 \right].
\end{align*}
\end{lemma}
\begin{proof}
\begin{align}
&\E\left[\left \lVert \widehat{\nabla}_\theta L(\theta_{t,k},\omega^*(\theta_{t,k}),D_t) -\widehat{\nabla}_\theta L(\theta_{t,k},\omega_t,D_t)\right \rVert^2 \right]  \nonumber\\
&= \E\big[\big \lVert \big[\frac{1}{|T^i_{t,k}|}\sum\nolimits_{(x^i,y^i)\in T^i_{t,k}}\nabla_{\theta}l_{\theta_{t,k}}(x^i,y^i)\big]_{i=1}^{n} \left(h_{\omega^*(\theta_{t,k})}-h_{\omega_t}\right)\big \rVert^2 \big]\nonumber\\
 &\stackrel{(a)}{\leq} \E\big[\sum_{i=1}^{n}\big \lVert \frac{1}{|T^i_{t,k}|}\sum\nolimits_{(x^i,y^i)\in T^i_{t,k}}\nabla_{\theta}l_{\theta_{t,k}}(x^i,y^i)\big \rVert^2\left \lVert h_{\omega^*(\theta_{t,k})}-h_{\omega_t}\right \rVert^2 \big]\nonumber\\
 &\stackrel{(b)}{\leq} nL_l^2\E\left[\left \lVert h_{\omega^*(\theta_{t,k})}-h_{\omega_t}\right \rVert^2 \right]\nonumber\\
 &\stackrel{(c)}{\leq} nL_l^2\E\left[\left \lVert \omega^*(\theta_{t,k})-\omega_t \right \rVert^2 \right].
\end{align}
The step $(a)$ follows from the fact that for a vector $a$ and a matrix $A$, $\lVert Aa \rVert^2 \leq  \Vert a \rVert^2\left(\sum_{i=1}^n\lVert  A_i \rVert^2\right)$, where $A_i$ is the $i^{th}$ column of $A$. The step $(b)$ follows from Assumption \ref{as:grad_l_bound} and from the fact that $\lVert \sum_{i=1}^{n} a_i\rVert^2 \leq n \sum_{i=1}^{n}\lVert  a_i\rVert^2$. The step $(c)$ follows from Lemma \ref{lm:cent_softmax_lip}.
\end{proof}
\begin{lemma}
\label{lm:cent_hat_grad_L_bnd}
\begin{align*}
\E\left[\left\lVert  \widehat{\nabla}_\theta L(\theta_{t,k},\omega_t,D_t) \right\rVert^2\right]\leq nL_l^2.
\end{align*}
\end{lemma}
\begin{proof}
\begin{align}
\E\left[\left\lVert  \widehat{\nabla}_\theta L(\theta_{t,k},\omega_t,D_t) \right\rVert^2\right]
&= \E\big[\big \lVert \big[\frac{1}{|T^i_{t,k}|}\sum\nolimits_{(x^i,y^i)\in T^i_{t,k}}\nabla_{\theta}l_{\theta_{t,k}}(x^i,y^i)\big]_{i=1}^{n}
h_{\omega_t}\big \rVert^2 \big]\nonumber\\
 &\stackrel{(a)}{\leq} \E\big[\sum_{i=1}^{n}\big \lVert \frac{1}{|T^i_{t,k}|}\sum\nolimits_{(x^i,y^i)\in T^i_{t,k}}\nabla_{\theta}l_{\theta_{t,k}}(x^i,y^i)\big \rVert^2 \left\lVert h_{\omega_t}\right \rVert^2 \big]\nonumber\\
 &\stackrel{(b)}{\leq} nL_l^2\E\left[\left \lVert h_{\omega_t}\right \rVert^2 \right]\nonumber\\
  &\stackrel{(c)}{\leq} nL_l^2.
\end{align}
The step $(a)$ follows from the fact that for a vector $a$ and a matrix $A$, $\lVert Aa \rVert^2 \leq  \Vert a \rVert^2\left(\sum_{i=1}^n\lVert  A_i \rVert^2\right)$, where $A_i$ is the $i^{th}$ column of $A$. The step $(b)$ follows from Assumption \ref{as:grad_l_bound} and from the fact that $\lVert \sum_{i=1}^{n} a_i\rVert^2 \leq n \sum_{i=1}^{n}\lVert  a_i\rVert^2$. The step $(c)$ follows from the fact that $\left \lVert h_{\omega_t}\right \rVert^2 \leq 1$.
\end{proof}
\subsubsection*{Proof [\textbf{Theorem \ref{thm:abt}}]}
\label{subsubsec:pf_thm_abt}
Using the fundamental theorem of calculus, we obtain
    \begin{align}
        & J(\theta_{t,k+1}) -J(\theta_{t,k}) \nonumber\\
        &=\langle \nabla J(\theta_{t,k}),  \theta_{t,k+1}-\theta_{t,k}  \rangle
        + \int_0^1 \left\langle  \nabla J(\theta_{t,k}+\tau(\theta_{t,k+1}-\theta_{t,k}))-\nabla J(\theta_{t,k}), \theta_{t,k+1} - \theta_{t,k} \right\rangle d\tau\nonumber\\
        &\leq\langle \nabla J(\theta_{t,k}), \theta_{t,k+1} - \theta_{t,k} \rangle
        +\int_0^1 \left\lVert\nabla J(\theta_{t,k}+\tau(\theta_{t,k+1} - \theta_{t,k}))-\nabla J(\theta_{t,k})\right\rVert \left\lVert \theta_{t,k+1} - \theta_{t,k} \right\rVert d\tau\nonumber\\
        &\stackrel{(a)}{\leq} \left \langle \nabla J(\theta_{t,k}), \theta_{t,k+1} - \theta_{t,k} \right \rangle
        + L_{l'}\left\lVert \theta_{t,k+1} - \theta_{t,k} \right\rVert^2  \int_0^1 \tau d\tau \nonumber\\
        &= \left \langle \nabla J(\theta_{t,k}), \theta_{t,k+1} - \theta_{t,k} \right \rangle + \frac{L_{l'}}{2}\left\lVert \theta_{t,k+1} - \theta_{t,k} \right\rVert^2 \nonumber\\
        &= \alpha_{t} \left \langle \nabla J(\theta_{t,k}),  -\widehat{\nabla}_\theta L(\theta_{t,k},\omega_t,D_t) \right \rangle + \frac{L_{l'}\alpha_{t}^2}{2}\left\lVert  \widehat{\nabla}_\theta L(\theta_{t,k},\omega_t,D_t) \right\rVert^2 \nonumber\\
         &= \alpha_{t} \left \langle \nabla J(\theta_{t,k}), \nabla_\theta L(\theta_{t,k},\omega^*(\theta_{t,k}),D) -\widehat{\nabla}_\theta L(\theta_{t,k},\omega_t,D_t)  \right \rangle -\alpha_{t} \left \langle \nabla J(\theta_{t,k}), \nabla_\theta L(\theta_{t,k},\omega^*(\theta_{t,k}),D)  \right \rangle\nonumber\\
         &\quad+ \frac{L_{l'}\alpha_{t}^2}{2}\left\lVert  \widehat{\nabla}_\theta L(\theta_{t,k},\omega_t,D_t) \right\rVert^2 \nonumber\\
         &\stackrel{(b)}{\leq} \frac{\alpha_{t}}{2} \left\lVert \nabla J(\theta_{t,k})\right\rVert^2+ \frac{\alpha_{t}}{2}\left \lVert \nabla_\theta L(\theta_{t,k},\omega^*(\theta_{t,k}),D) -\widehat{\nabla}_\theta L(\theta_{t,k},\omega_t,D_t)  \right \rVert^2 -\alpha_{t} \left \langle \nabla J(\theta_{t,k}), \nabla_\theta L(\theta_{t,k},\omega^*(\theta_{t,k}),D)  \right \rangle\nonumber\\
         &\quad+ \frac{L_{l'}\alpha_{t}^2}{2}\left\lVert  \widehat{\nabla}_\theta L(\theta_{t,k},\omega_t,D_t) \right\rVert^2 \label{eq:cent_pf1}
    \end{align}
In the above, the step $(a)$ follows from Lemma \ref{lm:grad_J_lip}, and the step $(b)$ follows from the fact that $2\langle x, y \rangle \leq \lVert x \rVert^2+ \Vert y \rVert^2$.
Taking expectations on both sides of \eqref{eq:cent_pf1}, we obtain
\begin{align}
& \E\left[J(\theta_{t,k+1})\right] - \E\left[J(\theta_{t,k})\right] \nonumber\\
 &\leq \frac{\alpha_{t}}{2} \E\left[\left\lVert \nabla J(\theta_{t,k})\right\rVert^2 \right]+ \frac{\alpha_{t}}{2} \E\left[\left \lVert \nabla_\theta L(\theta_{t,k},\omega^*(\theta_{t,k}),D) -\widehat{\nabla}_\theta L(\theta_{t,k},\omega_t,D_t)  \right \rVert^2 \right]\nonumber\\
&\quad-\alpha_{t}\E\left[ \left \langle \nabla J(\theta_{t,k}), \nabla_\theta L(\theta_{t,k},\omega^*(\theta_{t,k}),D)  \right \rangle\right] + \frac{L_{l'}\alpha_{t}^2}{2}\E\left[\left\lVert  \widehat{\nabla}_\theta L(\theta_{t,k},\omega_t,D_t) \right\rVert^2\right]\nonumber\\
&\stackrel{(a)}{=} \frac{\alpha_{t}}{2} \E\left[\left\lVert \nabla J(\theta_{t,k})\right\rVert^2 \right]+ \frac{\alpha_{t}}{2} \E\left[\left \lVert \nabla_\theta L(\theta_{t,k},\omega^*(\theta_{t,k}),D) -\widehat{\nabla}_\theta L(\theta_{t,k},\omega_t,D_t)  \right \rVert^2 \right]-\alpha_{t}\E\left[\left\lVert \nabla J(\theta_{t,k})\right\rVert^2 \right] \nonumber\\
&\quad+ \frac{L_{l'}\alpha_{t}^2}{2}\E\left[\left\lVert  \widehat{\nabla}_\theta L(\theta_{t,k},\omega_t,D_t) \right\rVert^2\right]\nonumber\\
&= -\frac{\alpha_{t}}{2} \E\left[\left\lVert \nabla J(\theta_{t,k})\right\rVert^2 \right]+ \frac{\alpha_{t}}{2} \E\left[\left \lVert \nabla_\theta L(\theta_{t,k},\omega^*(\theta_{t,k}),D) -\widehat{\nabla}_\theta L(\theta_{t,k},\omega_t,D_t)  \right \rVert^2 \right]\nonumber\\
&\quad+ \frac{L_{l'}\alpha_{t}^2}{2}\E\left[\left\lVert  \widehat{\nabla}_\theta L(\theta_{t,k},\omega_t,D_t) \right\rVert^2\right]\nonumber\\
&= -\frac{\alpha_{t}}{2} \E\left[\left\lVert \nabla J(\theta_{t,k})\right\rVert^2 \right]+ \alpha_{t}\E\left[\left \lVert \nabla_\theta L(\theta_{t,k},\omega^*(\theta_{t,k}),D) -\widehat{\nabla}_\theta L(\theta_{t,k},\omega^*(\theta_{t,k}),D_t)  \right \rVert^2 \right]\nonumber\\
& \quad+\alpha_{t}\E\left[\left \lVert \widehat{\nabla}_\theta L(\theta_{t,k},\omega^*(\theta_{t,k}),D_t) -\widehat{\nabla}_\theta L(\theta_{t,k},\omega_t,D_t)  \right \rVert^2 \right]+ \frac{L_{l'}\alpha_{t}^2}{2}\E\left[\left\lVert  \widehat{\nabla}_\theta L(\theta_{t,k},\omega_t,D_t) \right\rVert^2\right],\label{eq:a}
\end{align}
where the step $(a)$ follows from Assumption \ref{as:cent_w*} since $\nabla_\theta J(\theta)= \E_{D\sim\mathcal{D^{\theta}}}\left[\nabla_{\theta} L(\theta,\omega^*(\theta),D)\right]$.

Let
\begin{align}
\label{eq:r_t_k}
R_{t,k}=\E\left[J(\theta_{t,k})\right]+ r_{t,k} \E\left[ \left\lVert \omega^*(\theta_{t,k})-\omega_t \right\rVert^2\right].
\end{align}
Now,
\begin{align}
\label{eq:r_t_k-1}
&R_{t,{k+1}}\nonumber\\
&=\E\left[J(\theta_{t,{k+1}})\right]+ r_{t,k+1} \E\left[ \left\lVert \omega^*(\theta_{t,{k+1}})-\omega_t \right\rVert^2\right]\nonumber\\
&\stackrel{(a)}{\leq} \E\left[J(\theta_{t,k})\right] + r_{t,k+1} \E\left[ \left\lVert \omega^*(\theta_{t,{k+1}})-\omega_t \right\rVert^2\right] -\frac{\alpha_{t}}{2} \E\left[\left\lVert \nabla J(\theta_{t,k})\right\rVert^2 \right]\nonumber\\
&\quad+ \alpha_{t}\E\left[\left \lVert \nabla_\theta L(\theta_{t,k},\omega^*(\theta_{t,k}),D) -\widehat{\nabla}_\theta L(\theta_{t,k},\omega^*(\theta_{t,k}),D_t)  \right \rVert^2 \right]\nonumber\\
& \quad+\alpha_{t}\E\left[\left \lVert \widehat{\nabla}_\theta L(\theta_{t,k},\omega^*(\theta_{t,k}),D_t) -\widehat{\nabla}_\theta L(\theta_{t,k},\omega_t,D_t)  \right \rVert^2 \right]+ \frac{L_{l'}\alpha_{t}^2}{2}\E\left[\left\lVert  \widehat{\nabla}_\theta L(\theta_{t,k},\omega_t,D_t) \right\rVert^2\right] \nonumber\\
&\stackrel{(b)}{\leq} \E\left[J(\theta_{t,k})\right] + r_{t,k+1}  \left(1+\frac{1}{n_b}\right)\E\left[\left\lVert\omega^*(\theta_{t,k})- \omega_{t} \right\rVert^2\right]-\frac{\alpha_{t}}{2} \E\left[\left\lVert \nabla J(\theta_{t,k})\right\rVert^2 \right]\nonumber\\
&\quad+ \alpha_{t}\E\left[\left \lVert \nabla_\theta L(\theta_{t,k},\omega^*(\theta_{t,k}),D) -\widehat{\nabla}_\theta L(\theta_{t,k},\omega^*(\theta_{t,k}),D_t)  \right \rVert^2 \right]\nonumber\\
& \quad+\alpha_{t}\E\left[\left \lVert \widehat{\nabla}_\theta L(\theta_{t,k},\omega^*(\theta_{t,k}),D_t) -\widehat{\nabla}_\theta L(\theta_{t,k},\omega_t,D_t)  \right \rVert^2 \right]\nonumber\\
&\quad+ \alpha_{t}^2\left(\frac{L_{l'}}{2}+r_{t,k+1}\left(1+n_b\right)\frac{nL_m^2}{\lambda_\omega^2} \right)\E\left[\left\lVert  \widehat{\nabla}_\theta L(\theta_{t,k},\omega_t,D_t) \right\rVert^2\right] \nonumber\\
&\stackrel{(c)}{\leq} \E\left[J(\theta_{t,k})\right] + \left(r_{t,k+1}\left(1+\frac{1}{n_b}\right) + \alpha_{t}nL_l^2 \right)\E\left[\left\lVert\omega^*(\theta_{t,k})- \omega_{t} \right\rVert^2\right]-\frac{\alpha_{t}}{2} \E\left[\left\lVert \nabla J(\theta_{t,k})\right\rVert^2 \right]\nonumber\\
&\quad+ \alpha_{t}\E\left[\left \lVert \nabla_\theta L(\theta_{t,k},\omega^*(\theta_{t,k}),D) -\widehat{\nabla}_\theta L(\theta_{t,k},\omega^*(\theta_{t,k}),D_t)  \right \rVert^2 \right]\nonumber\\
&\quad+ \alpha_{t}^2\left(\frac{L_{l'}}{2}+r_{t,k+1}\left(1+n_b\right)\frac{nL_m^2}{\lambda_\omega^2} \right)\E\left[\left\lVert  \widehat{\nabla}_\theta L(\theta_{t,k},\omega_t,D_t) \right\rVert^2\right],
\end{align}
where the step $(a)$ follows from \eqref{eq:a}. The step $(b)$ follows from Lemma \ref{lm:cent_w_rec} with $\delta=\nicefrac{1}{n_b}$, and step $(c)$ follows from Lemma \ref{lm:cent_hat_nabla_L_diff}.

Let
\begin{align}
\label{eq:cent_r_k_rec1}
r_{t,k}=
\begin{cases}
\alpha_{t}nL_l^2+r_{t,k+1}\left(1+\frac{1}{n_b}\right) & \textrm{ for } k \in \{0,\cdots,n_b-1\}\\
0& \textrm{ for } k\geq n_b.
\end{cases}
\end{align}
Solving the recursion in \eqref{eq:cent_r_k_rec1}, we obtain
\begin{align}
\label{eq:cent_r_k_rec2}
r_{t,k}=\alpha_{t}nL_l^2n_b\left(\left(1+\frac{1}{n_b}\right)^{n_b-k}-1\right).
\end{align}
We can see that
\begin{align}
\label{eq:cent_r_k_bnd}
r_{t,k}&\leq r_{t,0},\; \forall k, \textrm { and } \nonumber\\
r_{t,0}&= \alpha_{t}nL_l^2n_b\left(\left(1+\frac{1}{n_b}\right)^{n_b}-1\right)\nonumber\\
&\stackrel{(a)}{\leq}\alpha_{t}nL_l^2n_b\left(e-1\right)
\leq\frac{7}{4}\alpha_{t}nL_l^2n_b,
\end{align}
where the step $(a)$ follows since $\lim_{n\to \infty}\left(1+\nicefrac{1}{n}\right)^n = e$.

From \eqref{eq:cent_r_k_bnd}, we obtain
\begin{align}
\label{eq:cent_c_k}
\mathop{\argmax}_{k\in \{0,\cdots n_b-1\}}\alpha_{t}^2\left(\frac{L_{l'}}{2}+r_{t,k+1}\left(1+n_b\right)\frac{nL_m^2}{\lambda_\omega^2} \right)
&=\alpha_{t}^2\left(\frac{L_{l'}}{2}+r_{t,0}\left(1+n_b\right)\frac{nL_m^2}{\lambda_\omega^2} \right)\nonumber\\
&\leq \alpha_{t}^2\left(\frac{L_{l'}}{2}+ \alpha_{t}\frac{7n^2L_m^2L_l^2n_b\left(1+n_b\right)}{4\lambda_\omega^2} \right).
\end{align}
Applying \eqref{eq:r_t_k}, \eqref{eq:cent_r_k_rec1} and \eqref{eq:cent_c_k} in \eqref{eq:r_t_k-1}, we obtain
\begin{align}
\label{eq:r_t_k-2}
R_{t,{k+1}}
&\leq R_{t,{k}}-\frac{\alpha_{t}}{2} \E\left[\left\lVert \nabla J(\theta_{t,k})\right\rVert^2 \right]
+ \alpha_{t}\E\left[\left \lVert \nabla_\theta L(\theta_{t,k},\omega^*(\theta_{t,k}),D) -\widehat{\nabla}_\theta L(\theta_{t,k},\omega^*(\theta_{t,k}),D_t)  \right \rVert^2 \right]\nonumber\\
&\quad+ \left(\alpha_{t}^2\left(\frac{L_{l'}}{2}+ \alpha_{t}\frac{7n^2L_m^2L_l^2n_b\left(1+n_b\right)}{4\lambda_\omega^2} \right)\right)\E\left[\left\lVert  \widehat{\nabla}_\theta L(\theta_{t,k},\omega_t,D_t) \right\rVert^2\right],
\end{align}

Summing \eqref{eq:r_t_k-2} from $k=0,\cdots,n_b-1$, we obtain
\begin{align}
\label{eq:r_t_k-3}
R_{t,{n_b}}&\leq R_{t,{0}}- \sum_{k=0}^{n_b-1}\frac{\alpha_{t}}{2} \E\left[\left\lVert \nabla J(\theta_{t,k})\right\rVert^2 \right]\nonumber\\
&\quad+ \sum_{k=0}^{n_b-1}\alpha_{t}\E\left[\left \lVert \nabla_\theta L(\theta_{t,k},\omega^*(\theta_{t,k}),D) -\widehat{\nabla}_\theta L(\theta_{t,k},\omega^*(\theta_{t,k}),D_t)  \right \rVert^2 \right]\nonumber\\
&\quad+\sum_{k=0}^{n_b-1} \left(\alpha_{t}^2\left(\frac{L_{l'}}{2}+ \alpha_{t}\frac{7n^2L_m^2L_l^2n_b\left(1+n_b\right)}{4\lambda_\omega^2} \right)\right)\E\left[\left\lVert  \widehat{\nabla}_\theta L(\theta_{t,k},\omega_t,D_t) \right\rVert^2\right].
\end{align}
From \eqref{eq:r_t_k}, we obtain
\begin{align}
\label{eq:R_t_nb}
R_{t,{n_b}} &= \E\left[J(\theta_{t,{n_b}})\right]+ r_{t,n_b} \E\left[ \left\lVert \omega^*(\theta_{t,{n_b}})-\omega_t \right\rVert^2\right] \nonumber\\
&\stackrel{(a)}{=} \E\left[J(\theta_{t,{n_b}})\right]\stackrel{(b)}{=} \E\left[J(\theta_{t+1})\right],
\end{align}
and
\begin{align}
\label{eq:R_t_0}
R_{t,{0}}&= \E\left[J(\theta_{t,0})\right]+ r_{t,0} \E\left[ \left\lVert \omega^*(\theta_{t,0})-\omega_t \right\rVert^2\right]\nonumber\\
&\stackrel{(c)}{=} \E\left[J(\theta_{t})\right]+ r_{t,0} \E\left[ \left\lVert \omega^*(\theta_{t})-\omega_t \right\rVert^2\right].
\end{align}
where step $(a)$ follows since $r_{t,n_b}=0$ from \eqref{eq:cent_r_k_rec1}, and step $(b)$ and $(c)$ follows since $\theta_{t,{n_b}}=\theta_{t+1}$ and $\theta_{t,0}=\theta_{t}$ from \eqref{eq:abt_theta_it}.

Applying \eqref{eq:R_t_nb} and \eqref{eq:R_t_0} in \eqref{eq:r_t_k-3}, we obtain
\begin{align}
\label{eq:r_t_k-4}
\E\left[J(\theta_{t+1})\right]&\leq \E\left[J(\theta_{t})\right]+ r_{t,0} \E\left[ \left\lVert \omega^*(\theta_{t})-\omega_t \right\rVert^2\right]- \sum_{k=0}^{n_b-1}\frac{\alpha_{t}}{2} \E\left[\left\lVert \nabla J(\theta_{t,k})\right\rVert^2 \right]\nonumber\\
&\quad+ \sum_{k=0}^{n_b-1}\alpha_{t}\E\left[\left \lVert \nabla_\theta L(\theta_{t,k},\omega^*(\theta_{t,k}),D) -\widehat{\nabla}_\theta L(\theta_{t,k},\omega^*(\theta_{t,k}),D_t)  \right \rVert^2 \right]\nonumber\\
&\quad+\sum_{k=0}^{n_b-1} \left(\alpha_{t}^2\left(\frac{L_{l'}}{2}+ \alpha_{t}\frac{7n^2L_m^2L_l^2n_b\left(1+n_b\right)}{4\lambda_\omega^2} \right)\right)\E\left[\left\lVert  \widehat{\nabla}_\theta L(\theta_{t,k},\omega_t,D_t) \right\rVert^2\right].
\end{align}


Let
\begin{align}
\label{eq:s_t}
S_{t}=\E\left[J(\theta_{t})\right]+ s_t \E\left[ \left\lVert \omega^*(\theta_{t})-\omega_t \right\rVert^2\right].
\end{align}
Now,
\begin{align}
\label{eq:s_t-1}
&S_{t+1}\nonumber\\
&=\E\left[J(\theta_{t+1})\right]+ s_{t+1} \E\left[ \left\lVert \omega^*(\theta_{t+1})-\omega_{t+1} \right\rVert^2\right]\nonumber\\
&\leq \E\left[J(\theta_{t})\right] + s_{t+1} \E\left[ \left\lVert \omega^*(\theta_{t+1})-\omega_{t+1} \right\rVert^2\right]+r_{t,0} \E\left[ \left\lVert \omega^*(\theta_{t})-\omega_t \right\rVert^2\right]\nonumber\\
&\quad+ \sum_{k=0}^{n_b-1}\alpha_{t}\E\left[\left \lVert \nabla_\theta L(\theta_{t,k},\omega^*(\theta_{t,k}),D) -\widehat{\nabla}_\theta L(\theta_{t,k},\omega^*(\theta_{t,k}),D_t)  \right \rVert^2 \right]\nonumber\\
&\quad+\sum_{k=0}^{n_b-1} \left(\alpha_{t}^2\left(\frac{L_{l'}}{2}+ \alpha_{t}\frac{7n^2L_m^2L_l^2n_b\left(1+n_b\right)}{4\lambda_\omega^2} \right)\right)\E\left[\left\lVert  \widehat{\nabla}_\theta L(\theta_{t,k},\omega_t,D_t) \right\rVert^2\right]
- \sum_{k=0}^{n_b-1}\frac{\alpha_{t}}{2} \E\left[\left\lVert \nabla J(\theta_{t,k})\right\rVert^2 \right]\nonumber\\
&\leq \E\left[J(\theta_{t})\right]+\left(r_{t,0}+s_{t+1}\left(1- \left(\beta_t\lambda_\omega -\beta_t^2\lambda_\omega^2\right)\right)\right)\E\left[\left\lVert \omega_{t}-\omega^*(\theta_{t})\right\rVert^2 \right]\nonumber\\
&\quad+s_{t+1}\left(\left(3\beta_t^2+7\beta_t^3\lambda_\omega\right) \E\left[\left\lVert\nabla_\omega M(\theta_{t},\omega^*(\theta_{t}),D)\right\rVert^2\right]\right.\nonumber\\
&\qquad+\left(3\beta_t^2+\frac{8\beta_t}{\lambda_\omega}+7\beta_t^3\lambda_\omega \right)\E\left[\left\lVert \widehat{\nabla}_\omega M(\theta_{t+1},\omega_t,D_t)-\nabla_\omega M(\theta_{t+1},\omega_t,D) \right\rVert^2\right] \nonumber\\
&\qquad\left.+\left( 1+ 7\beta_t^3\lambda_\omega^3+3\beta_t^2\lambda_\omega^2+\frac{7}{\beta_t\lambda_\omega}+17\beta_t\lambda_\omega\right)\frac{nL_m^2}{\lambda_\omega^2}\E\left[\left\lVert \theta_{t+1}-\theta_t\right\rVert^2\right]\right) \nonumber\\
&\quad+ \sum_{k=0}^{n_b-1}\alpha_{t}\E\left[\left \lVert \nabla_\theta L(\theta_{t,k},\omega^*(\theta_{t,k}),D) -\widehat{\nabla}_\theta L(\theta_{t,k},\omega^*(\theta_{t,k}),D_t)  \right \rVert^2 \right]\nonumber\\
&\quad+\sum_{k=0}^{n_b-1} \left(\alpha_{t}^2\left(\frac{L_{l'}}{2}+ \alpha_{t}\frac{7n^2L_m^2L_l^2n_b\left(1+n_b\right)}{4\lambda_\omega^2} \right)\right)\E\left[\left\lVert  \widehat{\nabla}_\theta L(\theta_{t,k},\omega_t,D_t) \right\rVert^2\right]
- \sum_{k=0}^{n_b-1}\frac{\alpha_{t}}{2} \E\left[\left\lVert \nabla J(\theta_{t,k})\right\rVert^2 \right]\nonumber\\
&\stackrel{(a)}{\leq} \E\left[J(\theta_{t})\right]+\left(r_{t,0}+s_{t+1}\left(1- \left(\beta_t\lambda_\omega -\beta_t^2\lambda_\omega^2\right)\right)\right)\E\left[\left\lVert \omega_{t}-\omega^*(\theta_{t})\right\rVert^2 \right]\nonumber\\
&\quad+s_{t+1}\left(\left(3\beta_t^2+7\beta_t^3\lambda_\omega\right) \E\left[\left\lVert\nabla_\omega M(\theta_{t},\omega^*(\theta_{t}),D)\right\rVert^2\right]\right.\nonumber\\
&\qquad+\left(3\beta_t^2+\frac{8\beta_t}{\lambda_\omega}+7\beta_t^3\lambda_\omega \right)\E\left[\left\lVert \widehat{\nabla}_\omega M(\theta_{t+1},\omega_t,D_t)-\nabla_\omega M(\theta_{t+1},\omega_t,D) \right\rVert^2\right] \nonumber\\
&\qquad\left.+\left( 1+ 7\beta_t^3\lambda_\omega^3+3\beta_t^2\lambda_\omega^2+\frac{7}{\beta_t\lambda_\omega}+17\beta_t\lambda_\omega\right)\frac{nL_m^2}{\lambda_\omega^2}\alpha_t^2n_b\sum_{k=0}^{n_b-1}\E\left[\left\lVert \widehat{\nabla}_\theta L(\theta_{t,k},\omega_t,D_t)\right\rVert^2\right]\right) \nonumber\\
&\quad+ \sum_{k=0}^{n_b-1}\alpha_{t}\E\left[\left \lVert \nabla_\theta L(\theta_{t,k},\omega^*(\theta_{t,k}),D) -\widehat{\nabla}_\theta L(\theta_{t,k},\omega^*(\theta_{t,k}),D_t)  \right \rVert^2 \right]\nonumber\\
&\quad+\sum_{k=0}^{n_b-1}\left(\alpha_{t}^2\left(\frac{L_{l'}}{2}+ \alpha_{t}\frac{7n^2L_m^2L_l^2n_b\left(1+n_b\right)}{4\lambda_\omega^2} \right)\right)\E\left[\left\lVert  \widehat{\nabla}_\theta L(\theta_{t,k},\omega_t,D_t) \right\rVert^2\right]
- \sum_{k=0}^{n_b-1}\frac{\alpha_{t}}{2} \E\left[\left\lVert \nabla J(\theta_{t,k})\right\rVert^2 \right]\nonumber\\
&= \E\left[J(\theta_{t})\right]+\left(r_{t,0}+s_{t+1}\left(1- \left(\beta_t\lambda_\omega -\beta_t^2\lambda_\omega^2\right)\right)\right)\E\left[\left\lVert \omega_{t}-\omega^*(\theta_{t})\right\rVert^2 \right]\nonumber\\
&\quad+s_{t+1}\left(\left(3\beta_t^2+7\beta_t^3\lambda_\omega\right) \E\left[\left\lVert\nabla_\omega M(\theta_{t},\omega^*(\theta_{t}),D)\right\rVert^2\right]\right.\nonumber\\
 &\qquad+\left(3\beta_t^2+\frac{8\beta_t}{\lambda_\omega}+7\beta_t^3\lambda_\omega \right)\E\left[\left\lVert \widehat{\nabla}_\omega M(\theta_{t+1},\omega_t,D_t)-\nabla_\omega M(\theta_{t+1},\omega_t,D) \right\rVert^2\right] \nonumber\\
&\qquad + \alpha_t^2\left(  \beta_t^3 7\lambda_\omega L_m^2nn_b +\beta_t^2 3L_m^2nn_b+\frac{\beta_t 17L_m^2nn_b}{\lambda_\omega}+\frac{\alpha_{t}7n^2L_m^2L_l^2n_b\left(1\!+n_b\right)}{4\lambda_\omega^2}
+\frac{\lambda_\omega^2 L_{l'}+2 L_m^2nn_b}{2\lambda_\omega^2} +\frac{7L_m^2nn_b}{\beta_t\lambda_\omega^3} \right)\nonumber\\
&\qquad\quad\left.\times\sum_{k=0}^{n_b-1}\E\left[\left\lVert \widehat{\nabla}_\theta L(\theta_{t,k},\omega_t,D_t)\right\rVert^2\right]\right) \nonumber\\
&\quad+ \sum_{k=0}^{n_b-1}\alpha_{t}\E\left[\left \lVert \nabla_\theta L(\theta_{t,k},\omega^*(\theta_{t,k}),D) -\widehat{\nabla}_\theta L(\theta_{t,k},\omega^*(\theta_{t,k}),D_t)  \right \rVert^2 \right]
- \sum_{k=0}^{n_b-1}\frac{\alpha_{t}}{2} \E\left[\left\lVert \nabla J(\theta_{t,k})\right\rVert^2 \right].
\end{align}
In the above, the step $(a)$ follows from \eqref{eq:abt_theta_it} and the fact that $\left\lVert \sum_{i=1}^{n}a_i\right\rVert^2\leq n\sum_{i=1}^{n} \left\lVert a_i\right\rVert^2$.
Let
\begin{align}
\label{eq:cent_s_t_rec}
s_t=
\begin{cases}
r_{t,0}+s_{t+1}\left(1-\left(\beta_t\lambda_\omega -\beta_t^2\lambda_\omega^2\right)\right) & \textrm{ for } t \in \{0,T-1\}\\
0& \textrm{ for } t\geq T,
\end{cases}
\end{align}
Solving the recursion in \eqref{eq:cent_s_t_rec1}, we obtain
\begin{align}
\label{eq:cent_s_t_rec1}
s_t=\sum_{i=t}^{T-1}r_{i,0}\prod_{j=t}^{i-1}\left(1-\left(\beta_j\lambda_\omega -\beta_j^2\lambda_\omega^2\right)\right), \textrm{ where }
\sum_{i=t}^{t-1}r_{i,0}=0 \textrm{ and } \prod_{j=t}^{t-1}\left(1-\left(\beta_j\lambda_\omega -\beta_j^2\lambda_\omega^2\right)\right)=1.
\end{align}

We can see that
\begin{align}
\label{eq:cent_s_t_bnd}
\forall t, s_t&\leq s_0,
\end{align}
and
\begin{align}
\label{eq:cent_s_0_bnd}
s_0 &= \sum_{i=0}^{T-1}r_{i,0}\prod_{j=0}^{i-1}\left(1-\left(\beta_j\lambda_\omega -\beta_j^2\lambda_\omega^2\right)\right)\nonumber\\
&= \sum_{i=0}^{T-1}\alpha_{i}nL_l^2n_b\left(\left(1+\frac{1}{n_b}\right)^{n_b}-1\right)\prod_{j=0}^{i-1}\left(1-\left(\beta_j\lambda_\omega -\beta_j^2\lambda_\omega^2\right)\right)\nonumber\\
&\leq\sum_{i=0}^{T-1}\frac{7}{4}\alpha_{i}nn_bL_l^2\prod_{j=0}^{i-1}\left(1-\left(\beta_j\lambda_\omega -\beta_j^2\lambda_\omega^2\right)\right)\nonumber\\
&\stackrel{(a)}{\leq}\sum_{i=0}^{T-1}\frac{7}{4}\alpha_{i}nn_bL_l^2 \exp\left(-\sum_{j=0}^{i-1}\left(\beta_j\lambda_\omega -\beta_j^2\lambda_\omega^2\right)\right).
\end{align}
In the above, the step $(a)$ follows from the fact that $\prod_{i=0}^{n}(1-a_i) \leq \exp\left(-\sum_{i=0}^{n}a_i\right),\,a_i \geq 0$ along with $\left(\beta_j\lambda_\omega -\beta_j^2\lambda_\omega^2\right)\geq0$ as $\forall t,\, \beta_t\in(0,1)$ and $\lambda_\omega \in (0,0.5]$.

Applying \eqref{eq:cent_s_t_rec} and \eqref{eq:cent_s_t_rec1} in \eqref{eq:s_t-1}, we obtain
\begin{align}
\label{eq:s_t-2}
&S_{t+1}\leq S_t\nonumber\\
&+s_{t+1}\left(\left(3\beta_t^2+7\beta_t^3\lambda_\omega\right) \E\left[\left\lVert\nabla_\omega M(\theta_{t},\omega^*(\theta_{t}),D)\right\rVert^2\right]\right.\nonumber\\
&\quad+\left(3\beta_t^2+\frac{8\beta_t}{\lambda_\omega}+7\beta_t^3\lambda_\omega \right)\E\left[\left\lVert \widehat{\nabla}_\omega M(\theta_{t+1},\omega_t,D_t)-\nabla_\omega M(\theta_{t+1},\omega_t,D) \right\rVert^2\right] \nonumber\\
&\quad + \alpha_t^2\left(  \beta_t^3 7\lambda_\omega L_m^2nn_b +\beta_t^2 3L_m^2nn_b+\frac{\beta_t 17L_m^2nn_b}{\lambda_\omega}+\frac{\alpha_{t}7n^2L_m^2L_l^2n_b\left(1+n_b\right)}{4\lambda_\omega^2}
+\frac{\lambda_\omega^2 L_{l'}+2 L_m^2nn_b}{2\lambda_\omega^2} +\frac{7L_m^2nn_b}{\beta_t\lambda_\omega^3} \right)\nonumber\\
&\qquad\left.\times\sum_{k=0}^{n_b-1}\E\left[\left\lVert \widehat{\nabla}_\theta L(\theta_{t,k},\omega_t,D_t)\right\rVert^2\right]\right) \nonumber\\
&+ \sum_{k=0}^{n_b-1}\alpha_{t}\E\left[\left \lVert \nabla_\theta L(\theta_{t,k},\omega^*(\theta_{t,k}),D) -\widehat{\nabla}_\theta L(\theta_{t,k},\omega^*(\theta_{t,k}),D_t)  \right \rVert^2 \right]
- \sum_{k=0}^{n_b-1}\frac{\alpha_{t}}{2} \E\left[\left\lVert \nabla J(\theta_{t,k})\right\rVert^2 \right].
\end{align}

Summing \eqref{eq:s_t-2} from $t=0,\cdots, T-1$, we obtain
\begin{align}
\label{eq:s_t-3}
&S_{T} \leq S_0 \nonumber\\
&+\sum_{t=0}^{T-1} s_{t+1}\left(\left(3\beta_t^2+7\beta_t^3\lambda_\omega\right) \E\left[\left\lVert\nabla_\omega M(\theta_{t},\omega^*(\theta_{t}),D)\right\rVert^2\right]\right.\nonumber\\
&\quad+\left(3\beta_t^2+\frac{8\beta_t}{\lambda_\omega}+7\beta_t^3\lambda_\omega \right)\E\left[\left\lVert \widehat{\nabla}_\omega M(\theta_{t+1},\omega_t,D_t)-\nabla_\omega M(\theta_{t+1},\omega_t,D) \right\rVert^2\right] \nonumber\\
&\quad + \alpha_t^2\left(  \beta_t^3 7\lambda_\omega L_m^2nn_b +\beta_t^2 3L_m^2nn_b+\frac{\beta_t 17L_m^2nn_b}{\lambda_\omega}+\frac{\alpha_{t}7n^2L_m^2L_l^2n_b\left(1+n_b\right)}{4\lambda_\omega^2}
+\frac{\lambda_\omega^2 L_{l'}+2 L_m^2nn_b}{2\lambda_\omega^2} +\frac{7L_m^2nn_b}{\beta_t\lambda_\omega^3} \right)\nonumber\\
&\qquad \times\left.\sum_{k=0}^{n_b-1}\E\left[\left\lVert \widehat{\nabla}_\theta L(\theta_{t,k},\omega_t,D_t)\right\rVert^2\right]\right) \nonumber\\
&+ \sum_{t=0}^{T-1}\sum_{k=0}^{n_b-1}\alpha_{t}\E\left[\left \lVert \nabla_\theta L(\theta_{t,k},\omega^*(\theta_{t,k}),D) -\widehat{\nabla}_\theta L(\theta_{t,k},\omega^*(\theta_{t,k}),D_t)  \right \rVert^2 \right]
- \sum_{t=0}^{T-1}\sum_{k=0}^{n_b-1}\frac{\alpha_{t}}{2} \E\left[\left\lVert \nabla J(\theta_{t,k})\right\rVert^2 \right].
\end{align}

From \eqref{eq:s_t}, we obtain
\begin{align}
\label{eq:S_T}
S_T &= \E\left[J(\theta_{T})\right]+ s_T \E\left[ \left\lVert \omega^*(\theta_{T})-\omega_T \right\rVert^2\right]
\stackrel{(a)}{=}  \E\left[J(\theta_{T})\right],
\end{align}
and
\begin{align}
\label{eq:S_0}
S_0 &= \E\left[J(\theta_{0})\right]+ s_0 \E\left[ \left\lVert \omega^*(\theta_{0})-\omega_0 \right\rVert^2\right].
\end{align}
where step $(a)$ follows since $s_T=0$ from \eqref{eq:cent_s_t_rec}.

Applying \eqref{eq:S_T} and \eqref{eq:S_0} in \eqref{eq:s_t-3}, we obtain
\begin{align}
\label{eq:s_t-4}
&\E\left[J(\theta_{T})\right]\nonumber\\
&\leq \E\left[J(\theta_{0})\right]+ s_0 \E\left[ \left\lVert \omega^*(\theta_{0})-\omega_0 \right\rVert^2\right]+\sum_{t=0}^{T-1} s_{t+1}\left(3\beta_t^2+7\beta_t^3\lambda_\omega\right) \E\left[\left\lVert\nabla_\omega M(\theta_{t},\omega^*(\theta_{t}),D)\right\rVert^2\right]\nonumber\\
 &\quad+\sum_{t=0}^{T-1} s_{t+1}\left(3\beta_t^2+\frac{8\beta_t}{\lambda_\omega}+7\beta_t^3\lambda_\omega \right)\E\left[\left\lVert \widehat{\nabla}_\omega M(\theta_{t+1},\omega_t,D_t)-\nabla_\omega M(\theta_{t+1},\omega_t,D) \right\rVert^2\right] \nonumber\\
&\quad+\sum_{t=0}^{T-1} s_{t+1}\alpha_t^2\left(  \beta_t^3 7\lambda_\omega L_m^2nn_b +\beta_t^2 3L_m^2nn_b+\frac{\beta_t 17L_m^2nn_b}{\lambda_\omega}+\frac{\alpha_{t}7n^2L_m^2L_l^2n_b\left(1+n_b\right)}{4\lambda_\omega^2}\right.\nonumber\\
&\qquad\left.+\frac{\lambda_\omega^2 L_{l'}+2 L_m^2nn_b}{2\lambda_\omega^2} +\frac{7L_m^2nn_b}{\beta_t\lambda_\omega^3} \right)\sum_{k=0}^{n_b-1}\E\left[\left\lVert \widehat{\nabla}_\theta L(\theta_{t,k},\omega_t,D_t)\right\rVert^2\right] \nonumber\\
&\quad+ \sum_{t=0}^{T-1}\sum_{k=0}^{n_b-1}\alpha_{t}\E\left[\left \lVert \nabla_\theta L(\theta_{t,k},\omega^*(\theta_{t,k}),D) -\widehat{\nabla}_\theta L(\theta_{t,k},\omega^*(\theta_{t,k}),D_t)  \right \rVert^2 \right]
- \sum_{t=0}^{T-1}\sum_{k=0}^{n_b-1}\frac{\alpha_{t}}{2} \E\left[\left\lVert \nabla J(\theta_{t,k})\right\rVert^2 \right]\nonumber\\
&\stackrel{(a)}{\leq} \E\left[J(\theta_{0})\right]+ s_0 \E\left[ \left\lVert \omega^*(\theta_{0})-\omega_0 \right\rVert^2\right]+\sum_{t=0}^{T-1} s_{0}\left(3\beta_t^2+7\beta_t^3\lambda_\omega\right) \left(2nM_m^2 + 2\lambda_\omega^2M_{\omega^*}^2\right)\nonumber\\
 &\quad+\sum_{t=0}^{T-1} s_{0}\left(3\beta_t^2+\frac{8\beta_t}{\lambda_\omega}+7\beta_t^3\lambda_\omega \right)4M_m^2\sum_{i=1}^n\E\left[\frac{1}{\lfloor(1-p)|D_t^i|\rfloor}\right] \nonumber\\
&\quad+\sum_{t=0}^{T-1} s_{0}\alpha_t^2\left(  \beta_t^3 7\lambda_\omega L_m^2nn_b +\beta_t^2 3L_m^2nn_b+\frac{\beta_t 17L_m^2nn_b}{\lambda_\omega}+\frac{\alpha_{t}7n^2L_m^2L_l^2n_b\left(1+n_b\right)}{4\lambda_\omega^2}\right.\nonumber\\
&\qquad\left.+\frac{\lambda_\omega^2 L_{l'}+2 L_m^2nn_b}{2\lambda_\omega^2} +\frac{7L_m^2nn_b}{\beta_t\lambda_\omega^3} \right)nn_bL_l^2
+ \sum_{t=0}^{T-1}\alpha_{t}16{L_l}^2e^2n_b^2\sum_{i=1}^n \E\left[\frac{1}{\left\lfloor p|D^i_t|\right\rfloor}\right]
- \sum_{t=0}^{T-1}\sum_{k=0}^{n_b-1}\frac{\alpha_{t}}{2} \E\left[\left\lVert \nabla J(\theta_{t,k})\right\rVert^2 \right]\nonumber\\
&\stackrel{(b)}{\leq} \E\left[J(\theta_{0})\right]+ s_0 \E\left[ \left\lVert \omega^*(\theta_{0})-\omega_0 \right\rVert^2\right]+\sum_{t=0}^{T-1} s_{0}\left(3\beta_t^2+7\beta_t^3\lambda_\omega\right) \left(2nM_m^2 + 2\lambda_\omega^2M_{\omega^*}^2\right)\nonumber\\
 &\quad+\sum_{t=0}^{T-1} \left(s_{0}\left(3\beta_t^2+\frac{8\beta_t}{\lambda_\omega}+7\beta_t^3\lambda_\omega \right)4M_m^2+\alpha_{t}16{L_l}^2e^2n_b^2\right)\sum_{i=1}^n\E\left[\frac{1}{\lfloor (1-p)|D_t^i|\rfloor}\right] \nonumber\\
&\quad+\sum_{t=0}^{T-1} s_{0}\alpha_t^2\left(  \beta_t^3 7\lambda_\omega L_m^2nn_b +\beta_t^2 3L_m^2nn_b+\frac{\beta_t 17L_m^2nn_b}{\lambda_\omega}+\frac{\alpha_{t}7n^2L_m^2L_l^2n_b\left(1+n_b\right)}{4\lambda_\omega^2}\right.\nonumber\\
&\qquad\left.+\frac{\lambda_\omega^2 L_{l'}+2 L_m^2nn_b}{2\lambda_\omega^2} +\frac{7L_m^2nn_b}{\beta_t\lambda_\omega^3} \right)nn_bL_l^2 \nonumber\\
&\quad - \sum_{t=0}^{T-1}\sum_{k=0}^{n_b-1}\frac{\alpha_{t}}{2} \E\left[\left\lVert \nabla J(\theta_{t,k})\right\rVert^2 \right].
\end{align}
In the above, the step $(a)$ follows from \eqref{eq:cent_s_t_bnd}, and Lemmas \ref{lm:cent_M_bnd}, \ref{lm:cent_M_mse}, \ref{lm:cent_hat_grad_L_bnd} and \ref{lm:cent_mse_gard_L}. The step $(b)$ follows since $p\in(0.5,1)$ from \eqref{eq:abt_tv}.

As $\forall t,\,\alpha_t=\frac{1}{T^{\nicefrac{3}{5}}};\;\beta_t=\frac{1}{T^{\nicefrac{2}{5}}}$, from \eqref{eq:cent_s_0_bnd}, we obtain
\begin{align}
\label{eq:cent_s_0_bnd1}
s_0
& \leq\sum_{i=0}^{T-1}\frac{7}{4}\alpha_{i}nn_bL_l^2 \exp\left(-\sum_{j=0}^{i-1}\left(\beta_j\lambda_\omega -\beta_j^2\lambda_\omega^2\right)\right)\nonumber\\
& \leq\frac{7}{4}\frac{1}{T^{\nicefrac{3}{5}}}nn_bL_l^2\sum_{i=0}^{T-1} \exp\left(-i\left(\frac{1}{T^{\nicefrac{2}{5}}}\lambda_\omega -\frac{1}{T^{\nicefrac{4}{5}}}\lambda_\omega^2\right)\right)\nonumber\\
& \leq\frac{7}{4}\frac{1}{T^{\nicefrac{3}{5}}}nn_bL_l^2\sum_{i=0}^{\infty} \exp\left(-i\left(\frac{1}{T^{\nicefrac{2}{5}}}\lambda_\omega-\frac{1}{T^{\nicefrac{4}{5}}}\lambda_\omega^2\right)\right)\nonumber\\
&\stackrel{(a)}{=}\frac{7}{4}\frac{1}{T^{\nicefrac{3}{5}}}nn_bL_l^2\frac{1}{1-\exp\left(-\left(\frac{1}{T^{\nicefrac{2}{5}}}\lambda_\omega-\frac{1}{T^{\nicefrac{4}{5}}}\lambda_\omega^2\right)\right)}\nonumber\\
&\stackrel{(b)}{\leq}\frac{7}{4}\frac{1}{T^{\nicefrac{3}{5}}}nn_bL_l^2\frac{2}{\frac{1}{T^{\nicefrac{2}{5}}}\lambda_\omega-\frac{1}{T^{\nicefrac{4}{5}}}\lambda_\omega^2}\nonumber\\
&=\frac{7}{2}nn_bL_l^2\frac{1}{T^{\nicefrac{1}{5}}\lambda_\omega-\frac{1}{T^{\nicefrac{1}{5}}}\lambda_\omega^2}\nonumber\\
&\stackrel{(c)}{\leq}\frac{7}{2}nn_bL_l^2\frac{1}{\lambda_\omega^2T^{\nicefrac{1}{5}}}.
\end{align}
In the above, the step $(a)$ follows from the fact that $\sum_{i=0}^{\infty}r^i=\nicefrac{1}{1-r},\,\lvert r \rvert <1$. The step $(b)$ follows from the fact that $\exp(-x) \leq 1-\nicefrac{x}{2},\,x\in[0,1.59]$. The step $(c)$ follows from the fact that $\nicefrac{1}{\left(ax^{\nicefrac{1}{5}}-\nicefrac{a^{2}}{x^{\nicefrac{1}{5}}}\right)}\leq \nicefrac{1}{a^{2}x^{\nicefrac{1}{5}}},\,x\geq1,a\in(0,0.5]$ along with $T\geq1$ and $\lambda_\omega\in (0,0.5]$ from \eqref{eq:M}.

Let $J^*=\argmin_{\theta\in\R^d} J(\theta)$. From \eqref{eq:abt_omega_it} we have $\omega_0=\mathbf{0}$, and from Assumption \ref{as:cent_w*}, we have $\omega^*(\theta_{0})\leq M_{\omega^*}^2$. We also have $\forall t,\,\alpha_t=\frac{1}{T^{\nicefrac{3}{5}}};\;\beta_t=\frac{1}{T^{\nicefrac{2}{5}}}$. Then, applying \eqref{eq:cent_s_0_bnd1} and   in \eqref{eq:s_t-4}, we obtain

\begin{align}
\label{eq:s_t-5}
&\sum_{t=0}^{T-1}\sum_{k=0}^{n_b-1}\frac{1}{2T^{\nicefrac{3}{5}}} \E\left[\left\lVert \nabla J(\theta_{t,k})\right\rVert^2 \right]
\leq J(\theta_{0}) - J^* \nonumber\\
&+ \frac{21nn_bL_l^2}{\lambda_\omega^2}\left(nM_m^2 + \lambda_\omega^2M_{\omega^*}^2\right)+ \frac{7nn_bL_l^2M_{\omega^*}^2}{2\lambda_\omega^2T^{\nicefrac{1}{5}}}+\frac{49nn_bL_l^2}{\lambda_\omega T^{\nicefrac{2}{5}}}\left(nM_m^2 + \lambda_\omega^2M_{\omega^*}^2\right)\nonumber\\
&+ \left(\frac{98nn_bL_l^2M_m^2}{\lambda_\omega T^{\nicefrac{7}{5}}}+\frac{42nn_bL_l^2M_m^2}{\lambda_\omega^2T}+\frac{16{L_l}^2n_b\left(\nicefrac{7nM_m^2}{\lambda_\omega^3}+ e^2n_b\right)}{T^{\nicefrac{3}{5}}}\right)\sum_{t=0}^{T-1}\sum_{i=1}^n\E\left[\frac{1}{\lfloor (1-p)|D_t^i|\rfloor}\right] \nonumber\\
&+\frac{7n^2n_b^2L_l^4}{2\lambda_\omega^2}\left(  \frac{7\lambda_\omega L_m^2nn_b}{T^{\nicefrac{8}{5}}}  +\frac{3L_m^2nn_b}{T^{\nicefrac{6}{5}}} +\frac{7n^2L_m^2L_l^2n_b\left(1+n_b\right)}{4\lambda_\omega^2 T}+\frac{ 17L_m^2nn_b}{\lambda_\omega T^{\nicefrac{4}{5}}}+\frac{\lambda_\omega^2 L_{l'}+2 L_m^2nn_b}{2\lambda_\omega^2 T^{\nicefrac{2}{5}}} +\frac{7L_m^2nn_b}{\lambda_\omega^3} \right).
\end{align}

Since $\p\left (R_1=t\right)=\nicefrac{1}{T}$ and $\p\left (R_2=k \right)= \nicefrac{1}{n_b}$, from \eqref{eq:s_t-5} we obtain
\begin{align}
\label{eq:SR-1}
&\E\left[\left\lVert \nabla J(\theta_{R_1,R_2})\right\rVert^2 \right]
=\frac{\sum_{t=0}^{T-1}\sum_{k=0}^{n_b-1} \E\left[\left\lVert \nabla J(\theta_{t,k})\right\rVert^2 \right]}{Tn_b}
=\frac{2\sum_{t=0}^{T-1}\sum_{k=0}^{n_b-1}\frac{1}{2T^{\nicefrac{3}{5}}} \E\left[\left\lVert \nabla J(\theta_{t,k})\right\rVert^2 \right]}{T^{\nicefrac{2}{5}}n_b}\nonumber\\
&\leq \frac{2\left(J(\theta_{0}) - J^*\right)}{n_bT^{\nicefrac{2}{5}}} +\frac{7L_l^2n\left(7L_l^2 L_m^2n^2n_b^2+ 6\lambda_\omega^3\left(M_m^2n + \lambda_\omega^2M_{\omega^*}^2\right)\right)}{\lambda_\omega^5 T^{\nicefrac{2}{5}}}\nonumber\\
&\quad+ \left(\frac{196 nL_l^2M_m^2}{\lambda_\omega T^{\nicefrac{4}{5}}}+\frac{84nL_l^2M_m^2}{\lambda_\omega^2T^{\nicefrac{2}{5}}}+\frac{32{L_l}^2}{\lambda_\omega^3 }\left(7nM_m^2+ \lambda_\omega^3e^2n_b\right)\right)\sum_{i=1}^n\E\left[\frac{1}{\lfloor (1-p)|D_{R_1}^i|\rfloor}\right] \nonumber\\
&\quad +\frac{7nL_l^2M_{\omega^*}^2}{\lambda_\omega^2T^{\nicefrac{3}{5}}}+ \frac{7L_l^2n\left(nn_bL_l^2\left(\lambda_\omega^2 L_{l'}+2 L_m^2nn_b\right)+28\lambda_\omega^3\left(M_m^2 n +  28\lambda_\omega^2M_{\omega^*}^2\right)\right)}{2\lambda_\omega^4 T^{\nicefrac{4}{5}}}\nonumber\\
&\quad+\frac{7n^3n_b^2L_l^4L_m^2}{\lambda_\omega^2}\left(  \frac{7\lambda_\omega }{T^{2}}  +\frac{3}{T^{\nicefrac{8}{5}}} +\frac{7L_l^2n\left(1+n_b\right)}{4\lambda_\omega^2 T^{\nicefrac{7}{5}}}+\frac{ 17}{\lambda_\omega T^{\nicefrac{6}{5}}}\right).
\end{align}
\hfill{\qed}


\end{document}